\documentclass[]{fairmeta}

\title{Tuning without Peeking: Provable Generalization Bounds and Robust LLM Post-Training}

\author[1,2]{Ismail Labiad}
\author[1,2]{Mathurin Videau}
\author[2,3]{Matthieu Kowalski}
\author[2,3]{Marc Schoenauer}
\author[4]{Alessandro Leite}
\author[1,5]{Julia Kempe}
\author[1]{Olivier Teytaud}
  
\affiliation[1]{Meta FAIR}
\affiliation[2]{Université Paris-Saclay}
\affiliation[3]{LISN, Inria, CNRS}
\affiliation[4]{INSA Rouen Normandy, LITIS}
\affiliation[5]{NYU Courant Institute and CDS}

\newcommand{\retrofit}{{\em BBoxER}}

\abstract{
Gradient-based optimization is the workhorse of deep learning, offering efficient and scalable training via backpropagation. However, exposing gradients during training can leak sensitive information about the underlying data, raising privacy and security concerns such as susceptibility to data poisoning attacks. In contrast, black-box optimization methods, which treat the model as an opaque function, relying solely on function evaluations to guide optimization, offer a promising alternative in scenarios where data access is restricted, adversarial risks are high, or overfitting is a concern.
This paper introduces \retrofit, an evolutionary black-box method for LLM {\em post-training}  that induces an information bottleneck via implicit compression of the training data.
Leveraging the tractability of information flow, we provide non-vacuous generalization bounds and strong theoretical guarantees for robustness to data poisoning attacks and extraction attacks, while ensuring privacy by design.
In experiments with LLMs, we demonstrate empirically that black-box optimization methods—despite the scalability and computational challenges inherent to black-box approaches—are able to learn, showing how a few iterations of \retrofit\ improve performance, generalize well on a benchmark of reasoning datasets, and are robust to membership inference attacks. This positions \retrofit\ as an attractive add-on on top of gradient-based optimization, offering suitability for deployment in restricted environments while also providing non-vacuous generalization guarantees.

}

\date{\today}
\correspondence{\email{ilabiad@meta.com}}

\PassOptionsToPackage{numbers, compress}{natbib}

\usepackage[utf8]{inputenc} %
\usepackage[T1]{fontenc}    %
\usepackage[bookmarksnumbered,debug,backref=page]{hyperref}
\usepackage{bbm}
\usepackage{amsthm,amssymb,amsmath}
\usepackage{thmtools}
\usepackage{mathtools}
\usepackage{upgreek}
\usepackage{dsfont}
\usepackage{comment}

\usepackage{url}            %
\usepackage{booktabs}       %
\usepackage{amsfonts}       %
\usepackage{nicefrac}       %
\usepackage{microtype}      %
\usepackage{ifthen}

\usepackage{wrapfig}
\usepackage{caption}

\usepackage{graphicx} %
\graphicspath{%
  {figures/}
  {cwplots/}
}
\usepackage{microtype}
\usepackage{subfig}
\usepackage{booktabs} %
\usepackage{multicol,multirow}
\usepackage{colortbl,placeins}
\usepackage{tabularx,longtable}
\usepackage[neveradjust]{paralist}
\usepackage{amsmath,amssymb,amsthm,algorithm,algorithmic}
\usepackage{enumitem}
\usepackage{xfrac}
\usepackage{titlesec}
\usepackage{tikz}
\usetikzlibrary{arrows.meta, positioning, calc, shapes.geometric}

\usepackage{tcolorbox}
\definecolor{linen}{RGB}{250,240,230}

\definecolor{mydarkblue}{rgb}{0,0.08,0.45}

\usepackage[normalem]{ulem}

\usepackage[noabbrev,nameinlink,capitalize]{cleveref}

\crefname{ALC@unique}{Line}{Lines}
\Crefname{ALC@unique}{Line}{Lines}
\crefname{algorithm}{Alg.}{Algs.}
\crefname{table}{Tab.}{Tabs.}
\crefname{section}{Sec.}{Secs.}
\crefname{theorem}{Thm}{Thms}
\crefname{theoremUnboxed}{Thm}{Thms}
\crefname{lemma}{Lem.}{Lems.}
\crefname{corollary}{Cor.}{Cors.}
\crefname{appendix}{App.}{Apps.}

\definecolor{paleBlue}{rgb}{0.9,0.9,1.0}

\tcolorboxenvironment{theorem}{colback=paleBlue,colframe=paleBlue!50!black!50}

\definecolor{linkcolor}{RGB}{83,83,182}

\hypersetup
{
  colorlinks=true,
  citecolor=linkcolor,
  linkcolor= linkcolor,%
  urlcolor=blue,
  pdfstartview=FitH,
  pdftitle={},
  pdfauthor={},
  pdfcreator={},
  pdfsubject={},
  pdfkeywords={},     
}

\makeatletter
\patchcmd{\ALG@step}
  {\refstepcounter{ALG@line}}%
  {\stepcounter{ALG@line}}%
  {}{}
\makeatother

\newcommand{\ie}{i.e.,}
\newcommand{\eg}{e.g.,}

\newcommand{\deprecate}[1]{}

\titlespacing*{\paragraph}{0pt}{0.5em}{0.5em}

\definecolor{MyGreen}{RGB}{0,175,0}
\definecolor{MyRed}{RGB}{250,90,90}

\newtheorem{theoremUnboxed}{Theorem}
\newtheorem{corollary}{Corollary}
\newtheorem{lemma}{Lemma}
\newtheorem{definition}{Definition}

\def\hp{parametrization}

\def\D{{\cal D}}

\def\X{{\cal X}}
\def\Y{{\cal Y}}
\def\N{{\mathbb{N}}}
\def\R{\mathbb{R}}
\def\e{\epsilon}
\def\w{\omega}
\def\E{\mathbb{E}}

\newcommand{\OneFifth}[0]{OneFifth}%
\newcommand{\CMA}[0]{CMA-ES}
\newcommand{\DCMA}[0]{D-CMA} %
\newcommand{\Lengler}[0]{Lengler} %
\newcommand{\COLengler}[0]{COLengler} %
\newcommand{\Triple}[0]{3-OneFifth} %
\newcommand{\MultiDisc}[0]{3-Disc} %
\newcommand{\Discrete}[0]{Discrete} %
\newcommand{\NgIohTuned}[0]{NgIohTuned} %
\newcommand{\BARD}[0]{2-DE/D-CMA} %
\newcommand{\FastGA}[0]{FastGA}
\newcommand{\Portfolio}[0]{Uniform-Disc}

\newcommand{\colormode}{0}
\newcommand{\setcolormode}[1]{%
  \renewcommand{\colormode}{#1}%
}
\DeclareRobustCommand{\changes}[1]{%
  \ifnum\colormode=1
    \textcolor{blue}{#1}%
  \else
    \textcolor{black}{#1}%
  \fi
}
\setcolormode{0} %

\DeclareRobustCommand{\rebuttal}[1]{%
  \ifnum\colormode=1
    \textcolor{blue}{#1}%
  \else
    \textcolor{black}{#1}%
  \fi
}

\begin{document}

\maketitle

\section{Introduction}\label{sec:itroduction}

Large Language Models (LLMs) have revolutionized natural language processing, achieving strong performance across diverse tasks through training on massive datasets -- often hundreds of billions of tokens -- and large-scale architectures~\cite{kaplan2020scaling,wei2022emergent}. 
Yet LLMs suffer from persistent weaknesses, foremost among them vulnerability to extraction attacks~\cite{privacyattack3,privacyattack4,carlini2024stealing} where training data can be leaked~\cite{privacyattack1,privacyattack2}, raising potential privacy concerns. They are also vulnerable to data poisoning~\cite{wan2023poisoning,miranda2025preserving}, and lack meaningful generalization guarantees at scale. These challenges are amplified in settings where data is scarce, sensitive, or adversarially curated, underscoring the need for more robust and privacy-aware optimization methods.

Black-Box Optimization (BBO) covers methods suited to settings where gradient information is unreliable or unavailable (e.g., non-differentiable or noisy objectives). These algorithms are naturally robust to local minima and non-convexity. Although around for decades, BBO has been recognized only recently in machine learning, becoming a core tool in AutoML~\cite{autoMLbook2019}, for hyperparameter tuning, algorithm selection~\cite{hpo-hutter2019}, and neural architecture search~(NAS)~\cite{white2023NASinsights}. Beyond random and grid search, techniques like Evolutionary Algorithms (e.g., CMA-ES~\cite{seminalCMAES,CMA,tutorialCMAES}, Differential Evolution~\cite{de}, Particle Swarm Optimization~\cite{pso}) and Bayesian Optimization~\cite{movckus1975bayesian,mockus1981bayesian,garnettbayesoptbook2023} offer diverse trade-offs in scalability, sample efficiency, and exploration-exploitation dynamics. %
BBO methods still face limitations: they scale poorly to high-dimensional parameter spaces and require many function evaluations, making them unsuitable for full model training (except for Deep Reinforcement Learning, due to its high parallelism~\cite{openAI-ES4DeepRL2017}). To reconcile the strengths and weaknesses of both paradigms, we adopt a {\em hybrid strategy}: using gradient-based pretraining for scale, followed by BBO on a small, targeted subset of the model to tackle privacy, poisoning, and overfitting. 

Post-training introduces its own risks: instruction tuning and generation can lead to exploitative behavior~\cite{shu2023exploitability,huang2024catastrophic}, overfitting may harm robustness and generalization~\cite{llmgeneralize}, and data poisoning can occur during reinforcement learning from human feedback~(RLHF)~\cite{wang2024rankpoison, baumgrtner2024bestofvenom} or direct preference optimization~(DPO)~\cite{rando2024universal,poisoningthreat}. This setting is well-suited for retrofitting -- a concept initially developed to refine word vectors post-hoc using semantic knowledge~\cite{retrofitting2015}, which we generalize here to LLMs for arbitrary downstream tasks. 

We introduce \retrofit\ (Black-Box Evolutionary Retrofitting), a comparison-based black-box retrofitting method~\cite{arxivtelo} applicable after pretraining, fine-tuning, or reinforcement learning loops such as %
GRPO~\cite{grpo}. Positioned in the outermost layer of the pipeline, \retrofit\ requires no gradient access and integrates seamlessly with existing black-box libraries and algorithms~\cite{nevergrad}. The resulting sequence of queries and updates forms an implicit compression trace, enabling the use of compression-based generalization theory~\cite{campi2024compressiongeneralizationlearning,compressionbasedlearning}, currently the only route to non-vacuous generalization bounds for LLMs~\cite{compressionllm}. It allows us to derive strong generalization guarantees that - unlike VC dimension~\cite{vapnik1971uniform} or Rademacher complexity~\cite{bartlett2002rademacher} - do not depend on the number of modified parameters but on the complexity of the optimization trajectory.
As a result, \retrofit\ provides provable, {\em non-vacuous} generalization bounds to control overfitting, robustness to data poisoning and extraction attacks, while naturally ensuring privacy by design through its compression bottleneck.  It is uniquely positioned to provide improvements from aggregate feedback of billions of users anonymously, without accessing individual data.

While our approach provides theoretical guarantees, these guarantees are most relevant when they can be achieved without substantially reducing utility. We therefore evaluate \retrofit\ under query budgets of only a few hundred model evaluations and observe consistent improvements in performance and generalization on billion-scale LLMs across reasoning benchmarks. These results suggest that the framework can provide {\it provable generalization and safety while preserving practical usefulness}, making it a viable post-hoc adaptation method for LLMs, particularly in restricted settings.

\textbf{Our contributions.} 
\retrofit\ introduces a new perspective on LLM post-training: instead of fine-tuning weights or relying on reinforcement learning, we show that simple black-box retrofitting can achieve safe and modular adaptation. 
Building on this idea, we make three main contributions:

\phantom{~~~~}\scalebox{0.7}{$\bullet$} \textbf{A general-purpose retrofitting framework.} 
We formalize \retrofit\ (\cref{alg:retro}), a comparison-based black-box optimization scheme that compresses optimization traces, enabling safe and modular adaptation of pretrained and post-trained LLMs (\cref{sec:retrofitting}).

\phantom{~~~~}\scalebox{0.7}{$\bullet$} \textbf{Strong theoretical guarantees.} 
We derive non-vacuous generalization bounds (\cref{cor1}) that scale linearly with dataset size, provide provable robustness to poisoning (\cref{tpp}) and to extraction attacks (\cref{co:unfair}), while additionally ensuring privacy by design (\cref{sec:generalization}).

\phantom{~~~~}\scalebox{0.7}{$\bullet$} \textbf{Empirical validation on billion-scale LLMs.} 
Despite tight query budgets, we show consistent gains on GSM8K and related math benchmarks with Llama3.1-8B and Qwen-2.5-3B, and provide empirical evidence supporting our theoretical claims: retrofitted models with \retrofit\ resist Membership Inference Attacks, unlike fine-tuned counterparts at equal utility (\cref{sec:xp}).

\section{Background and Related Work}\label{sec:relatedwork}
An extensive exposition of related work is provided in~\Cref{app:extrelatedwork}.

{\bf Generalization Bounds for LLMs.}
Classical generalization theory based on uniform convergence and VC dimension fails for large models like LLMs. Modern approaches either apply PAC-Bayes bounds using compression techniques~\cite{zhou2018non,LotfiPACcompression,compressionllm} or adopt algorithmic stability~\cite{as}, which does not rely on hypothesis class complexity. Our approach aligns with the latter by leveraging comparison-based optimization and using union-bound-like arguments (\eg{} Bonferroni correction) to derive non-trivial generalization bounds.

{\bf Data Poisoning.}
Recent poisoning attacks show that corrupting as little as 1–5\% of fine-tuning or preference data can induce persistent, targeted behaviors in LLMs~\cite{baumgrtner2024bestofvenom,wang2024rankpoison}, even in quantized variants~\cite{llmpoisoq}. These attacks evade standard detection and persist during inference. Moreover, the common strategy of training on public data then fine-tuning on private data is also vulnerable to backdoor injection~\citep{feng2024privacy}. Specifically, \citet{feng2024privacy} showed that allowing neurons to retain a gradient from a single input and later “deactivate” to avoid overwriting makes models susceptible to gradient inversion. Our black-box optimization approach is, by design, secure against such attacks. While theoretical bounds exist for simpler models~\cite{steinhardt2017certified}, robust guarantees for LLMs remain elusive.

\paragraph{Post-training for Reasoning.} Unlike common RL approaches such as GRPO~\cite{grpo}, our approach combines low-rank adaptations with retrofitting techniques limiting overfitting risks and prevents the amplification of memorized or poisoned content. Other works showed that it is possible to post-train LLMs using evolution algorithms, specifically evolution strategies in the case of \cite{qiu2025evolution, sarkar2025evolution}, but they differ from our setup as we only rely on comparison results rather than the result of the evaluation over samples.

\section{Methods: Retrofitting with BBO}\label{sec:retrofitting}
Retrofitting~\cite{douglas:06,dawson:07,dixon:13} is routinely deployed in industry, in order to adapt old devices to new contexts or new data. 
In machine learning, the term has mainly been used in natural language processing since~\citet{retrofitting2015}, who adapted the representation of the word vector to take semantic knowledge into account. It is performed after classical training, \ie{} without retraining the entire network.
\citet{arxivtelo} adapted retrofitting to the black-box optimization context. In the present paper, we \rebuttal{adapt retrofitting to LLM post-training and instantiate it} as \retrofit\ and consider comparison-based black-box optimization algorithms to obtain the explicit compression properties that lend themselves to the strong bounds proved in~\cref{sec:generalization}.

A key motivation for using BBO is that it allows post-training with non-differentiable loss functions such as accuracy or signals from A/B testing. And since it only requires model inference, it significantly reduces memory usage. More importantly, the main motivation of this work is to leverage the compression achieved by reducing the entire dataset’s forward pass to just a single—or a few—comparison results, yielding a finite branching factor. This constrains our deployment of BBO to comparison-based algorithms, which include Evolutionary Algorithms (EA) and Strategies (ES) like Population-Based Evolution Strategies (e.g., CMA, DiagonalCMA), Differential Evolution, and Particle Swarm Optimization - all described in \cref{sec:otherBBO}. While a vast class of BBO algorithms is not {\em comparison based} (see \cref{sec:otherBBO} for an anthology of BBO) this nonetheless gives us a substantial variety of BBO algorithms to deploy.
For example, the (1+1)-OneFifth strategy samples new points from a Gaussian centered on the current best solution and adjusts the mutation step size based on the acceptance rate. Similarly, CMA-ES updates both the mean and covariance of its sampling distribution based on comparisons within the current population. More algorithms are described in~\cref{sec:bbo}. Our \retrofit{} framework is modular and algorithm agnostic, and we instantiate it with methods from the Nevergrad suite~\cite{nevergrad}.

\subsection{The Retrofitting Framework: \retrofit}
{\bf Notation.} 
Let $\D$ be the class of datasets of size $s$ over a domain $\X\times \Y$. Models map $\X$ to $\Y$.
Let $\w$ be the random seed corresponding to the possibly randomized optimization algorithm $a$ and $\w_D$ the random seed associated with the draw of $D\in \D$.

\paragraph{Initial model and modifiers.} We denote $m_0$ the initial model, irrespective of its creation.
$modified(m_0,x)$ refers to a modification of $m_0$ parametrized by some $x$. We prove all results for an arbitrary $modified$ function.
For example, $modified$ can be a simple added perturbation to all model parameters, or a multiplicative low-rank update to a particular subset of the LLM layers, $x$ then contains the low-rank matrices that parametrize the update~(Examples of such $modified$ function are provided in \cref{concrete}). 

\cref{alg:retro} presents our abstract retrofitting framework \retrofit.

\begin{center}
\begin{algorithm}[H]
\caption{\small{\retrofit($m_0,\w,D,a,b$): 
Returns a modification of  initial model $m_0$ by BBO algorithm $a$ on dataset $D$ with budget $b$ and random seed $\w$.}
}
\label{alg:retro}
\begin{algorithmic}[1]
\begin{scriptsize}
\STATE{$I_0=a.initialize(\w)$}\label{initalg}
\FOR{$i\in\{1,\dots,b\}$}
\STATE{$m_i=modified(m_0,x_i:=a.ask())$}\label{line:ask}
\STATE{$k_i:=a.numCases()$}\label{branching}
\STATE{$choice_i=a.compare((m_1,\dots,m_i),D) \in \{1,\dots,k_i\}$}\label{alg:bound}
\STATE{$a.tell(choice_i)$}\label{alg:update}
\STATE{$I_i=(I_{i-1},choice_i)$}\label{alg:memorize}
\ENDFOR
\STATE{$FinalModel = modified(m_0, \widehat{x} = a.recommend())$} \label{finalresult}
\end{scriptsize}
\end{algorithmic}
\end{algorithm}
\end{center}

First, Algorithm $a$ is initialized and returns its initial internal state $I_0$ in \cref{initalg}. Then for $b$ iterations, Algorithm $a$, via $a.ask$ at \cref{line:ask}, proposes a new sample $x_i$, which is used to obtain a new model $m_i$ through $modified$, \rebuttal{which can be any function that modifies the model weights using the sampled noise by Algorithm $a$.} \cref{branching} introduces the branching factor $k_i$, representing the finite number of possible outputs of $a.compare$, which we assume is upper bounded. At \cref{alg:bound}, $a.compare$ returns $choice_i \in \{1,\dots,k_i\}$, the outcome of comparing various previously proposed models $m_1,\dots,m_i$ on the data $D$.
For instance, the comparison may involve the current model and the best model obtained so far, which corresponds to $k_i = 2$. The outcome of this comparison can be based on user-provided preference, such as in A/B testing, or on performance evaluations over $D$, where we compute a metric (e.g., accuracy) but only retain the binary comparison result (e.g., 0 or 1).
In \cref{alg:memorize}, the internal state of $a$ is updated to account for the comparison result. Finally $a.recommend$ at \cref{finalresult} returns the modification $\widehat{x}$ used to obtain the final model. \rebuttal{A step by step execution example of \retrofit{} can be found in \cref{section:bboxer_example}.}

It is important to note that the internal state $I_b$ of algorithm $a$ is entirely determined by $\w,a,b,\allowbreak choice_1,\dots,choice_b$ and contains all the information needed to obtain $\widehat{x}$, and thus the final modified model. Moreover, the number of distinct possible outputs of \retrofit\ is upper bounded by the number of distinct possible values of $I_b$. We assume that the randomness of algorithm $a$ is fixed given the seed $\w$; stochastic algorithms can be obtained by randomizing the seed $\omega$ in $a$.

\subsection{\retrofit: Link with compression}\label{subsec:compression}
In~\cref{alg:retro}, $FinalModel=modified(m_0,\hat{x})$ is a deterministic function of $m_0,\w,a,b,D$.
And since $\widehat x$ only depends on $\w,a,b,choice_1,\dots,choice_b$, we can construct $\widehat x$ without $D$. Hence, there exists a mapping X
\begin{flalign}\label{eq:impact}
    \text{s.t.} && \widehat x  &=  X(\w,a,b,choice_1,\dots,choice_b) &&
\end{flalign}

But $(choice_1,\dots,choice_b)$ is a deterministic function of $m_0,\w,a,b,D$. Hence, there exists a deterministic $compression$ function $c$ defined by
\begin{equation}
    c(m_0,\w,a,b,D):=(choice_1,\dots,choice_b) \label{eq:compression}
\end{equation}

Here, $X$ mimics the \retrofit\ \cref{alg:retro} but returns $\widehat x$ instead of $modified(m_0,\widehat x)$, and uses $c$ instead of $D$. 
The dataset, replaced by $c=(choice_1,\dots,choice_b)$, is entirely removed from  \cref{eq:impact}. 
To view \retrofit\ as compression, we can merge~\cref{eq:impact} and~\cref{eq:compression} into:
\begin{equation}\label{eq:allinone}
FinalModel=modified \big(m_0,X(\w,a,b, c:=compression(m_0,\w,a,b,D))\big)
\end{equation}

$FinalModel$ only depends on the data through the compression bottleneck $c$ of bounded size (see~
\cref{fig:fairuse}, Appendix, for illustration). 
The impact of this compression is formalized in~\cref{sec:generalization}, and used to prove guarantees in terms of overfitting, privacy robustness w.r.t poisoning and extraction~attacks.

\section{Theoretical analysis}\label{sec:generalization}

Analyzing \retrofit\ within a unified framework through its information bottleneck that compresses the dataset into a sequence of comparisons, allows us to state and prove strong and precise bounds on its ability to preserve generalization, its robustness to poisoning, its privacy guarantee, and its robustness against extraction attacks. For mathematical background see~\cref{sec:thapp}, additional results in~\cref{genbou}, and proofs in~\cref{sec:proofs}.

\subsection{Overfitting bounds}\label{mathover}
In this section, we address generalization by bounding the gap between empirical and true loss in terms of the variability of the algorithm’s internal state. Notably, these bounds are governed by the optimization path, with no dependence on the number of model parameters optimized during the retrofitting process.

Let $L(x) = L(modified(m_0,x))$ denote the true loss over dataset distribution $\D$, and $\widehat{L}(x) = \widehat{L}(modified(m_0,x))$ the empirical loss on sampled dataset $D$. Let $N(\w, a, b)$ denote the number of possible internal states over all possible datasets in $\D$, then from~\cref{alg:bound} in~\cref{alg:retro} we have:
\begin{align}
N(\w, a, b) \leq \sup_{D \in \D} \prod_{i=1}^b k_i(\w, D, a, b).\label{zebound}
\end{align}
We fix the algorithm $a$ seed \(\w\) (the stochastic case is discussed in~\cref{app:stochastic}) and study how the number of distinct possible internal states affects generalization. \cref{deterthm} shows that the generalization gap is controlled by the number of reachable configurations in~\cref{zebound}. It also provides a uniform bound over all intermediate models \((m_1, \dots, m_b)\), ensuring that empirical losses remain close to the true risk throughout optimization. The bound is obtained using mild assumptions, only requiring that the empirical loss is a good proxy for the true loss, a standard concentration-bound hypothesis in classical learning theory~\cite{VVW,vap95} (see ~\cref{sec:ldb}).

\setcounter{theoremUnboxed}{0}
\begin{theoremUnboxed}[Deterministic case]\label{deterthm}
Let $a$ be an algorithm, $b$ a budget, $s$ the dataset size, and $\epsilon > 0$. Assume $N(\w,a,b)$ is finite and that for some $\delta_{1,\epsilon}$,
\(
    \forall x, \; P_{\w_D}\left(|\widehat{L}(x) - L(x)| > \epsilon\right) \leq \delta_{1,\epsilon}.
\)
Then
\begin{align}
 P_{\w_D}\left(|\widehat L(\widehat x) - L(\widehat x)| \geq \epsilon\right) & \leq  N(\w,a,b)\cdot
\delta_{1,\e}.&& \label{g}\\
P_{\w_D}\left(\sup_{1\leq i \leq b}|\widehat L(x_i) - L(x_i)| \geq \epsilon\right) & \leq  b \cdot N(\w,a,b)\cdot
\delta_{1,\e}. && \label{e}
\end{align}
\end{theoremUnboxed}

Proof of \cref{deterthm} alongside the randomized version can be found in \cref{genbou}.
We present below the practical case where the individual losses are bounded in [0,1], as with accuracy or user preference signals $\in \{0,1\}$, and where data points are sampled independently from each~other.

\begin{corollary}\label{cor1}
    Under standard concentration assumptions  (\cref{eqone}), if individual losses are in $[0,1]$ and Hoeffding’s inequality applies (Eq.~\ref{eq:hoeffding}) with dataset size s, then:
    \begin{align}\label{eq:co1}
    P_{\w_D}\left(|\widehat L(\widehat x) - L(\widehat x)| \geq \epsilon\right)  \leq 2\cdot \left(\prod_{i=1}^b k_i\right) \cdot \exp(-2s\e^2) \leq 2\cdot \exp\left(\ln2\cdot \sum_{i=1}^b \log_2(k_i)-2s\e^2\right)
    \end{align}
\end{corollary}

Most BBO algorithms used within \retrofit\ bound the average $k_i$ by $2$ or $3$. Interpreting $\log_2(k_i)$ as a budget in bits means that to avoid overfitting we can afford a number of iterations $b$ proportional to $s$ divided by the bits used per step. Precisely, if we assume $\forall i, k_i = 2$, then for an error gap $\epsilon$ and confidence $\delta$, we have from~\cref{cor1} the following formula (numerical values can be found in~\cref{ssize}):
\begin{align}
P_{\w_D}\left(|\widehat L(\widehat x) - L(\widehat x)| \geq \epsilon\right) \leq \delta \Leftrightarrow b \leq \frac{\ln(\delta)+2s\e^2}{\ln(2)}-1 \label{bupperbound}
\end{align}

\textbf{Bounds on additional BBO algorithms.} Our analysis extends to a wide range of evolutionary BBO methods. In particular, \cref{sec:bar} shows that so called {\em Bet-And-Run} variants - running $N$ parallel runs each with a budget of $b/N$ - can reduce overfitting and scale the total budget while preserving generalization. \cref{sec:pop} and \cref{sec:depso} provide analogous bounds for Population-Based Evolution Strategies (e.g., CMA, DiagonalCMA), Differential Evolution, and Particle Swarm Optimization.

\subsection{Poisoning and privacy}\label{sec:privpoison}

We now consider the two key risks associated with model post-training:
\begin{inparaenum}[(a)] 
    \item \emph{poisoning}, where a small number of corrupted inputs can steer the outcome, and
        \item \emph{privacy leakage}, where individual training samples
        might be inferred from the final model.
\end{inparaenum} 
Standard fine-tuning methods, including privacy-aware approaches, rely on batch-level gradient signals that inherently encode information 
on the underlying data. \retrofit\, on the other hand, relies only on
aggregate values over the entire dataset, namely on the outcomes of model comparisons $(choice_1, \dots, choice_b)$. Each $choice_i$ represents a compressed signal derived from the full dataset $D$. %
For a training dataset $D$, in this context we define {\em preference signals} to be the sample-level decisions that lead to aggregate decision $choice_i$ over $D$. In the case of A/B testing on prompt/output pairs, these would be the binary user preferences that aggregate to a model preference. In the case of more general metrics, the preference signal for a sample is defined as the result of models comparison over the subdataset~\{sample\}.

To analyze \retrofit, we consider a simplified setting, where $\forall i,k_i=2$, and in which two models $\{0,1\}$ are compared by majority vote over \(s\) independent preference signals \(r_i \in \{0,1\}\). The observed frequency is \(f = \tfrac{1}{s} \sum_{i=1}^s r_i\).
An adversary may corrupt up to \(m\) values among the \(r_i\), yielding a perturbed frequency \(f^\prime\) such that \(|f'-f| \leq m/s\). This follows the standard poisoning model~\cite{malicious}, where the adversary is computationally unbounded and fully informed, but constrained in the number of values it can alter. \Cref{ltpp} in \cref{sec:proofs} bounds the probability that such perturbations alter the comparison outcome.

Using a union bound argument, we extend this analysis to \retrofit\ (\cref{alg:retro}), which performs \(b\) consecutive comparisons. In iteration \(j \in \{1,\dots,b\}\), the algorithm collects binary preferences \(r_{j,1},\dots,r_{j,s}\), aggregates them into \( f_j = \tfrac{1}{s} \sum_{i=1}^s r_{j,i}\), and selects a model based on whether \(f_j > 1/2\). Under poisoning, the adversary may shift \(f_j\) to \(f'_j\) with \(|f'_j - f_j| \leq m/s\), independently at each round. The following theorem bounds the impact of these perturbations on the final output.

\setcounter{theoremUnboxed}{1}
\begin{theoremUnboxed}[Robustness of \retrofit\ under poisoning]\label{tpp}
Let \(\texttt{FinalModel} = \texttt{modified}(m_0,\hat{x})\) denote the selected model using decisions \(choice_j = \mathrm{sign}(f_j - \tfrac{1}{2})\). Let $output_1$ denote the model chosen under unperturbed frequencies \(f_j\), and $output_2$ under perturbed frequencies \(f'_j\) satisfying \(|f'_j - f_j| \leq m/s\). Then:
\begin{equation}\label{probadifferenceb}
    P\!\left(output_1 \neq output_2 \right) \;\leq\; b \frac{2m+1}{\sqrt{2s\pi}}.
\end{equation}
\end{theoremUnboxed}

{\bf Privacy guarantee}\label{cor:dp} In the A/B testing scenario,
BBoxER relies solely on users preference signals to perform model comparisons, and does not access prompts or outputs, privacy is therefore ensured by design. In this scenario, the user's preference signal isn’t considered private, whereas the input prompts and outputs are. These inputs / outputs may vary arbitrarily without changing BBoxER's output, provided that the global preference signal determining the comparison result remains the same.

\subsection{Robustness to extraction attacks}\label{sec:fairuse}

Prior work shows that training data can sometimes be extracted from models~\citep{privacyattack4,fu2}, even in privacy-preserving settings like federated learning. 
We formalize vulnerability to data extraction through identifiability: if one can reliably infer which dataset was used during fine-tuning, the process may be considered vulnerable.

Formally, let \( m_0 \) be fine-tuned on unrelated datasets \( D_1, \dots, D_n \). For text, "unrelated" can mean differing by one word out of five while remaining coherent. A dataset \( D \) with \( B \) bits may yield \( n \approx 2^{B/k} \) such variants by altering one bit per \(k\). We define vulnerability to data extraction as~follows:

\begin{definition}[simple formalization of vulnerability to data extraction]\label{def:fair}
    An algorithm $A$ producing $A(m_0,D)$ from model $m_0$ and dataset $D$ is vulnerable to data extraction for ${\cal D}=(D_1,\dots,D_n)$ (all distinct) if there exists a mapping $E$ such that  
    \begin{equation}
        \forall i\in \{1,\dots,n\},\ E(A(m_0,D_i))=i.\label{unfaireq}
    \end{equation}
\end{definition}

This criterion can be tested practically: generate \(n\) unrelated variants \(D_1,\dots,D_n\), fine-tune models on each \(D_i\), and check whether each model assigns the highest likelihood to its own training set. If so, it indicates vulnerability to data extraction. We argue this is unlikely for retrofitting, due to limited capacity to encode training data. \cref{th:unfair} in \cref{app:fair} formalizes this, leading to the corollary below:

\begin{corollary}\label{co:unfair}
    For given $\w,a,b$ s.t. $k_i\leq 2$ for all $i$, \retrofit\ is not vulnerable to extraction for ${\cal D}=(D_1,\dots,D_n)$ if
    $n> 2^b$. More generally, \retrofit{} is not vulnerable for ${\cal D}$ if $n>\prod_i k_i$.
\end{corollary}

Thus, if the number of alternative datasets exceeds that of distinguishable outputs, no extractor can identify the training data. Since \(2^b\) grows only exponentially with budget \(b\), while the number of unrelated textual variants grows exponentially with dataset size (e.g., \(>400\)B bits for Wikipedia articles, \(>20\)M for \textit{War and Peace}), extraction is very unlikely in practice.

\section{Experimental results: Learning with Privacy and Robustness}\label{sec:xp}
The theoretical bounds in~\cref{sec:generalization} establish that comparison based BBOs in the framework of \retrofit\ offer strong generalization and robustness guarantees. However, theoretical guarantees alone do not ensure that these methods are practically effective—only empirical evidence can reveal whether they actually lead to meaningful learning.
To validate the capacity of \retrofit{} as an LLM retrofitting method, we focus on verifiable domains, like math, where performance improvements are both measurable and interpretable. We experiment with budgets $b$ determined theoretically for guaranteed generalization on the test set and we empirically examine how performance changes as the budget exceeds these thresholds.

\begin{figure*}[ht!]
\centering
\begin{minipage}[c]{0.1\textwidth}
    \centering
    \includegraphics[width=\textwidth]{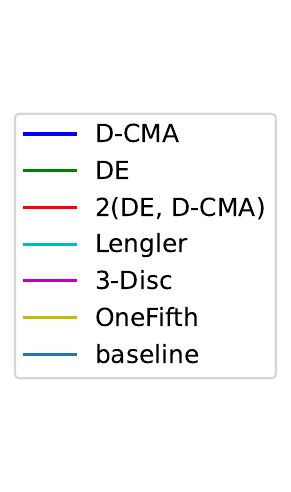}
\end{minipage}\hfill
\begin{minipage}[c]{0.44\textwidth}
\includegraphics[width=\textwidth]{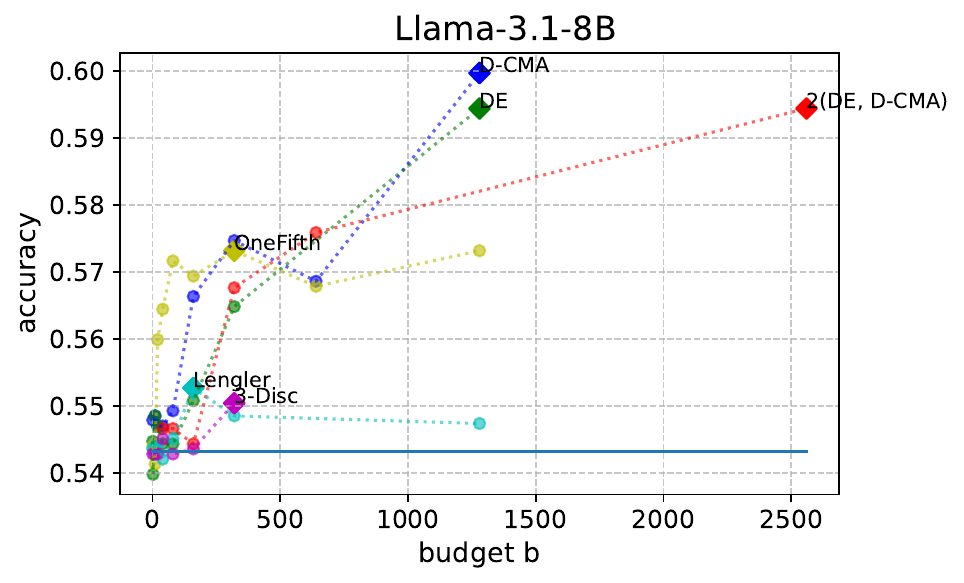}
\end{minipage}\hfill
\begin{minipage}[c]{0.44\textwidth}
\includegraphics[width=\textwidth]{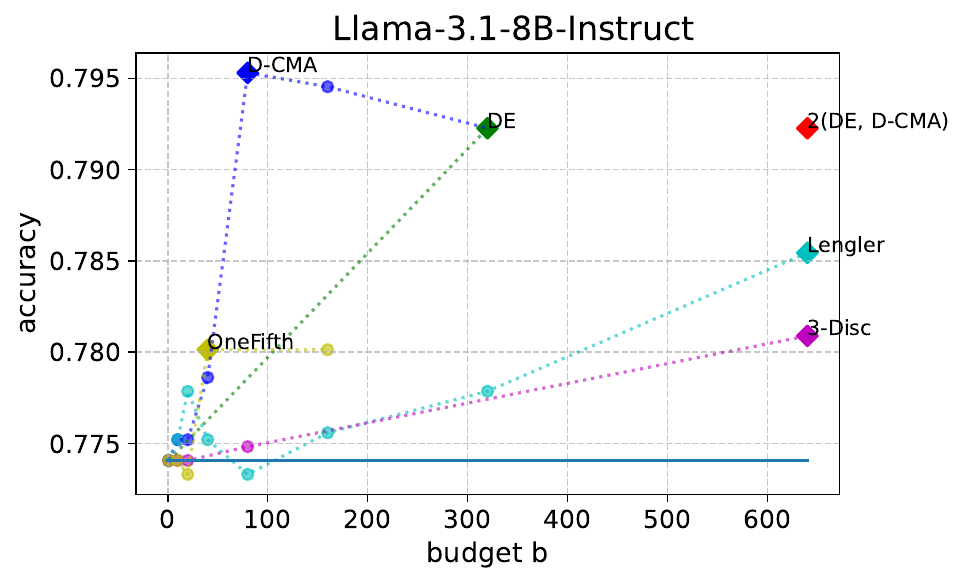}
\end{minipage}
\caption{\label{xplama} {{\bf Experiment 1 %
with Llama3.1-8B on GSM8K.} Accuracy (exact match) on GSM8K-test as a function of budget. {\bf Left:} Base. {\bf Right:}  Instruct Model.
\DCMA{} improves from 54.4\% to 60.0\% on base and from 77.4\% to 79.5\% on  instruct. Notably, it improves even for budgets above those computed from our theory ($b>150$).
}}
\end{figure*}

\subsection{Experiments}

{\bf Experimental setup.} We deploy a family of BBO algorithms with finite branching factors, described in~\cref{sec:bbo}, using the default parameters in Nevergrad.
We conduct experiments on both  base and instruct variants. Retrofitting the base models allows us to verify that learning indeed occurs, while targeting the instruct models demonstrates that our method can yield further improvements even on top of models already fine-tuned for reasoning tasks. All experimental details can be found in~\cref{appendix:expdetails}.

{\bf Datasets.} Most of our experiments deploy \retrofit\ on GSM8K~\cite{cobbe2021training}, a dataset composed of short grade-school math problems, consisting of $7,473$ problems for training and $1,319$ for testing. We train on the
train split and test on  GSM8K-test (in distribution) with standard 8-shot methodology (unless otherwise stated), in addition to various other benchmark datasets (out of distribution): MATH500 (4-shots)~\cite{hendrycks2021measuring}, Hellaswag (0-shots)~\cite{zellers2019hellaswag}, GSM+ test\_mini (8-shots)~\cite{gsmplus}, ARC Easy (0-shots) and ARC Challenge (0-shots)~\cite{allenai:arc}, AMC23 (0-shots). All benchmarks are evaluated with exact match using greedy decoding. The theoretical bounds in~\cref{sec:generalization} allow us to precisely determine the budget $b$ that guarantees in-domain generalization (as a function of dataset size): for GSM8K-train the permissible budget is between $100$-$150$ steps and approximately $3$ times larger for Bet-and-run variants like~\Triple.

{\bf Choices for parameter updates.} A key design choice in our framework is choosing which parameters the BBO algorithm should modify. To explore this, we experiment with a couple of  multiplicative update schemes, across different layers of the model: the outer (unembedding) layer (Experiment 1), $Q$ matrices of attention layers (Experiment 2)  or all normalization layers (Experiment 3). We distinguish between {\em low-rank (LoRA)} - rank-$1$ in fact (Experiment 1),  {\em broadcast} (Experiment 2) and {\em full} (Experiment 3) updates; {\em LoRA} multiplies the weight matrix by a low rank matrix, {\em broadcast} multiplies pointwise by a matrix with identical rows and {\em full} multiplies by an update matrix directly. For full details, refer to~\cref{concrete}. We optimize a small number of parameters, between $\approx 4K$ and $\approx 266K$, as BBO algorithms are known to scale poorly unless the problem’s effective dimensionality is much lower~\cite{nesterov2017random}, which is likely the case here.

\textbf{Experiment 1 (rank-1, output layer updates).} In a first series of experiments on Llama3.1-8B (resp. Llama3.1-8B-Instruct) we test \retrofit\ performance on GSM8K-test, mostly with 1 seed (sometimes 2).  \cref{xplama} illustrates the results: while \DCMA\ was our favorite choice, we ran other BBO algorithms to validate the robustness of the \retrofit\ approach.

\begin{table}[t]
\caption{\label{tab} {\bf Experiment 2}. Performance difference (after$-$before) \retrofit{} with Llama3.1-8B (Base) on GSM8K-train, evaluated on GSM8K-test (ID), (OOD) SVAMP and GSM+. \textcolor{MyRed}{Red: detrimental runs,  $<-1\%$.} \textcolor{MyGreen}{Green: beneficial runs, $>1\%$.} \textcolor{cyan}{Blue rows are bet-and-run cases}.}
\begin{center}
\begin{small}
\begin{tabular}{|l|c|c|c|}
\hline
BBO algorithm       &  GSM8K &  SVAMP &  GSM+\\
\hline 
\cellcolor{lightgray} Base score  & \cellcolor{lightgray} 56.93 & \cellcolor{lightgray} 76.6 & \cellcolor{lightgray}51.34 \\

\hline \multicolumn{4}{|c|}{First attention layer, d=4096 }\\ 
\hline

 \OneFifth\   & \cellcolor{MyGreen} 1.64$\pm$ 1.28 & \cellcolor{MyRed} -1.0 $\pm$ 1.08 & 0.35 $\pm$ 0.50 \\

 \DCMA\   &  0.49 $\pm$ 0.19 & 0.65 $\pm$ 0.75 & 0.00 $\pm$ 0.02 \\ 

 \COLengler\  & -0.41 $\pm$ 0.41 & -0.15 $\pm$ 0.05 & -0.06 $\pm$ 0.02\\
\NgIohTuned\  & \cellcolor{MyGreen}  1.28 &  0.90  & 0.18\\
\rowcolor{cyan} \Triple\  & \cellcolor{MyGreen}  1.59 & \cellcolor{MyRed}  -2.10 & -0.18\\

\hline \multicolumn{4}{|c|}{All attention layers, d=131072}\\ \hline 
\OneFifth\    & \cellcolor{MyGreen}  1.13 &  -0.80 &0.19\\
\COLengler\   & \cellcolor{MyGreen}  1.36 &  -0.26 & 0.29\\
\rowcolor{cyan} \Triple\   & \cellcolor{MyGreen}  2.27 & -0.10 &  \cellcolor{MyGreen}  1.01\\
\rowcolor{cyan} \Triple\ b300   & \cellcolor{MyGreen}  {\bf 7.12} &  -0.90 &  \cellcolor{MyGreen}  {\bf 4.99}\\
\rowcolor{cyan} \Triple\ b450  & \cellcolor{MyGreen}  5.61 &  -0.20  & \cellcolor{MyGreen}  4.85\\
\rowcolor{cyan} \Triple\ b600   & \cellcolor{MyGreen}  5.61 &  -0.20  & \cellcolor{MyGreen}  4.85\\
\hline \hline
\multicolumn{4}{|c|}{ All attention layers, no few-shot, d=131072}\\ 
\hline
\cellcolor{lightgray} Base score  & \cellcolor{lightgray} 20.31 & \cellcolor{lightgray} 58.30 & \cellcolor{lightgray} 19.06 \\
\hline 
\DCMA &	+0.38±0.06 &	+0.89±1.10 &	+0.25±0.09 \\
\OneFifth &	+0.58±0.38	& +0.56±0.16 &	+0.31±0.20 \\
\rowcolor{cyan} \Triple\   & \cellcolor{MyGreen}  1.97  & \cellcolor{MyGreen}  3.30 & 0.04\\
\hline
\end{tabular}
\end{small}
\end{center}
\end{table}

\textbf{Experiment 2 (broadcast updates on attention layers).} In this setting, we again train a range of BBO algorithms on the GSM8K training set on Llama3.1-8B (Base)  
to explore out-of-distribution generalization capabilities of \retrofit\ with evaluations on SVAMP~\citep{patel-etal-2021-nlp} and  GSM+\citep{gsmplus}, in addition to GSM8K-test. We experiment with {\em broadcast} updates on $Q$-matrices of the 1st, 8th, or all attention layers. We have added bet-and-run variants which enjoy particularly strong generalization guarantees (see ~\cref{sec:generalization}).

\cref{tab} shows the results for different settings. Each line without std.dev. ($\pm$) corresponds to one single run, and all runs use a budget $b=150$ except otherwise specified; on the first attention layer, OneFifth was run with 3 seeds and D-CMA and CoLengler with 2 seeds each.
Except for the ``no few-shot'' line, the default $8$-shot evaluation is used on GSM8K. The results demonstrate the stability of \retrofit{} in terms of learning capabilities, and a notable performance improvement of up to $7\%$ in distribution  (GSM8K test) and $5\%$ in out-of-distribution transfer to GSM+. Additional results %
are reported in \cref{fulltable} in~\cref{sec:addexps}.

\begin{table*}[t!]{
\caption{{ {\bf Experiment 3 (full update, normalization layers)} Comparison of model performance of BBoxER run with D-CMA across various benchmarks, 5 seeds.}
}\label{tab:various_gsm8k}
\centering
\tabcolsep=0.1cm
\resizebox{1.\linewidth}{!}{
    \begin{tabular}{c@{\hskip 0.2in}  c@{\hskip 0.2in} cccccc}
        \multirow{2}{*}{Models} & \multicolumn{1}{c}{ID} & \multicolumn{6}{c}{OOD}\\
        \cmidrule(rr){2-2} \cmidrule(rr){3-8}
        & GSM8K & GSM+ & MATH500 & ARC Easy & ARC Challenge & AMC23 & Hellaswag \\
        \midrule
        Llama3.1-8B & 54.97 & 36.50 & 20.20 & 83.04 & 55.19 & 0.00 & 80.71 \\
        Llama3.1-8B-BBoxER (b=150) & 55.24$\pm$0.92 & 37.28$\pm$0.48 & 19.72$\pm$0.68 & 83.05$\pm$0.12 & 54.97$\pm$0.12 & 4.50$\pm$2.92 & 80.67$\pm$0.05 \\
        Llama3.1-8B-BBoxER (b=300) & 56.65$\pm$0.26 & 37.88$\pm$0.33 & 20.12$\pm$0.87 & 83.16$\pm$0.13 & 55.09$\pm$0.24 & 4.00$\pm$3.39 & 80.65$\pm$0.08 \\
        \midrule
        Llama3.1-8B-Instruct & 77.26 & 54.37 & 37.60 & 79.62 & 55.45 & 22.50 & 79.90 \\
        Llama3.1-8B-Instruct-BBoxER (b=150) & 77.79$\pm$0.41 & 55.11$\pm$0.15 & 37.16$\pm$0.53 & 79.16$\pm$0.32 & 55.48$\pm$0.17 & 25.00$\pm$3.54 & 79.98$\pm$0.05 \\
        Llama3.1-8B-Instruct-BBoxER (b=300) & 78.57$\pm$0.41 & 55.33$\pm$0.53 & 37.56$\pm$0.51 & 79.64$\pm$0.12 & 55.61$\pm$0.15 & 21.00$\pm$3.00 & 79.94$\pm$0.04 \\
        \midrule
        Qwen-2.5-3B-instruct & 79.98 & 62.29 & 41.00 & 72.39 & 47.38 & 32.50 & 74.94 \\
        Qwen-2.5-3B-instruct-BBoxER (b=150) & 82.90$\pm$0.37 & 62.27$\pm$0.45 & 42.48$\pm$0.41 & 72.57$\pm$0.19 & 47.59$\pm$0.31 & 38.00$\pm$2.45 & 75.00$\pm$0.06 \\
        Qwen-2.5-3B-instruct-BBoxER (b=300) & 83.55$\pm$0.36 & 61.94$\pm$0.64 & 41.32$\pm$0.95 & 72.62$\pm$0.17 & 47.76$\pm$0.22 & 36.50$\pm$3.00 & 75.02$\pm$0.09 \\
        \bottomrule
    \end{tabular}}
}
\end{table*}

\begin{table}[h]
\caption{\label{tab:exp4}
{{\bf Experiment 4 - \retrofit{} is robust to MIA}. 
Comparison of \retrofit{}, SFT and DP-SFT in the setting of Experiment 3 with Llama3.1-8B-Instruct on GSM8K. Runs with comparable test accuracy are highlighted in the same color to compare robustness at similar accuracy. We see that both finetuning and DP-finetuning approaches lead to NLL metrics at least an order of magnitude larger than comparable \retrofit{} runs.}}
\centering
\setlength{\tabcolsep}{0.3em}
\centering
\begin{tabular}{lcccc}
\toprule
Method & Diff-Full & Diff-CoT-a & $\text{AUC}_{\text{\tiny WBC}}$ & Acc \\ 
\midrule
ft-5epochs         & {\scriptsize 2.25e-01}  & {\scriptsize 5.06e-01} & {\scriptsize 0.797}  & {\scriptsize +7.88}\\
\hline
\rowcolor{yellow!30} %
ft-2epochs         & {\scriptsize 2.53e-01}  & {\scriptsize 3.99e-01} & {\scriptsize 0.608}  & {\scriptsize +3.79}\\
\hline
\rowcolor{blue!20}
ft-norm-5epochs    & {\scriptsize 6.00e-03}  & {\scriptsize 4.62e-02} & {\scriptsize \textbf{0.510}}  & {\scriptsize +0.45}\\
\hline
\rowcolor{green!50!black!20}
DP-AdamW-eps=8     & {\scriptsize 8.66e-01}  & {\scriptsize 2.17e-01} & {\scriptsize 0.540}  & {\scriptsize +2.05}\\
\hline
\rowcolor{blue!20}
\retrofit-norm-b=150  & {\scriptsize \textbf{7.19e-04}}  & {\scriptsize \textbf{2.90e-03}} & {\scriptsize 0.515}  & {\scriptsize +0.15}\\
\hline
\rowcolor{green!50!black!20}
\retrofit-norm-b=300  & {\scriptsize \textbf{1.01e-03}}  & {\scriptsize \textbf{3.68e-03}} & {\scriptsize \textbf{0.499}}  & {\scriptsize +1.97}\\
\hline
\rowcolor{yellow!30} %
\retrofit-norm-b=1200 \hspace{5pt} & {\scriptsize \textbf{6.33e-03}}  & {\scriptsize \textbf{1.17e-02}} & {\scriptsize \textbf{0.505}} & {\scriptsize +3.87}\\
\bottomrule
\end{tabular} %
\end{table}

\textbf{Experiment 3 (full update of normalization layers).} In this experiment, \DCMA\ is used to directly optimize all normalization layers of Llama3.1-8B(-Instruct) and Qwen2.5-3B-Instruct (no low-rank and no broadcast involved) on GSM8K. 
As in Experiment 2, we test both in distribution (GSM8K-test), and on various other benchmark datasets (out of distribution). We report the results averaged over 5 different runs with the standard deviation. 
The results (\cref{tab:various_gsm8k}) indicate that running \retrofit\ improves in distribution performance and generally preserves it on very different tasks (e.g., Hellaswag). Additional experiments are described in~\cref{sec:addexps}: BetAndRun results are reported in \cref{tab:various_gsm8k_BetAndRun}, and experiments using a larger training set can be found in \cref{tab:various_gsm8k_math} in~\cref{sec:addexps}.

\textbf{Experiment 4 (Robustness to data extraction).} To assess the robustness of \retrofit{} against membership inference attacks (MIAs), we begin from the observation that most MIAs rely on the negative log-likelihood (NLL) of training samples as their primary signal to infer membership~\cite{fu2023practical,carlini2022membership,privacyattack4} - see also~\cref{appendix:expdetails} for an exposition. Accordingly, we propose to measure the average of absolute differences in negative log-likelihood (NLL) between the base and the updated model for each training sample.
\changes{We also include the AUC score of WBC~\cite{chen2026window} that uses the NLL on the token level.}
We consider GSM8K as our training dataset, Llama-3.1-8B as our base model, and compute three metrics: Diff-full (NLL computed over the entire prompt), Diff-CoT-a (NLL computed over only the CoT and the answer), and the relative average exact match improvement on GSM8K-test from the base model.
For a fair comparison, we evaluate \retrofit{} alongside standard fine-tuning and DP-AdamW, training all methods on GSM8K under different settings-varying budget, epochs, and parameter subsets (full details in~\cref{appendix:expdetails})—so that models are compared at similar utility.
Our results, reported in~\cref{tab:exp4}, demonstrate that at matched test accuracy, the NLL shifts under \retrofit{} are orders of magnitude smaller than under fine-tuning or DP-AdamW. This provides empirical confirmation of our theoretical claim in~\cref{sec:fairuse}, when compared at matched accuracy: \retrofit{} is significantly more robust to extraction and membership inference attacks.

\subsection{Discussion of experimental results}
\textbf{Robustness and statistical validation.} 
Our experiments demonstrate that \retrofit{} achieves statistically significant performance gains without overfitting across a range of BBO algorithms (see P-value validation  in~\cref{pval}), these gains are also observed on Instruct models. Consistent with theory, the Bet-and-Run variant \Triple{} further improves performance and allows for larger budgets as shown in experiment 2. More discussion is provided in~\cref{sec:addexps}.

\textbf{Validation of transfer.} The compact optimization trace in \retrofit{} reduces the risk of overfitting to spurious patterns or idiosyncrasies in the training data. As a result, we may expect improved generalization and stronger performance on unseen datasets.
In Experiments 2 and 3, indeed, we observe an improvement on GSM+ when evaluating our model trained exclusively on GSM8K. This positive transfer, in particular with \Triple, is especially noticeable when fine-tuning all the attention layers, highlighting the robustness of \retrofit. GSM+ is an interesting and challenging evaluation benchmark, as it modifies GSM8K questions to make them trickier by introducing various scenarios aimed at testing model robustness. In contrast, SVAMP presents simpler questions, typically involving only a single operation. As a result, improvements on SVAMP are more likely to reflect better numerical reasoning rather than enhanced problem comprehension.

\textbf{Computational cost.} \rebuttal{A limitation of our approach is that each iteration requires a full pass over the dataset to generate a comparison result. However, \retrofit{} remains} memory-efficient, as it relies solely on LLM inference without the need to compute and store gradients, or large optimizer states. This allows us to fit a large batch size during training. The overhead introduced by the internal state of the BBO algorithm is negligible compared to the computational cost of LLM inference.%

\section{Conclusion}\label{sec:conclusion}
This paper introduced \retrofit{}, a principled post-training method that leverages comparison-based BBO algorithms for retrofitting. By compressing the training set into a comparison trace, measured in only a few bits, BBoxER offers strong theoretical guarantees on generalization and robustness, while never exposing raw training data. Despite this compression and the limited number of optimization iterations, BBoxER empirically demonstrates effective learning and transfer on GSM8K with contemporary LLMs, as well as empirical robustness to membership inference attacks. Theoretical insights are reinforced by experimental results, especially with Bet-and-Run algorithm variants. Importantly, this work is the first to bridge the LLM training landscape with the full BBO zoology. This connection, grounded in formal analysis,  paves the way for safer, more robust, and adaptable language models, and opens the door to broader adoption of remote A/B testing or post-training directly on user satisfaction without any leakage of the underlying data.

\section*{Acknowledgments}
JK thanks the Simons Foundation for support through the Collaborative Grant “The Physics of Learning and Neural Computation”.

\bibliographystyle{unsrtnat}
\bibliography{overfit,julia}

\begin{thebibliography}{123}
\providecommand{\natexlab}[1]{#1}
\providecommand{\url}[1]{\texttt{#1}}
\expandafter\ifx\csname urlstyle\endcsname\relax
  \providecommand{\doi}[1]{doi: #1}\else
  \providecommand{\doi}{doi: \begingroup \urlstyle{rm}\Url}\fi

\bibitem[Kaplan et~al.(2020)Kaplan, McCandlish, Henighan, et~al.]{kaplan2020scaling}
Jared Kaplan, Sam McCandlish, Tom Henighan, et~al.
\newblock Scaling laws for neural language models.
\newblock \emph{arXiv:2001.08361}, 2020.

\bibitem[Wei et~al.(2022{\natexlab{a}})Wei, Tay, Bommasani, et~al.]{wei2022emergent}
Jason Wei, Yi~Tay, Rishi Bommasani, et~al.
\newblock Emergent abilities of large language models.
\newblock \emph{Transactions on Machine Learning Research}, 2022{\natexlab{a}}.

\bibitem[Haim et~al.(2022)Haim, Vardi, Yehudai, Shamir, and Irani]{privacyattack3}
Niv Haim, Gal Vardi, Gilad Yehudai, Ohad Shamir, and Michal Irani.
\newblock Reconstructing training data from trained neural networks.
\newblock \emph{Advances in Neural Information Processing Systems}, 35:\penalty0 22911--22924, 2022.

\bibitem[Carlini et~al.(2021)Carlini, Tramer, Wallace, et~al.]{privacyattack4}
Nicholas Carlini, Florian Tramer, Eric Wallace, et~al.
\newblock Extracting training data from large language models.
\newblock In \emph{30th USENIX security symposium}, pages 2633--2650, 2021.

\bibitem[Carlini et~al.(2024)Carlini, Paleka, Dvijotham, Steinke, Hayase, Cooper, Lee, Jagielski, Nasr, Conmy, Wallace, Rolnick, and Tram\`{e}r]{carlini2024stealing}
Nicholas Carlini, Daniel Paleka, Krishnamurthy~Dj Dvijotham, Thomas Steinke, Jonathan Hayase, A.~Feder Cooper, Katherine Lee, Matthew Jagielski, Milad Nasr, Arthur Conmy, Eric Wallace, David Rolnick, and Florian Tram\`{e}r.
\newblock Stealing part of a production language model.
\newblock In Ruslan Salakhutdinov, Zico Kolter, Katherine Heller, et~al., editors, \emph{41st International Conference on Machine Learning}, volume 235, pages 5680--5705, 2024.

\bibitem[Zhu et~al.(2019)Zhu, Liu, and Han]{privacyattack1}
Ligeng Zhu, Zhijian Liu, and Song Han.
\newblock Deep leakage from gradients.
\newblock \emph{Advances in neural information processing systems}, 32, 2019.

\bibitem[Hitaj et~al.(2017)Hitaj, Ateniese, and Perez-Cruz]{privacyattack2}
Briland Hitaj, Giuseppe Ateniese, and Fernando Perez-Cruz.
\newblock Deep models under the gan: Information leakage from collaborative deep learning.
\newblock In \emph{ACM SIGSAC Conference on Computer and Communications Security}, pages 603--618, 2017.

\bibitem[Wan et~al.(2023)Wan, Wallace, Shen, and Klein]{wan2023poisoning}
Alexander Wan, Eric Wallace, Sheng Shen, and Dan Klein.
\newblock Poisoning language models during instruction tuning.
\newblock In Andreas Krause, Emma Brunskill, Kyunghyun Cho, Barbara Engelhardt, Sivan Sabato, and Jonathan Scarlett, editors, \emph{40th International Conference on Machine Learning}, volume 202, pages 35413--35425, 2023.

\bibitem[Miranda et~al.(2025)Miranda, Ruzzetti, Santilli, Zanzotto, Brati{\`e}res, and Rodol{\`a}]{miranda2025preserving}
Michele Miranda, Elena~Sofia Ruzzetti, Andrea Santilli, Fabio~Massimo Zanzotto, S{\'e}bastien Brati{\`e}res, and Emanuele Rodol{\`a}.
\newblock Preserving privacy in large language models: A survey on current threats and solutions.
\newblock \emph{Transactions on Machine Learning Research}, 2025.

\bibitem[Hutter et~al.(2019)Hutter, Kotthoff, and Vanschoren]{autoMLbook2019}
Frank Hutter, Lars Kotthoff, and Joaquin Vanschoren.
\newblock \emph{Automated Machine Learning: Methods, Systems, Challenges}.
\newblock Springer, 1st edition, 2019.

\bibitem[Feurer and Hutter(2019)]{hpo-hutter2019}
Matthias Feurer and Frank Hutter.
\newblock Hyperparameter optimization.
\newblock In Frank Hutter, Lars Kotthoff, and Joaquin Vanschoren, editors, \emph{Automated Machine Learning: Methods, Systems, Challenges}, pages 3--38. Springer, 2019.

\bibitem[White et~al.(2023)White, Safari, Sukthanker, Ru, Elsken, Zela, Dey, and Hutter]{white2023NASinsights}
Colin White, Mahmoud Safari, Rhea Sukthanker, Binxin Ru, Thomas Elsken, Arber Zela, Debadeepta Dey, and Frank Hutter.
\newblock Neural architecture search: Insights from 1000 papers.
\newblock \emph{arXiv:2301.08727}, 2023.

\bibitem[Hansen and Ostermeier(1996)]{seminalCMAES}
N.~Hansen and A.~Ostermeier.
\newblock Adapting arbitrary normal mutation distributions in evolution strategies: the covariance matrix adaptation.
\newblock In \emph{IEEE International Conference on Evolutionary Computation}, pages 312--317, 1996.

\bibitem[Hansen and Ostermeier(2003)]{CMA}
Nikolaus Hansen and Andreas Ostermeier.
\newblock Completely derandomized self-adaptation in evolution strategies.
\newblock \emph{Evolutionary Computation}, 11\penalty0 (1), 2003.

\bibitem[Hansen(2023)]{tutorialCMAES}
Nikolaus Hansen.
\newblock {The {CMA} Evolution Strategy: A Tutorial - V2}.
\newblock \emph{arXiv:1604.00772}, 2023.

\bibitem[Storn and Price(1997)]{de}
Rainer Storn and Kenneth Price.
\newblock {Differential Evolution - A Simple and Efficient Heuristic for Global Optimization over Continuous Spaces}.
\newblock \emph{Journal of Global Optimization}, 11\penalty0 (4):\penalty0 341--359, 1997.

\bibitem[Kennedy and Eberhart(1995)]{pso}
James Kennedy and Russell~C. Eberhart.
\newblock {Particle Swarm Optimization}.
\newblock In \emph{IEEE International Conference on Neural Networks}, pages 1942--1948, 1995.

\bibitem[Mo{\v{c}}kus(1975)]{movckus1975bayesian}
J.~Mo{\v{c}}kus.
\newblock On bayesian methods for seeking the extremum.
\newblock In G.~I. Marchuk, editor, \emph{Optimization Techniques IFIP Technical Conference Novosibirsk, July 1--7, 1974}, pages 400--404, Berlin, Heidelberg, 1975. Springer Berlin Heidelberg.
\newblock ISBN 978-3-540-37497-8.

\bibitem[Mockus(1981)]{mockus1981bayesian}
Jonas Mockus.
\newblock The bayesian approach to global optimization.
\newblock In \emph{10th IFIP Conference System Modeling and Optimization}, pages 473--481, 1981.

\bibitem[Garnett(2023)]{garnettbayesoptbook2023}
Roman Garnett.
\newblock \emph{{Bayesian Optimization}}.
\newblock Cambridge University Press, 2023.

\bibitem[Salimans et~al.(2017)Salimans, Ho, Chen, Sidor, and Sutskever]{openAI-ES4DeepRL2017}
Tim Salimans, Jonathan Ho, Xi~Chen, Szymon Sidor, and Ilya Sutskever.
\newblock Evolution strategies as a scalable alternative to reinforcement learning.
\newblock \emph{arXiv:1703.03864}, 2017.

\bibitem[Shu et~al.(2023)Shu, Wang, Zhu, et~al.]{shu2023exploitability}
Manli Shu, Jiongxiao Wang, Chen Zhu, et~al.
\newblock On the exploitability of instruction tuning.
\newblock In Alice Oh, Tristan Naumann, Amir Globerson, et~al., editors, \emph{36th Annual Conference on Neural Information Processing Systems}, 2023.

\bibitem[Huang et~al.(2024)Huang, Gupta, Xia, Li, and Chen]{huang2024catastrophic}
Yangsibo Huang, Samyak Gupta, Mengzhou Xia, Kai Li, and Danqi Chen.
\newblock Catastrophic jailbreak of open-source {LLM}s via exploiting generation.
\newblock In \emph{12th International Conference on Learning Representations}, 2024.

\bibitem[Yang et~al.(2024)Yang, Zhang, Xu, Lu, Heng, and Lam]{llmgeneralize}
Haoran Yang, Yumeng Zhang, Jiaqi Xu, Hongyuan Lu, Pheng-Ann Heng, and Wai Lam.
\newblock Unveiling the generalization power of fine-tuned large language models.
\newblock In Kevin Duh, Helena Gomez, and Steven Bethard, editors, \emph{Proceedings of the 2024 Conference of the North American Chapter of the Association for Computational Linguistics: Human Language Technologies (Volume 1: Long Papers)}, pages 884--899, Mexico City, Mexico, {June} 2024. Association for Computational Linguistics.

\bibitem[Wang et~al.(2024)Wang, Wu, Chen, Vorobeychik, and Xiao]{wang2024rankpoison}
Jiongxiao Wang, Junlin Wu, Muhao Chen, Yevgeniy Vorobeychik, and Chaowei Xiao.
\newblock {RLHFP}oison: Reward poisoning attack for reinforcement learning with human feedback in large language models.
\newblock In Lun-Wei Ku, Andre Martins, and Vivek Srikumar, editors, \emph{62nd Annual Meeting of the Association for Computational Linguistics}, pages 2551--2570, 2024.

\bibitem[Baumg{\"a}rtner et~al.(2024)Baumg{\"a}rtner, Gao, Alon, and Metzler]{baumgrtner2024bestofvenom}
Tim Baumg{\"a}rtner, Yang Gao, Dana Alon, and Donald Metzler.
\newblock Best-of-venom: Attacking {RLHF} by injecting poisoned preference data.
\newblock In \emph{First Conference on Language Modeling}, 2024.

\bibitem[Rando and Tram{\`e}r(2024)]{rando2024universal}
Javier Rando and Florian Tram{\`e}r.
\newblock Universal jailbreak backdoors from poisoned human feedback.
\newblock In \emph{12th International Conference on Learning Representations}, 2024.

\bibitem[Pathmanathan et~al.(2025)Pathmanathan, Chakraborty, Liu, Liang, and Huang]{poisoningthreat}
Pankayaraj Pathmanathan, Souradip Chakraborty, Xiangyu Liu, Yongyuan Liang, and Furong Huang.
\newblock Is poisoning a real threat to {DPO}? maybe more so than you think.
\newblock \emph{AAAI Conference on Artificial Intelligence}, 39\penalty0 (26):\penalty0 27556--27564, 2025.

\bibitem[Faruqui et~al.(2015)Faruqui, Dodge, Jauhar, Dyer, Hovy, and Smith]{retrofitting2015}
Manaal Faruqui, Jesse Dodge, Sujay~Kumar Jauhar, Chris Dyer, Eduard Hovy, and Noah~A Smith.
\newblock Retrofitting word vectors to semantic lexicons.
\newblock In \emph{Proceedings of the 2015 Conference of the North American Chapter of the Association for Computational Linguistics: Human Language Technologies}, pages 1606--1615, 2015.

\bibitem[Videau et~al.(2024{\natexlab{a}})Videau, Zameshina, Leite, Najman, Schoenauer, and Teytaud]{arxivtelo}
Mathurin Videau, Mariia Zameshina, Alessandro Leite, Laurent Najman, Marc Schoenauer, and Olivier Teytaud.
\newblock Evolutionary retrofitting.
\newblock \emph{arXiv:2410.11330}, 2024{\natexlab{a}}.

\bibitem[Shao et~al.(2024)Shao, Wang, Zhu, et~al.]{grpo}
Zhihong Shao, Peiyi Wang, Qihao Zhu, et~al.
\newblock {DeepSeekMath: Pushing the Limits of Mathematical Reasoning in Open Language Models}.
\newblock \emph{arXiv:2402.03300}, 2024.

\bibitem[Rapin and Teytaud(2018)]{nevergrad}
J.~Rapin and O.~Teytaud.
\newblock {Nevergrad - A gradient-free optimization platform}.
\newblock \url{github.com/FacebookResearch/Nevergrad}, 2018.

\bibitem[Campi and Garatti(2023)]{campi2024compressiongeneralizationlearning}
Marco~C. Campi and Simone Garatti.
\newblock Compression, generalization and learning.
\newblock \emph{Journal of Machine Learning Research}, 24\penalty0 (339):\penalty0 1--74, 2023.

\bibitem[Jiang et~al.(2022)Jiang, Liu, Eysenbach, Kolter, and Finn]{compressionbasedlearning}
Yiding Jiang, Evan Liu, Benjamin Eysenbach, J.~Zico Kolter, and Chelsea Finn.
\newblock Learning options via compression.
\newblock In S.~Koyejo, S.~Mohamed, A.~Agarwal, et~al., editors, \emph{Advances in Neural Information Processing Systems}, volume~35, pages 21184--21199, 2022.

\bibitem[Lotfi et~al.(2024)Lotfi, Finzi, Kuang, Rudner, Goldblum, and Wilson]{compressionllm}
Sanae Lotfi, Marc~Anton Finzi, Yilun Kuang, Tim G.~J. Rudner, Micah Goldblum, and Andrew~Gordon Wilson.
\newblock Non-vacuous generalization bounds for large language models.
\newblock In Ruslan Salakhutdinov, Zico Kolter, Katherine Heller, et~al., editors, \emph{41st International Conference on Machine Learning}, volume 235, pages 32801--32818, 2024.

\bibitem[Vapnik and Chervonenkis(1971)]{vapnik1971uniform}
Vladimir~N. Vapnik and Alexey~Y. Chervonenkis.
\newblock On the uniform convergence of relative frequencies of events to their probabilities.
\newblock \emph{Theory of Probability \& Its Applications}, 16\penalty0 (2):\penalty0 264--280, 1971.

\bibitem[Bartlett and Mendelson(2002)]{bartlett2002rademacher}
Peter~L. Bartlett and Shahar Mendelson.
\newblock Rademacher and gaussian complexities: Risk bounds and structural results.
\newblock \emph{Journal of Machine Learning Research}, 3:\penalty0 463--482, 2002.

\bibitem[Zhou et~al.(2018)Zhou, Veitch, Austern, Adams, and Orbanz]{zhou2018non}
Wenda Zhou, Victor Veitch, Morgane Austern, Ryan~P Adams, and Peter Orbanz.
\newblock Non-vacuous generalization bounds at the imagenet scale: a pac-bayesian compression approach.
\newblock \emph{arXiv:1804.05862}, 2018.

\bibitem[Lotfi et~al.(2022)Lotfi, Finzi, Kapoor, Potapczynski, Goldblum, and Wilson]{LotfiPACcompression}
Sanae Lotfi, Marc Finzi, Sanyam Kapoor, Andres Potapczynski, Micah Goldblum, and Andrew~G Wilson.
\newblock Pac-bayes compression bounds so tight that they can explain generalization.
\newblock In S.~Koyejo, S.~Mohamed, A.~Agarwal, D.~Belgrave, K.~Cho, and A.~Oh, editors, \emph{Advances in Neural Information Processing Systems}, volume~35, pages 31459--31473, 2022.

\bibitem[Bousquet and Elisseeff(2000)]{as}
Olivier Bousquet and Andr\'{e} Elisseeff.
\newblock Algorithmic stability and generalization performance.
\newblock In T.~Leen, T.~Dietterich, and V.~Tresp, editors, \emph{Advances in Neural Information Processing Systems}, volume~13, 2000.

\bibitem[Egashira et~al.(2024)Egashira, Vero, Staab, He, and Vechev]{llmpoisoq}
Kazuki Egashira, Mark Vero, Robin Staab, Jingxuan He, and Martin Vechev.
\newblock Exploiting {LLM} quantization.
\newblock In A.~Globerson, L.~Mackey, D.~Belgrave, A.~Fan, U.~Paquet, J.~Tomczak, and C.~Zhang, editors, \emph{Advances in Neural Information Processing Systems}, volume~37, pages 41709--41732, 2024.

\bibitem[Feng and Tram\`{e}r(2024)]{feng2024privacy}
Shanglun Feng and Florian Tram\`{e}r.
\newblock Privacy backdoors: stealing data with corrupted pretrained models.
\newblock In \emph{41st International Conference on Machine Learning}, pages 13326--13364, 2024.

\bibitem[Steinhardt et~al.(2017)Steinhardt, Koh, and Liang]{steinhardt2017certified}
Jacob Steinhardt, Pang Wei~W Koh, and Percy~S Liang.
\newblock Certified defenses for data poisoning attacks.
\newblock In \emph{Advances in neural information processing systems}, pages 3517--3529, 2017.

\bibitem[Qiu et~al.(2025)Qiu, Gan, Hayes, Liang, Meyerson, Hodjat, and Miikkulainen]{qiu2025evolution}
Xin Qiu, Yulu Gan, Conor~F Hayes, Qiyao Liang, Elliot Meyerson, Babak Hodjat, and Risto Miikkulainen.
\newblock Evolution strategies at scale: Llm fine-tuning beyond reinforcement learning.
\newblock \emph{arXiv preprint arXiv:2509.24372}, 2025.

\bibitem[Sarkar et~al.(2025)Sarkar, Fellows, Duque, Letcher, Villares, Sims, Cope, Liesen, Seier, Wolf, et~al.]{sarkar2025evolution}
Bidipta Sarkar, Mattie Fellows, Juan~Agustin Duque, Alistair Letcher, Antonio~Le{\'o}n Villares, Anya Sims, Dylan Cope, Jarek Liesen, Lukas Seier, Theo Wolf, et~al.
\newblock Evolution strategies at the hyperscale.
\newblock \emph{arXiv preprint arXiv:2511.16652}, 2025.

\bibitem[Douglas(2006)]{douglas:06}
James Douglas.
\newblock \emph{Building adaptation}.
\newblock Routledge, 2006.

\bibitem[Dawson(2007)]{dawson:07}
Richard Dawson.
\newblock Re-engineering cities: a framework for adaptation to global change.
\newblock \emph{Philosophical Transactions of the Royal Society A: Mathematical, Physical and Engineering Sciences}, 365\penalty0 (1861):\penalty0 3085--3098, 2007.

\bibitem[Dixon and Eames(2013)]{dixon:13}
Tim Dixon and Malcolm Eames.
\newblock Scaling up: the challenges of urban retrofit.
\newblock \emph{Building research \& information}, 41\penalty0 (5):\penalty0 499--503, 2013.

\bibitem[van~der Vaart and Wellner(1996)]{VVW}
A.~van~der Vaart and J.A. Wellner.
\newblock \emph{Weak Convergence and Empirical Processes: With Applications to Statistics}.
\newblock Springer, 1996.

\bibitem[Vapnik(1995)]{vap95}
Vladimir~N. Vapnik.
\newblock \emph{The Nature of Statistical Learning Theory}.
\newblock Springer-Verlag, 1995.
\newblock ISBN 0-387-94559-8.

\bibitem[Kearns and Li(1993)]{malicious}
Michael Kearns and Ming Li.
\newblock Learning in the presence of malicious errors.
\newblock \emph{SIAM Journal on Computing}, 22\penalty0 (4):\penalty0 807--837, 1993.

\bibitem[Henderson et~al.(2023)Henderson, Li, Jurafsky, Hashimoto, Lemley, and Liang]{fu2}
Peter Henderson, Xuechen Li, Dan Jurafsky, Tatsunori Hashimoto, Mark~A. Lemley, and Percy Liang.
\newblock Foundation models and fair use.
\newblock \emph{Journal of Machine Learning Research}, 24\penalty0 (400):\penalty0 1--79, 2023.

\bibitem[Cobbe et~al.(2021)Cobbe, Kosaraju, Bavarian, et~al.]{cobbe2021training}
Karl Cobbe, Vineet Kosaraju, Mohammad Bavarian, et~al.
\newblock Training verifiers to solve math word problems.
\newblock \emph{arXiv:2110.14168}, 2021.

\bibitem[Hendrycks et~al.(2021)Hendrycks, Burns, Kadavath, Arora, Basart, Tang, Song, and Steinhardt]{hendrycks2021measuring}
Dan Hendrycks, Collin Burns, Saurav Kadavath, Akul Arora, Steven Basart, Eric Tang, Dawn Song, and Jacob Steinhardt.
\newblock Measuring mathematical problem solving with the {MATH} dataset.
\newblock In J.~Vanschoren and S.~Yeung, editors, \emph{Neural Information Processing Systems Track on Datasets and Benchmarks}, volume~1, 2021.

\bibitem[Zellers et~al.(2019)Zellers, Holtzman, Bisk, Farhadi, and Choi]{zellers2019hellaswag}
Rowan Zellers, Ari Holtzman, Yonatan Bisk, Ali Farhadi, and Yejin Choi.
\newblock Hellaswag: Can a machine really finish your sentence?
\newblock \emph{arXiv preprint arXiv:1905.07830}, 2019.

\bibitem[Li et~al.(2024)Li, Cui, Zhao, Kong, and Bi]{gsmplus}
Qintong Li, Leyang Cui, Xueliang Zhao, Lingpeng Kong, and Wei Bi.
\newblock {GSM}-{P}lus: A comprehensive benchmark for evaluating the robustness of {LLM}s as mathematical problem solvers.
\newblock In Lun-Wei Ku, Andre Martins, and Vivek Srikumar, editors, \emph{62nd Annual Meeting of the Association for Computational Linguistics}, pages 2961--2984, 2024.

\bibitem[Clark et~al.(2018)Clark, Cowhey, Etzioni, Khot, Sabharwal, Schoenick, and Tafjord]{allenai:arc}
Peter Clark, Isaac Cowhey, Oren Etzioni, Tushar Khot, Ashish Sabharwal, Carissa Schoenick, and Oyvind Tafjord.
\newblock Think you have solved question answering? try arc, the ai2 reasoning challenge.
\newblock \emph{arXiv:1803.05457v1}, 2018.

\bibitem[Nesterov and Spokoiny(2017)]{nesterov2017random}
Yurii Nesterov and Vladimir Spokoiny.
\newblock Random gradient-free minimization of convex functions.
\newblock \emph{Foundations of Computational Mathematics}, 17\penalty0 (2):\penalty0 527--566, 2017.

\bibitem[Patel et~al.(2021)Patel, Bhattamishra, and Goyal]{patel-etal-2021-nlp}
Arkil Patel, Satwik Bhattamishra, and Navin Goyal.
\newblock Are {NLP} models really able to solve simple math word problems?
\newblock In \emph{Conference of the North American Chapter of the Association for Computational Linguistics: Human Language Technologies}, pages 2080--2094, 2021.

\bibitem[Fu et~al.(2023)Fu, Wang, Gao, Liu, Li, and Jiang]{fu2023practical}
Wenjie Fu, Huandong Wang, Chen Gao, Guanghua Liu, Yong Li, and Tao Jiang.
\newblock Practical membership inference attacks against fine-tuned large language models via self-prompt calibration.
\newblock \emph{arXiv preprint arXiv:2311.06062}, 2023.

\bibitem[Carlini et~al.(2022)Carlini, Chien, Nasr, Song, Terzis, and Tramer]{carlini2022membership}
Nicholas Carlini, Steve Chien, Milad Nasr, Shuang Song, Andreas Terzis, and Florian Tramer.
\newblock Membership inference attacks from first principles.
\newblock In \emph{2022 IEEE symposium on security and privacy (SP)}, pages 1897--1914. IEEE, 2022.

\bibitem[Chen et~al.(2026)Chen, Du, Zhang, Kundu, Fleming, Ribeiro, and Li]{chen2026window}
Yuetian Chen, Yuntao Du, Kaiyuan Zhang, Ashish Kundu, Charles Fleming, Bruno Ribeiro, and Ninghui Li.
\newblock Window-based membership inference attacks against fine-tuned large language models.
\newblock \emph{arXiv preprint arXiv:2601.02751}, 2026.

\bibitem[Kearns and Schapire(1990)]{fat}
M.J. Kearns and R.E. Schapire.
\newblock Efficient distribution-free learning of probabilistic concepts.
\newblock In \emph{31st Annual Symposium on Foundations of Computer Science}, pages 382--391 vol.1, 1990.

\bibitem[Nagarajan and Kolter(2019)]{nagarajan2019uniform}
Vaishnavh Nagarajan and J~Zico Kolter.
\newblock Uniform convergence may be unable to explain generalization in deep learning.
\newblock \emph{Advances in Neural Information Processing Systems}, 32, 2019.

\bibitem[Li et~al.(2023{\natexlab{a}})Li, Ildiz, Papailiopoulos, and Oymak]{li2023transformeralgo}
Yingcong Li, Muhammed~Emrullah Ildiz, Dimitris Papailiopoulos, and Samet Oymak.
\newblock Transformers as algorithms: Generalization and stability in in-context learning.
\newblock In Andreas Krause, Emma Brunskill, Kyunghyun Cho, Barbara Engelhardt, Sivan Sabato, and Jonathan Scarlett, editors, \emph{40th International Conference on Machine Learning}, volume 202, pages 19565--19594, 2023{\natexlab{a}}.

\bibitem[Chen et~al.(2017)Chen, Liu, Li, Lu, and Song]{poisoningforbackdoor}
Xinyun Chen, Chang Liu, Bo~Li, Kimberly Lu, and Dawn Song.
\newblock Targeted backdoor attacks on deep learning systems using data poisoning.
\newblock \emph{arXiv:1712.05526}, 2017.

\bibitem[Alber et~al.(2025)Alber, Yang, Alyakin, et~al.]{medicalpoisoning}
Daniel Alber, Zihao Yang, Anton Alyakin, et~al.
\newblock Medical large language models are vulnerable to data-poisoning attacks.
\newblock \emph{Nature Medicine}, 31\penalty0 (2):\penalty0 618--626, 2025.

\bibitem[Christiano et~al.(2017)Christiano, Leike, Brown, Martic, Legg, and Amodei]{christiano2017rlhf}
Paul~F. Christiano, Jan Leike, Tom~B. Brown, Miljan Martic, Shane Legg, and Dario Amodei.
\newblock Deep reinforcement learning from human preferences.
\newblock In I.~Guyon, U.~Von Luxburg, S.~Bengio, H.~Wallach, R.~Fergus, S.~Vishwanathan, and R.~Garnett, editors, \emph{31st International Conference on Neural Information Processing Systems}, pages 4302--4310, 2017.

\bibitem[Rafailov et~al.(2023)Rafailov, Sharma, Mitchell, Manning, Ermon, and Finn]{rafailov2023dpo}
Rafael Rafailov, Archit Sharma, Eric Mitchell, Christopher~D Manning, Stefano Ermon, and Chelsea Finn.
\newblock Direct preference optimization: Your language model is secretly a reward model.
\newblock In A.~Oh, T.~Naumann, A.~Globerson, K.~Saenko, M.~Hardt, and S.~Levine, editors, \emph{Advances in Neural Information Processing Systems}, volume~36, pages 53728--53741, 2023.

\bibitem[Yan et~al.(2024)Yan, Yadav, Li, Chen, Tang, Wang, Srinivasan, Ren, and Jin]{yan2024backdooring}
Jun Yan, Vikas Yadav, Shiyang Li, Lichang Chen, Zheng Tang, Hai Wang, Vijay Srinivasan, Xiang Ren, and Hongxia Jin.
\newblock Backdooring instruction-tuned large language models with virtual prompt injection.
\newblock In Kevin Duh, Helena Gomez, and Steven Bethard, editors, \emph{Conference of the North American Chapter of the Association for Computational Linguistics: Human Language Technologies}, pages 6065--6086, 2024.

\bibitem[Zhong et~al.(2023)Zhong, Huang, Wettig, and Chen]{zhong2023poisoning}
Zexuan Zhong, Ziqing Huang, Alexander Wettig, and Danqi Chen.
\newblock Poisoning retrieval corpora by injecting adversarial passages.
\newblock In Houda Bouamor, Juan Pino, and Kalika Bali, editors, \emph{Conference on Empirical Methods in Natural Language Processing}, pages 13764--13775, 2023.

\bibitem[Zeng et~al.(2024)Zeng, Lin, Zhang, Yang, Jia, and Shi]{zeng2024johnny}
Yi~Zeng, Hongpeng Lin, Jingwen Zhang, Diyi Yang, Ruoxi Jia, and Weiyan Shi.
\newblock How {J}ohnny can persuade {LLM}s to jailbreak them: Rethinking persuasion to challenge {AI} safety by humanizing {LLM}s.
\newblock In Lun-Wei Ku, Andre Martins, and Vivek Srikumar, editors, \emph{62nd Annual Meeting of the Association for Computational Linguistics}, pages 14322--14350, 2024.

\bibitem[Zhao et~al.(2025)Zhao, Deng, Ng, and Chua]{zhao2025aligning}
Yong Zhao, Yang Deng, See-Kiong Ng, and Tat-Seng Chua.
\newblock Aligning large language models for faithful integrity against opposing argument.
\newblock \emph{AAAI Conference on Artificial Intelligence}, 2025.

\bibitem[Yu et~al.(2024)Yu, Kairouz, Oh, and Xu]{yu2024privacyalign}
Da~Yu, Peter Kairouz, Sewoong Oh, and Zheng Xu.
\newblock Privacy-preserving instructions for aligning large language models.
\newblock In Ruslan Salakhutdinov, Zico Kolter, Katherine Heller, Adrian Weller, Nuria Oliver, Jonathan Scarlett, and Felix Berkenkamp, editors, \emph{41st International Conference on Machine Learning}, volume 235, pages 57480--57506, 2024.

\bibitem[Touvron et~al.(2023)Touvron, Martin, Stone, Albert, Almahairi, Babaei, Bashlykov, Batra, Bhargava, Bhosale, et~al.]{touvron2023llama}
Hugo Touvron, Louis Martin, Kevin Stone, Peter Albert, Amjad Almahairi, Yasmine Babaei, Nikolay Bashlykov, Soumya Batra, Prajjwal Bhargava, Shruti Bhosale, et~al.
\newblock Llama 2: Open foundation and fine-tuned chat models.
\newblock \emph{arXiv:2307.09288}, 2023.

\bibitem[Li et~al.(2023{\natexlab{b}})Li, Bubeck, Eldan, Giorno, Gunasekar, and Lee]{li2023textbooks}
Yuanzhi Li, Sébastien Bubeck, Ronen Eldan, Allie~Del Giorno, Suriya Gunasekar, and Yin~Tat Lee.
\newblock Textbooks are all you need {II}: phi-1.5 technical report.
\newblock \emph{arXiv:2309.05463}, 2023{\natexlab{b}}.

\bibitem[Mirzadeh et~al.(2025)Mirzadeh, Alizadeh, Shahrokhi, Tuzel, Bengio, and Farajtabar]{gsmnoop}
Seyed~Iman Mirzadeh, Keivan Alizadeh, Hooman Shahrokhi, Oncel Tuzel, Samy Bengio, and Mehrdad Farajtabar.
\newblock {GSM}-symbolic: Understanding the limitations of mathematical reasoning in large language models.
\newblock In \emph{13th International Conference on Learning Representations}, 2025.

\bibitem[Wei et~al.(2022{\natexlab{b}})Wei, Wang, Schuurmans, Bosma, ichter, Xia, Chi, Le, and Zhou]{wei2022chain}
Jason Wei, Xuezhi Wang, Dale Schuurmans, Maarten Bosma, brian ichter, Fei Xia, Ed~Chi, Quoc~V Le, and Denny Zhou.
\newblock Chain-of-thought prompting elicits reasoning in large language models.
\newblock In S.~Koyejo, S.~Mohamed, A.~Agarwal, D.~Belgrave, K.~Cho, and A.~Oh, editors, \emph{Advances in Neural Information Processing Systems}, volume~35, pages 24824--24837, 2022{\natexlab{b}}.

\bibitem[Fu et~al.(2022)Fu, Peng, Sabharwal, Clark, and Khot]{fu2022complexity}
Yao Fu, Hao-Chun Peng, Ashish Sabharwal, Peter Clark, and Tushar Khot.
\newblock Complexity-based prompting for multi-step reasoning.
\newblock \emph{International Conference on Learning Representations}, 2022.

\bibitem[SU et~al.(2023)SU, Kasai, Wu, Shi, Wang, Xin, Zhang, Ostendorf, Zettlemoyer, Smith, and Yu]{r3p1}
Hongjin SU, Jungo Kasai, Chen~Henry Wu, Weijia Shi, Tianlu Wang, Jiayi Xin, Rui Zhang, Mari Ostendorf, Luke Zettlemoyer, Noah~A. Smith, and Tao Yu.
\newblock Selective annotation makes language models better few-shot learners.
\newblock In \emph{11th International Conference on Learning Representations}, 2023.

\bibitem[Gupta et~al.(2023)Gupta, Gardner, and Singh]{r1p1}
Shivanshu Gupta, Matt Gardner, and Sameer Singh.
\newblock Coverage-based example selection for in-context learning.
\newblock In Houda Bouamor, Juan Pino, and Kalika Bali, editors, \emph{Findings of the Association for Computational Linguistics: EMNLP 2023}, pages 13924--13950, 2023.

\bibitem[Peng and Risteski(2022)]{c0learn}
Binghui Peng and Andrej Risteski.
\newblock Continual learning: a feature extraction formalization, an efficient algorithm, and fundamental obstructions.
\newblock In S.~Koyejo, S.~Mohamed, A.~Agarwal, D.~Belgrave, K.~Cho, and A.~Oh, editors, \emph{Advances in Neural Information Processing Systems}, volume~35, pages 28414--28427, 2022.

\bibitem[Hu et~al.(2022)Hu, yelong shen, Wallis, Allen-Zhu, Li, Wang, Wang, and Chen]{hu2021lora}
Edward~J Hu, yelong shen, Phillip Wallis, Zeyuan Allen-Zhu, Yuanzhi Li, Shean Wang, Lu~Wang, and Weizhu Chen.
\newblock Lo{RA}: Low-rank adaptation of large language models.
\newblock In \emph{International Conference on Learning Representations}, 2022.

\bibitem[Yen et~al.(2025)Yen, Si, Meng, Yu, Duvvuri, Dhillon, Hsieh, and Kumar]{robustlora}
Jui-Nan Yen, Si~Si, Zhao Meng, Felix Yu, Sai~Surya Duvvuri, Inderjit~S Dhillon, Cho-Jui Hsieh, and Sanjiv Kumar.
\newblock {Lo{RA} Done {RITE}: Robust Invariant Transformation Equilibration for Lo{RA} Optimization}.
\newblock In \emph{13th International Conference on Learning Representations}, 2025.

\bibitem[Ren and Sutherland(2025)]{poorsft}
Yi~Ren and Danica~J. Sutherland.
\newblock Learning dynamics of {LLM} finetuning.
\newblock In \emph{13th International Conference on Learning Representations}, 2025.

\bibitem[Qi et~al.(2025)Qi, Panda, Lyu, Ma, Roy, Beirami, Mittal, and Henderson]{firsttokens}
Xiangyu Qi, Ashwinee Panda, Kaifeng Lyu, Xiao Ma, Subhrajit Roy, Ahmad Beirami, Prateek Mittal, and Peter Henderson.
\newblock Safety alignment should be made more than just a few tokens deep.
\newblock In \emph{Thirteenth International Conference on Learning Representations}, 2025.

\bibitem[Zheng et~al.(2025)Zheng, Decugis, Gehring, Cohen, Negrevergne, and Synnaeve]{codegenreasoning}
Kunhao Zheng, Juliette Decugis, Jonas Gehring, Taco Cohen, Benjamin Negrevergne, and Gabriel Synnaeve.
\newblock What makes large language models reason in (multi-turn) code generation ?
\newblock In \emph{Thirteenth International Conference on Learning Representations}, 2025.

\bibitem[Sun et~al.(2025)Sun, Chen, Zhao, Xu, Zhang, and Yin]{sun2025self1improvement}
Yutao Sun, Mingshuai Chen, Tiancheng Zhao, Ruochen Xu, Zilun Zhang, and Jianwei Yin.
\newblock The self-improvement paradox: Can language models bootstrap reasoning capabilities without external scaffolding?
\newblock \emph{arXiv:2502.13441}, 2025.

\bibitem[Devroye et~al.(1996)Devroye, Györfi, and Lugosi]{dgl}
Luc Devroye, László Györfi, and Gábor Lugosi.
\newblock \emph{A Probabilistic Theory of Pattern Recognition}, volume~31 of \emph{Stochastic Modelling and Applied Probability}.
\newblock Springer, 1996.

\bibitem[Hoeffding(1963)]{hoeffding1963}
Wassily Hoeffding.
\newblock Probability inequalities for sums of bounded random variables.
\newblock \emph{Journal of the American Statistical Association}, 58:\penalty0 13--30, 1963.

\bibitem[Bennett(1962)]{bbbound}
G~Bennett.
\newblock Probability inequalities for the sum of independent random variables.
\newblock \emph{Journal of the American Statistical Association}, 57\penalty0 (297):\penalty0 33--45, 1962.

\bibitem[Bonferroni(1936)]{bonf}
C.E. Bonferroni.
\newblock \emph{Teoria statistica delle classi e calcolo delle probabilit{\`a}}.
\newblock Pubblicazioni del R. Istituto superiore di scienze economiche e commerciali di Firenze. Seeber, 1936.

\bibitem[Dunn(1961)]{dunn}
Olive~Jean Dunn.
\newblock Multiple comparisons among means.
\newblock \emph{Journal of the American Statistical Association}, 56\penalty0 (293):\penalty0 52--64, 1961.

\bibitem[Videau et~al.(2024{\natexlab{b}})Videau, Leite, Schoenauer, and Teytaud]{incantation}
Mathurin Videau, Alessandro Leite, Marc Schoenauer, and Olivier Teytaud.
\newblock Evolutionary pre-prompt optimization for mathematical reasoning.
\newblock \emph{arXiv:2412.04291}, 2024{\natexlab{b}}.

\bibitem[Fournier and Teytaud(2011)]{teytaudfournier}
Herv{\'{e}} Fournier and Olivier Teytaud.
\newblock Lower bounds for comparison based evolution strategies using vc-dimension and sign patterns.
\newblock \emph{Algorithmica}, 59\penalty0 (3):\penalty0 387--408, 2011.

\bibitem[Eiben and Smith(2015)]{eibensmithbook2015}
Agoston~E Eiben and James~E Smith.
\newblock \emph{Introduction to evolutionary computing}.
\newblock Springer, 2015.

\bibitem[Beyer and Schwefel(2002)]{beyer_02_evolutionstrategiescomprehensive}
Hans-Georg Beyer and Hans-Paul Schwefel.
\newblock Evolution strategies – a comprehensive introduction.
\newblock \emph{Natural Computing}, 1\penalty0 (1):\penalty0 3--52, {May} 2002.

\bibitem[Rechenberg(1973)]{rechenberg}
Ingo Rechenberg.
\newblock \emph{Evolutionstrategie: Optimierung Technischer Systeme nach Prinzipien des Biologischen Evolution}.
\newblock Fromman-Holzboog Verlag, 1973.

\bibitem[Schwefel(1977)]{schwefel:74}
Hans-Paul Schwefel.
\newblock \emph{Evolutionsstrategien f{\"u}r die numerische Optimierung}, pages 123--176.
\newblock Birkh{\"a}user Basel, 1977.

\bibitem[Schumer and Steiglitz(1968)]{onefifth}
M.~Schumer and K.~Steiglitz.
\newblock Adaptive step size random search.
\newblock \emph{IEEE Transactions on Automatic Control}, 13\penalty0 (13):\penalty0 270--276, 1968.

\bibitem[Ros and Hansen(2008)]{diagcma}
Raymond Ros and Nikolaus Hansen.
\newblock A simple modification in {CMA}-{ES} achieving linear time and space complexity.
\newblock In \emph{Parallel Problem Solving from Nature}, pages 296--305, 2008.

\bibitem[Doerr et~al.(2019)Doerr, Doerr, and Lengler]{lengler}
Benjamin Doerr, Carola Doerr, and Johannes Lengler.
\newblock Self-adjusting mutation rates with provably optimal success rules.
\newblock In \emph{Genetic and Evolutionary Computation Conference}, pages 1479--1487, 2019.

\bibitem[Einarsson et~al.(2019)Einarsson, Gauy, Lengler, Meier, Mujika, Steger, and Weissenberger]{relengler}
Hafsteinn Einarsson, Marcelo~Matheus Gauy, Johannes Lengler, Florian Meier, Asier Mujika, Angelika Steger, and Felix Weissenberger.
\newblock {The linear hidden subset problem for the (1+1)-EA with scheduled and adaptive mutation rates}.
\newblock \emph{Theoretical Computer Science}, 785:\penalty0 150--170, 2019.

\bibitem[Dang and Lehre(2016)]{danglehre}
Duc{-}Cuong Dang and Per~Kristian Lehre.
\newblock Self-adaptation of mutation rates in non-elitist populations.
\newblock In \emph{14th International Conference on Parallel Problem Solving from Nature}, pages 803--813, 2016.

\bibitem[Doerr et~al.(2017)Doerr, Le, and Makhmara]{fastga}
Benjamin Doerr, Huu~Phuoc Le, and R~Makhmara.
\newblock Fast genetic algorithms.
\newblock In \emph{Proceedings of the Genetic and Evolutionary Computation Conference}, GECCO '17, pages 777--784. ACM, 2017.

\bibitem[Xu et~al.(2008)Xu, Hutter, Hoos, and Leyton-Brown]{satzilla}
Lin Xu, Frank Hutter, Holger~H. Hoos, and Kevin Leyton-Brown.
\newblock {SATzilla}: portfolio-based algorithm selection for sat.
\newblock \emph{Journal of Artificial Intelligence Research}, 32\penalty0 (1):\penalty0 565--606, 2008.

\bibitem[Awad et~al.(2020)Awad, Shala, Deng, Mallik, Feurer, Eggensperger, Biedenkapp, Vermetten, Wang, Doerr, Lindauer, and Hutter]{squirrel}
Noor Awad, Gresa Shala, Difan Deng, Neeratyoy Mallik, Matthias Feurer, Katharina Eggensperger, Andre' Biedenkapp, Diederick Vermetten, Hao Wang, Carola Doerr, Marius Lindauer, and Frank Hutter.
\newblock Squirrel: A switching hyperparameter optimizer, 2020.

\bibitem[Meunier et~al.(2021)Meunier, Rakotoarison, Wong, Roziere, Rapin, Teytaud, Moreau, and Doerr]{ngoptmeuniertevc22}
Laurent Meunier, Herilalaina Rakotoarison, Pak~Kan Wong, Baptiste Roziere, J{\'e}r{\'e}my Rapin, Olivier Teytaud, Antoine Moreau, and Carola Doerr.
\newblock Black-box optimization revisited: Improving algorithm selection wizards through massive benchmarking.
\newblock \emph{IEEE Transactions on Evolutionary Computation}, 26\penalty0 (3):\penalty0 490--500, 2021.

\bibitem[Gissler(2024)]{gisslerPhD2024}
Armand Gissler.
\newblock \emph{{Linear convergence of evolution strategies with covariance matrix adaptation}}.
\newblock Phd. thesis, {Institut Polytechnique de Paris}, {December} 2024.

\bibitem[Auger et~al.(2009)Auger, Hansen, Perez~Zerpa, Ros, and Schoenauer]{CMA-PSO-DE2009}
Anne Auger, Nikolaus Hansen, Jorge~M. Perez~Zerpa, Raymond Ros, and Marc Schoenauer.
\newblock {Experimental Comparisons of Derivative Free Optimization Algorithms}.
\newblock In Jan Vahrenhold, editor, \emph{{LNCS}}. {Springer Verlag}, 2009.

\bibitem[Powell(1994)]{cobyla}
Michael~J.D. Powell.
\newblock \emph{A Direct Search Optimization Method That Models the Objective and Constraint Functions by Linear Interpolation}, pages 51--67.
\newblock Springer Netherlands, 1994.

\bibitem[Powell(1964)]{powell}
Michael~J.D. Powell.
\newblock An efficient method for finding the minimum of a function of several variables without calculating derivatives.
\newblock \emph{The Computer Journal}, 7\penalty0 (2):\penalty0 155--162, 1964.

\bibitem[Williams(1992)]{williams1992simple}
Ronald~J Williams.
\newblock Simple statistical gradient-following algorithms for connectionist reinforcement learning.
\newblock \emph{Machine learning}, 8:\penalty0 229--256, 1992.

\bibitem[McKay et~al.(1979)McKay, Beckman, and Conover]{LHS}
Michael~D. McKay, Richard~J. Beckman, and William~J. Conover.
\newblock {A Comparison of Three Methods for Selecting Values of Input Variables in the Analysis of Output from a Computer Code}.
\newblock \emph{Technometrics}, 21:\penalty0 239--245, 1979.

\bibitem[Niederreiter(1992)]{Niederreiter1992}
Harald Niederreiter.
\newblock \emph{Random Number Generation and quasi-Monte Carlo Methods}.
\newblock Society for Industrial and Applied Mathematics, 1992.

\bibitem[Roberts and Royer(2023)]{DSproba}
Lindon Roberts and Cl\'{e}ment Royer.
\newblock Direct search based on probabilistic descent in reduced spaces.
\newblock \emph{SIAM Journal on Optimization}, 33\penalty0 (4):\penalty0 3057--3082, 2023.

\bibitem[Beyer(2001)]{Beyer:bookES}
Hans-Georg Beyer.
\newblock \emph{The Theory of Evolution Strategies}.
\newblock Natural Computing Series. Springer, Heideberg, 2001.

\bibitem[Gershoff(2025)]{privacydesign}
Matthew Gershoff.
\newblock {K-Anonymous A/B Testing}.
\newblock \emph{arXiv:2501.14329}, 2025.

\bibitem[Biggio et~al.(2012)Biggio, Nelson, and Laskov]{biggio2012poisoning}
Battista Biggio, Blaine Nelson, and Pavel Laskov.
\newblock Poisoning attacks against support vector machines.
\newblock In \emph{Proceedings of the 29th International Coference on International Conference on Machine Learning}, pages 1467--1474, 2012.

\bibitem[Videau et~al.(2024{\natexlab{c}})Videau, Idrissi, Haziza, Wehrstedt, Copet, Teytaud, and Lopez-Paz]{meta_lingua}
Mathurin Videau, Badr~Youbi Idrissi, Daniel Haziza, Luca Wehrstedt, Jade Copet, Olivier Teytaud, and David Lopez-Paz.
\newblock {Meta Lingua}: A minimal {PyTorch LLM} training library, 2024{\natexlab{c}}.
\newblock URL \url{github.com/facebookresearch/lingua}.

\bibitem[Kwon et~al.(2023)Kwon, Li, Zhuang, Sheng, Zheng, Yu, Gonzalez, Zhang, and Stoica]{kwon2023efficient}
Woosuk Kwon, Zhuohan Li, Siyuan Zhuang, Ying Sheng, Lianmin Zheng, Cody~Hao Yu, Joseph~E. Gonzalez, Hao Zhang, and Ion Stoica.
\newblock Efficient memory management for large language model serving with pagedattention.
\newblock In \emph{Proceedings of the ACM SIGOPS 29th Symposium on Operating Systems Principles}, 2023.

\bibitem[Abadi et~al.(2016)Abadi, Chu, Goodfellow, McMahan, Mironov, Talwar, and Zhang]{abadi2016deep}
Martin Abadi, Andy Chu, Ian Goodfellow, H~Brendan McMahan, Ilya Mironov, Kunal Talwar, and Li~Zhang.
\newblock Deep learning with differential privacy.
\newblock In \emph{Proceedings of the 2016 ACM SIGSAC conference on computer and communications security}, pages 308--318, 2016.

\bibitem[Lin et~al.(2025)Lin, Lin, Xie, and Ji]{lin2025cppo}
Zhihang Lin, Mingbao Lin, Yuan Xie, and Rongrong Ji.
\newblock {CPPO}: Accelerating the training of group relative policy optimization-based reasoning models.
\newblock \emph{arXiv:2503.22342}, 2025.

\end{thebibliography}

\appendix
\clearpage

\renewcommand{\theequation}{A.\arabic{equation}}
\setcounter{equation}{0}

\section{Extended Related Work}\label{app:extrelatedwork}
\paragraph{Generalization Bounds for LLMs.}
Early generalization bounds, such as those derived from the VC dimension~\cite{vap95}, covering numbers~\cite{VVW}, and the fat-shattering dimension~\cite{fat}, become ineffective in the overparameterized regime of LLMs~\cite{nagarajan2019uniform}. PAC-Bayes-based methods have been proposed to derive bounds from model compressibility~\cite{zhou2018non,LotfiPACcompression,compressionllm}, offering data-dependent generalization guarantees by leveraging informed priors over parameters to exploit the implicit preference of a network for simpler functions~\cite{LotfiPACcompression}. This approach connects generalization to model compressibility: if a trained Transformer can be heavily compressed while retaining accuracy, it implies an effectively smaller hypothesis complexity. Indeed, some of the tightest known bounds come from compressing large networks. For instance, by quantizing parameters in a low-dimensional subspace. Such results echo Occam's razor, suggesting that big neural networks contain simpler sub-models that drive their generalization.

Another paradigm is algorithmic stability~\cite{as}, that instead analyzes how sensitive the learning algorithm is to perturbations in the training data. This makes it applicable to settings with infinite VC dimension. However, applying stability or norm-based analyses to huge LLMs often requires strong assumptions~(\eg{} tiny learning rates or bounded layer norms) and tends to yield loose bounds in practice. Hence, while these theoretical tools have advanced our understanding, for example, by explaining how a 100-billion-parameter Transformer might effectively behave like such a ``simpler'' model, extending them to practically useful guarantees is non-trivial. In particular, LLMs introduce extra complexities (\eg{} autoregressive dependencies and unbounded loss values) that defy many standard assumptions. \citet{li2023transformeralgo} showed that the self-attention mechanism is uniformly stable under certain Lipschitz constraints, which allows a generalization bound for in-context learning with Transformers. 

Our contribution shares this spirit, exploiting the comparison-based nature of optimization and the limited branching factor of black-box optimizers. By leveraging union-bound variants like the Bonferroni correction, we derive non-vacuous generalization guarantees.

\paragraph{Data Poisoning.}
Data poisoning alters training data to inject specific behaviors, either to degrade performance or to create targeted backdoors~\cite{poisoningforbackdoor}. This includes subtle manipulations—e.g., poisoning just $0.001\%$ of training data in medical applications~\cite{medicalpoisoning}—that remain undetected by conventional metrics. Attacks have proven transferable across quantized and full-precision models~\cite{llmpoisoq}. Recent work shows that inserting only a few percent of malicious preference pairs in RLHF~\cite{christiano2017rlhf} or DPO~\cite{rafailov2023dpo} data can dramatically shift model behavior~\cite{baumgrtner2024bestofvenom,wang2024rankpoison,yan2024backdooring,zhong2023poisoning,zeng2024johnny,zhao2025aligning,poisoningforbackdoor}, such as increasing output verbosity or sentiment polarity. These changes are activated by specific triggers and survive quantization. Theoretical work~\cite{malicious,steinhardt2017certified} confirms that even small adversarial corruptions can be highly effective, especially in overparameterized models. However, certified robustness against poisoning in LLMs remains an open problem. Compared to \eg{} \citep{steinhardt2017certified}, we provide an explicit algorithm with the desired stability assumptions.

\paragraph{Privacy and Robustness.} Other works have focused on privacy-preserving strategies for instruction tuning and alignment of large language models. For instance, \citet{yu2024privacyalign} introduced a two-stage framework that addresses privacy risks not only during training, but also in the instruction collection phase. In this case, they first fine-tune a pretrained LLM using differentially private optimization (DP-Adam) to generate synthetic user instructions, and then apply a DP histogram-based distribution matching technique to resample instructions such that they resemble the distribution of the original private data. This approach ensures that both annotators and the final alignment pipeline never see raw user inputs, thereby preventing both direct exposure and downstream memorization. Applied to LLaMA~\cite{touvron2023llama} and Phi-models~\cite{li2023textbooks}, their approach achieved state-of-the-art performance in both supervised fine-tuning and RLHF settings, with utility comparable to non-private baselines. In particular, they highlight the importance of privacy-preserving data preparation in addition to model training\@.

\paragraph{Post-training for Reasoning.}
Reasoning tasks like GSM8K~\cite{cobbe2021training}, SVAMP~\cite{patel-etal-2021-nlp}, MATH~\cite{hendrycks2021measuring}, and GSM-Symbolic~\cite{gsmnoop} highlight the limitations of standard LLM (pre)training. Post-training approaches for reasoning tasks are now commonplace in LLM training. Enhancements via chain-of-thought prompting~\cite{wei2022chain}, complex reasoning extensions~\cite{fu2022complexity}, and in-context learning~\cite{r3p1,r1p1} are effective but often brittle. Reinforcement learning approaches such as GRPO~\cite{grpo} allow fine-tuning using unsupervised signals, opening the door for real-time user-driven improvement~\cite{c0learn}. LoRA~\cite{hu2021lora} and robust tuning~\cite{robustlora} enhance adaptability but are prone to hallucination~\cite{poorsft} and exploitative behaviors~\cite{firsttokens,codegenreasoning}.  Our approach combines low-rank adaptations with retrofitting techniques and avoids reward-based fine-tuning. It limits overfitting risks and prevents the amplification of memorized or poisoned content by emphasizing rank-based signal aggregation over token-level supervision. A somewhat complementary approach  is self-improving~\cite{sun2025self1improvement}. In such a case, a large language model tries to improve its performance by self-fine-tuning with synthetic data generated by itself. Nevertheless, such approaches have mostly focused on performance, and they are not risk-free from data poisoning or leaking private data.

\FloatBarrier

\section{Examples of Modifiers}\label{concrete}
The \retrofit{} framework is described in \cref{alg:retro}. This algorithm has various hyperparameters, such as a budget, a black-box optimization algorithm, and a $modified(m_0,x)$ function which specifies how a vector $x$ proposed by the black-box optimization method is used to modify an initial model $m_0$. Here we present a few examples of such $modified(\dots)$ functions, involving broadcast, low-rank, or full updates.

\subsection{Examples of Modifiers for One Layer}\label{sec:onelayer}

In the present section $\times$ refers to pointwise multiplication and $@$ refers to matrix multiplication.

\paragraph{Full Layer Update.} A simple example of a multiplicative modifier is the following:

\begin{algorithmic}[1]
\STATE{Define $model=modified(m_0,x)$}
\STATE{$model=copy(m_0)$}
\STATE{$model.output\_matrix= model.output\_matrix \times \exp(0.01 \times x)$}
\STATE{Return $model$}
\end{algorithmic}

In this example, we modify the normalization layer of the model, using a matrix $x$.

Other update formulas  include $+0.01\times x$, $\times (Id+0.01\times x)$, and other constants than $0.01$. However, in most cases, we will use vectors $x$ rather than matrices, as explained in low rank and broadcast.

\paragraph{Low Rank Update.} When updating entire weight matrices, we prefer to work with low-rank ({\bf LoRA}) updates (rank 1 in fact),  defined as follows, when $model.output\_matrix$ has shape $(hidden\_dim,vocab\_size)$ and $x$ has dimension $hidden\_dim+vocab\_size$:

\begin{algorithmic}[1]
\STATE{Define $model=modified(m_0,x)$}
\STATE{$model=copy(m_0)$}
\STATE{$x_1,x_2=split(x)$\COMMENT{$x_1$ has shape $(hidden\_dim,1)$ and $x_2$ has shape $(1,vocab\_size)$}}
\STATE{$model.output\_matrix = model.output\_matrix \times \exp(0.01 \times x_1@x_2)$}
\STATE{Return $model$}
\end{algorithmic}

\paragraph{Broadcast Update.} In some cases, we use {\bf broadcast} rather than low-rank updates.
$modified(m_0,x)$ is then defined as follows, where $x$ as a row vector is repeated $M$ times to fit the shape of the matrix we want to update. In \cref{sec:xp} $x$ has dimension $hidden\_dim$ and $M=hidden\_dim$:

\begin{algorithmic}[1]
\STATE{Define $model=modified(m_0,x)$}
\STATE{$model=copy(m_0)$}
\STATE{$x_1=(1,\dots,1)$} \COMMENT{$x_1$ has shape $(1,hidden\_dim)$}
\STATE{$model.output\_matrix = model.output\_matrix \times \exp(0.01 \times x_1^t@x)$}
\STATE{Return $model$}
\end{algorithmic}

In \cref{sec:addexps}, $x$ has dimension $vocab\_size$ when broadcast updates are applied to the output layer.

\subsection{Modifying Several Layers}\label{sec:multi_layers}
In the examples above, we consider updates of one layer only. However, this can be naturally extended to updates of $l$ layers.
Then, $x$ is split into $x_1,\dots,x_l$, and each $x_i$ is used for updating a specific layer, as in \cref{sec:onelayer}.

In our experiments in \cref{sec:xp} we performed the following updates:
\begin{itemize}
    \item Optimizing the output layer with the rank 1 Low Rank Update. (Experiment 1)
    \item Optimizing all $Q$ matrices of all attention layers with the Broadcast Update. (Experiment 2)
    \item Optimizing all normalization layers with the Full Layer Update. (Experiment 3)
\end{itemize}

For all these experiments the update constant is $0.01$. A few additional experiments in \cref{sec:addexps} use a different constant.

\section{Theoretical Background: A Brief Overview of Statistical Learning Theory}\label{sec:thapp}

\subsection{Non Asymptotic Large Deviation Bounds from the 60s: Hoeffding, \rebuttal{Bennett}, Bernstein}\label{sec:ldb}%
Large deviation bounds refer to bounds on the difference between an expectation and an empirical average. They exist independently of machine learning (contrary to overfitting bounds, below) and are used for polls, Monte Carlo integration, and others.

Indeed, for 
a dataset size $s$, consider $\delta_{1,\epsilon}$ the risk against the randomness $\w_D$ w.r.t. dataset $D$ of an $\epsilon$-divergence between the empirical risk and the risk in generalization for a given \hp{} $x$, i.e.,  
\begin{equation}
    \forall (x,\e,s),P_{\w_D}\left(|\hat L(x) -L(x)| > \epsilon\right) \leq \delta_{1,\epsilon}.\label{eqone}%
\end{equation}
$\delta_{1,\e}$ refers here to the risk of deviation $\e$ for one (hence the subscript 1) model: when $\hat L$ is computed by average on a sample, we can invoke $\delta_{1,\e}$ by Hoeffding, Chernoff or Bernstein bounds~\citep{dgl}.
By Hoeffding's bound ~\citep{hoeffding1963}, if individual losses are in $[0,1]$ we get
\begin{equation}
\delta_{1,\e}=2\exp(-2s\e^2). \label{eq:hoeffding}
\end{equation}
In case we have a prior on a bound $\sigma^2$ on the variance of a randomly drawn individual loss, we can use the bound of \rebuttal{Bennett} and Bernstein~\citep{bbbound}:

\begin{equation*}
\delta_{1,\e}=2\exp\left(-s\epsilon h_1\left(\frac{\epsilon}{\sigma^2}\right)\right).\label{eq:benet}
\end{equation*}

where $h_1(\lambda) = (1 + \frac{1}{\lambda})\ln(1+\lambda)-1$.

\subsection{Probably Approximately Correct (PAC) Learning: A Brief Overview}\label{mathover:reminder}

We propose a point of view on optimization algorithms based on their risk of overfitting in the context of huge VC-dimensions and covering numbers, deriving bounds that are not covered by the studies on the rates of convergence. These bounds were frequently developed in the '90s or '00s:
\begin{itemize}
\item \cref{seq:bonf} presents the Bonferroni correction: whereas ``non-uniform'' large deviation inequalities estimate the difference between an empirical error and a generalization error for a single predefined classifier (for example when testing a single model on a test set), the Bonferroni correction provides bounds applicable uniformly on a finite set of classifiers (e.g., for selecting, with a validation set, a classifier in a list of classifiers obtained on the training set). 
\item \cref{seq:cap} generalizes the Bonferroni correction to infinite families of classifiers: VC-dimension bounds, and others, are the most well-known bounds. They are based on limiting the ``capacity'' of a space of function, even when it is infinite.
\item \cref{seq:as} presents other bounds for infinite families of functions, taking into account the optimization algorithm used for selecting a classifier: the idea is that even if the space of functions is large, we can get bounds if the optimization algorithm satisfies some stability assumptions. This last part of the state-of-the-art is the closest to our work: we get non-trivial bounds on deep learning models, independently of the possibly huge number of parameters.
\end{itemize}

\subsubsection{Bounds on the Generalization Loss based on Bonferroni}\label{seq:bonf}

The Bonferroni bound~\cite{bonf,dunn} is a direct application to statistics of the union bound:
\begin{equation}
P(A_1\mbox{ or }A_2\mbox{ or \dots or}A_k)\leq \sum_{i=1}^k P(A_i)\label{simplebonfbound}
\end{equation}

The analysis below extends \cite{arxivtelo,incantation}, using the branching factor analysis in \cite{teytaudfournier}. 
$L_F(x)$ refers to the error in generalization for a probability distribution $F$ and a model $x$. 
$\hat L_D(x)$ (frequently denoted $\hat L(x)$ in the literature) refers to the random variable corresponding to the empirical error on a dataset $D$, of cardinality $s$, randomly independently drawn from $F$.
Whereas gradient-based optimization does not lead to any bound on the generalization loss, such bounds can be obtained when using comparison-based optimization (i.e., a subset of black-box optimization), as detailed below. These bounds suggest that the risk of overfitting is lower than for gradient-based methods. Our experimental results suggest that this interpretation is correct. 

Let us see the generalization of large deviation inequalities seen in \cref{sec:ldb}.

If we use the empirical error for picking up the best $\widehat x$, in a list of $b$ models $x_1,\dots,x_b$  (e.g., $x_1,\dots,x_b$ are chosen by random search or low-discrepancy search or choice by the user, which leads to $N(\w,a,b)=b$), the risk of deviation $\e$ for that $\widehat x$ is at most \textcolor{black}{$b\cdot \delta_{1,\epsilon}$} (union bound, also known as Bonferroni correction~\cite{bonf}) instead of \textcolor{black}{$\delta_{1,\epsilon}$}: \textcolor{black}{$P_{\w_d}(|\hat L(\widehat x) - L(\widehat x)| > \epsilon) \leq \delta_{\epsilon}= b\delta_{1,\epsilon}.$} 

\cref{deterthm} extends this to more sophisticated cases, with $N(\w,a,b)=\rebuttal{\prod_{i=1}^b} k_i(\w,a,b)>>b$ (in many cases, exponential in $b$ e.g. $N(\w,a,b)\leq 2^b$ and $k_i\leq 2$) instead of $N(\w,a,b)=b$. In spite of being larger, this bound, exponential in $b$, is tighter than bounds based on capacity.

\subsubsection{Bounds on the Generalization Loss based on Capacity: Statistical Learning Theory from the 90s}\label{seq:cap}
\citet{vap95} is a pioneering work on bounds on the generalization loss based on VC-dimension. Many other capacity measures have been defined, such as covering numbers \cite{VVW} and the fat-shattering dimension~\cite{fat}. 
The most classical bound on the generalization loss is probably the VC-bound in the case of classification:
$$ \mathbb{E}[L(\hat x)]\leq \inf_x L(x)+16\sqrt{\frac{VC\log(s)+4}{2s}} $$
(assuming that $\hat x$ minimizes $\widehat L(\widehat x)$).
When using shattering coefficients instead of VC-dimension, one can read:
$$ \mathbb{E}[L(\hat x)]\leq \inf_x L(x)+16\sqrt{\frac{\log(8e\times Shattering(n))}{2s}}.$$

When it is possible to reach zero loss, then
we get
$$ P(L(\hat x)>\epsilon)\leq 2\times Shattering(2n)2^{-s\epsilon/2}.$$
Both the Shattering coefficients and the VC-dimension depend on the class of functions.
Extensions exist for regression and other cases, and other capacity measures exist.
However, capacity bounds are not always relevant for vast models such as deep networks. 

\subsubsection{Algorithmic Stability: Bounds based on Properties of Algorithms in the 2000s}\label{seq:as}

Algorithmic Stability~\citep{as} is a radically different approach, based on some algorithm properties rather than on the capacity of the underlying family of  functions. Because we also use particularities of the learning algorithm to derive generalization bounds, our approach is close to that of Algorithmic Stability. 

Following~\citep{as}, we write $l(x,z)$ the loss of a model parametrized by $x$ on an example $z$, so that $L(x)=\E_z l(x,z)$.
An algorithm which outputs a model parametrized by $\hat x(D)$ on the training data $D$ has uniform stability $\beta$ if $\sup_{D,z} |l(\hat x(D),z)-l_{\neg i}(\hat x(D), z)|\leq \beta$ for all $i$, where $l_{\neg i}$ is the loss of the model learned from a training set obtained from $D$ by removing the $i^{th}$ example.
An algorithm with uniform stability $\beta$ and some other technical assumptions (including $0\leq l\leq 1$) has the following property:

\begin{equation}
\forall \delta>0, P\left(L(\hat x(D))> \hat L(\hat x(D)) + \sqrt{\frac{1+12s\beta}{2s\delta}}\right) < \delta.\nonumber
\end{equation}

\section{Black-box Optimization Algorithms}\label{sec:bbo}
The \retrofit{} algorithmic framework, as illustrated in \cref{fig:fairuse}, embeds an algorithm from the family of black-box optimization algorithms, which we now describe in detail.

\subsection{Evolutionary Algorithms} \label{sec:EAs}
For the reasons given in \cref{sec:retrofitting}, only Evolutionary Algorithms (EAs) have been used in the present work, as they are (at least for those uses in the present document) comparison-based, and hence de facto comply with the finite branching factor condition, which is mandatory for the theoretical results of~\cref{sec:generalization} to hold. \cref{sec:otherBBO} discusses this choice, and its possible relaxation, while the current section introduces EAs more thoroughly, and details all the comparison-based algorithms used in the experiments of \cref{sec:xp}.

\begin{figure}[ht!]
\centering
\includegraphics[width=\textwidth]{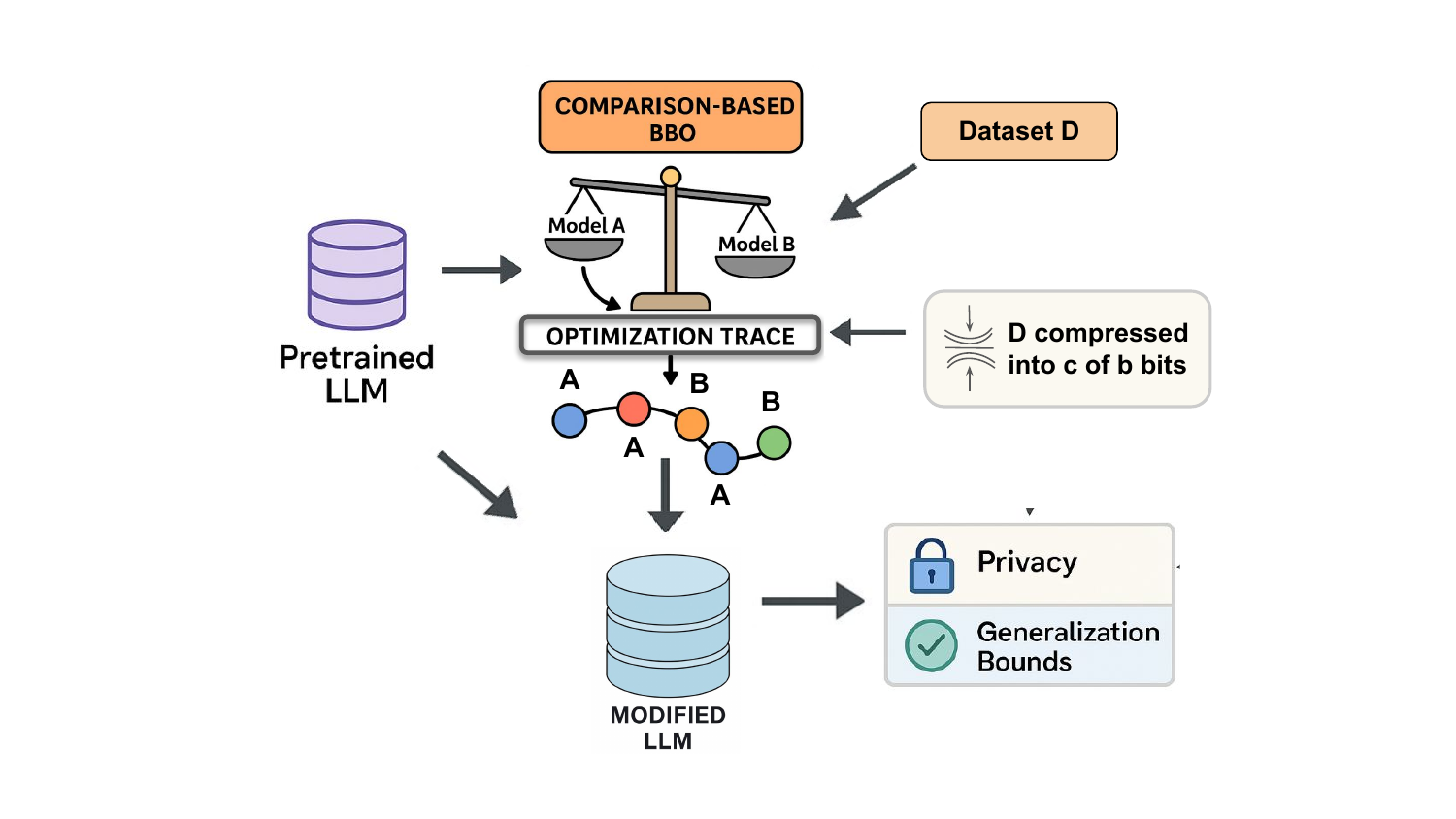}
\caption{\label{fig:fairuse} 
    { {\bf Schematic view of \retrofit{} \cref{alg:retro}.} Our approach compresses the data relative to the current model, leading to a list $c$ of $b$ bits (Eq. \ref{eq:allinone}, $b$ is the total budget). The final model depends only on those $b$ bits, giving perfect control over the number of bits transferred from the dataset to the model. In \retrofit, budgets $b$ are typically between 60 and 1500 bits, \ie{} very small compared to the dataset size. 
    The ``Retrofitting compressor'' refers to the part of~\cref{alg:retro} which outputs the list $c$ of outputs of the $Compare(\dots)$ function: $c$ represents a compression of the dataset $D$ conditional on the initial model. ``Retrofitting modifier'' is the part of~\cref{alg:retro} which outputs the final model, given the list $c$ and the initial model.
    }}
\end{figure}

Evolutionary algorithms~\cite{eibensmithbook2015} are stochastic optimization algorithms that crudely mimic the Darwinian evolution of a ``population'' of $\mu$ ``individual'' (\ie{} points of the search space) based on random variations and ``survival of the fittest'' selection: the $\mu$ ``parents'' generate $\lambda$ ``offspring'' by random operations; the selection process then retains $\mu$ individuals to build the next population; in ($\mu+\lambda$)-EAs, these are the fittest of the $\mu$ parents plus the $\lambda$ offspring, while ($\mu,\lambda$)-EAs select the best of the $\lambda$ offspring only.
EAs can search any search space on which random variation operators can be defined. These EAs are {\bf comparison-based}, \ie{} only comparisons of the values of the objective are used (to deterministically choose the best ones). But stochastic selection processes can also be used in general EAs.

When it comes to continuous optimization (the variables are continuous), among the most popular EAs are Evolution Strategies (ESs)~\cite{beyer_02_evolutionstrategiescomprehensive}, Differential Evolution (DE)~\cite{de}, and Particle Swarm Optimization (PSO)~\cite{pso}.

Evolution Strategies~\cite{rechenberg,schwefel:74,beyer_02_evolutionstrategiescomprehensive} are specific EAs for continuous optimization, using normal ``mutations'', i.e., sampling some Gaussian distribution, generally centered on the parent, to generate the offspring. The issue is then to adjust the variance and the covariance matrix of this normal mutation. 

The first adaptive ES uses the so-called ``one-fifth success rule'' \cite{onefifth,rechenberg}, established by theoretical analyses on simple linear and quadratic objective functions, on which the optimal behavior is reached for a success rate (offspring better than parent) of approximately $0.2$. This is available in Nevergrad as the {\em (1+1)-ES with 1/5th rule}, named in this paper {\bf \OneFifth}, whose pseudo-code is given in \cref{three}-left. 

\textbf{\CMA}~\cite{seminalCMAES,CMA,tutorialCMAES} iteratively updates a multivariate Gaussian random variable, initialized at the identity and parametrized by a covariance matrix and its variance (the so-called {\em step-size}). Each update is based on the correlations observed in successful mutations. This allows \CMA\ to adapt its step-size and search directions (the covariance matrix) to the problem landscape. Despite its $o(d^2)$ computational cost and memory usage, due \rebuttal{to} the handling of the full covariance matrix, it is widely regarded as a powerful and reliable optimizer. \\
\textbf{\DCMA}~\cite{diagcma} is the most popular of the several ``tricks'' that have been proposed to cope with this complexity. \DCMA\ uses a diagonal covariance matrix rather than a full positive definite matrix. Its efficiency, hence, depends on the degree of correlation between the variables, aka the {\em separability} of the objective function. \\
When handling a full population ($\mu>1$), both \CMA\ and \DCMA\ apply the adaptive Gaussian mutation to a linear recombination of the $\mu$ individuals. The historical \CMA\ used the mean (all weights of the linear recombination are 
$1/\mu$), whereas the state-of-the-art \CMA\ uses weights that decrease with the rank of each individual. In the latter case, more information is used, and the branching factor increases (see \cref{sec:pop}).

The simplest ESs (or EAs) are the (1+1)-EAs: the single current individual is modified, and the best of both becomes the new current one. While the already-mentioned {\bf \OneFifth} was the first historical successful adaptive EA (see again \cref{three}-Left),  another family of (1+1)-EAs is directly derived from the historical GAs (Genetic Algorithms), that optimize bitstrings: the mutation typically flips each bit with a given probability. In the continuous case, the bitflip is turned into a ``Gaussian flip'': each value is replaced by a value sampled from a standard normal mutation (${\cal N}(0,1)$ in \cref{three}). Several variants of such ``discrete-inspired'' algorithms exist in Nevergrad, depending on the probability of applying this Gaussian flip to each variable. 
The {\bf \Discrete} algorithm exactly mimics the standard setting in GAs by adopting a fixed $1/d$ probability; the {\bf \Lengler}~\cite{lengler,relengler} algorithm uses a gradually decreasing schedule. The {\bf \Portfolio}~\cite{danglehre} algorithm uses a probability that is uniformly sampled in $]0,1]$, and the {\bf \FastGA}~\cite{fastga} uses an $\alpha / d$ probability where $\alpha$ is sampled from a power-law distribution. \\
Additionally, {\bf \COLengler} is a variant of \Lengler\ that also performs, after mutation, a coordinate-wise crossover {\em à la} differential evolution.%

\textbf{DE} (Differential Evolution)~\cite{de} evolves a population of individuals: each individual is mutated using (i) differences between good points (for taking into account the key directions) and, in its ``{\bf Curr-to-best}'' variants (ii) typical differences between the population and the best point so far, before a coordinate-wise crossover is applied. In spite of its very low computational cost, DE is known as a very efficient algorithm.

\textbf{PSO} (Particle Swarm Optimization)~\cite{pso} uses a population of moving individuals ("particles"), each with an associated velocity. At each iteration, the velocity of each particle is biased towards the best value observed so far by this particle, and towards the best value over the complete population of particles, and the particle moves according to its new velocity.

\begin{wraptable}{l}{0.49\textwidth}

\centering
 \begin{scriptsize}
\caption{\small{Bound on the average branching factor for various algorithms (see \cref{sec:bar}, \cref{sec:pop}, \cref{sec:depso}). }} \label{bran}
\begin{tabular}{|l|c|}
\hline
\rowcolor{gray} \multicolumn{1}{c}{Algorithm} & $\left(\prod_{i=1}^b k_i\right)^{1/b} \leq$ \\
\hline
\rowcolor{lightgray} \multicolumn{2}{|c|}{$(1+1)$ algorithms}\\
\hline
\OneFifth & $2$ \\
\Lengler & $2$ \\
FastGA & $2$ \\
Portfolio & $2$ \\
\hline
\rowcolor{lightgray} \multicolumn{2}{|c|}{Bet-And-Run algorithms}\\
\hline
\MultiDisc\ & $\left(3\cdot 2^{b/3}\right)^{1/b}<2$ \\
\Triple\ & $\left(3\cdot 2^{b/3}\right)^{1/b}<2$ \\
\hline
\rowcolor{lightgray} \multicolumn{2}{|c|}{Population-based algorithms}\\
\hline
DE, PSO & $2$ \\
DE, PSO with curr-to-best & $3$ \\
\CMA/\DCMA & $2$ \\
\CMA/\DCMA & \\
\hspace{.2cm} with distinct weights & $2\times \sqrt{\mu}$ \\
\hline
\end{tabular}
\end{scriptsize}
\end{wraptable}

\begin{algorithm}[tb!]
\centering
\begin{scriptsize}
\begin{tabular}{ccc}
\begin{minipage}{.5\textwidth}
\centerline{\bf{Adaptive One-fifth $(1+1)$-ES}}
\smallskip
\begin{algorithmic}[1]
\REQUIRE{initial point $x_0\in \R^d$; budget $b\in \N$}
\REQUIRE{step-size $\sigma>0$}
\STATE{$x=x_0$}
\FOR{$i\in\{1,\dots b\}$}
    \STATE{$x'=x+\sigma {\cal N}(0,1)$}
    \IF{$x'$ better than $x$}
    \STATE{$x=x'$}
    \STATE{$\sigma=2\sigma$}
    \ELSE
    \STATE{$\sigma=2^{-\frac14}\sigma$}
    \ENDIF
\ENDFOR
\end{algorithmic}
~\\
\end{minipage} &

\begin{minipage}{.5\textwidth}
\centerline{{\bf Generic Discrete-inspired {$(1+1)$-EA }}}
\smallskip
\begin{algorithmic}[1]
\REQUIRE{initial point $x_0\in \R^d$; budget $b\in \N$}
\STATE{$x=x_0$}
\FOR{$i\in\{1,\dots b\}$}
    \STATE{$x'=x$}
    \WHILE{$x'=x$}
    \STATE{$x'=x$}    
    \FOR{$j\in\{1,\dots,d\}$}
        \STATE{With probability $p(i,b,d)$ \; $x'_j= {\cal N}(0,1)$}
    \ENDFOR
    \ENDWHILE
    \IF{$x'$ better than $x$}
         \STATE{$x=x'$}
    \ENDIF
\ENDFOR
\end{algorithmic}
\end{minipage}\\
\end{tabular}
\end{scriptsize}
\caption{\label{three} 
The two typical (1+1)-EAs. \\
{\bf Left}: all variables are added some Gaussian noise in the $(1+1)$-ES, but the step-size of the multivariate Gaussian mutation is modified according to the success rate. \\
{\bf Right}: each variable is replaced with a given probability by a value sampled from a fixed standard normal Gaussian in all discrete-inspired (1+1)-EA.}
\end{algorithm}

\FloatBarrier

Furthermore, Nevergrad also includes some {\em ensemble approaches}. \\
{\bf \NgIohTuned} is a {\em wizard}. It chooses one algorithm from its portfolio of BBO algorithms, based on some characteristics of the problem, and following a hand-made rule. Portfolios and wizards have been first proposed for SAT solving~\citep{satzilla}, and later ported to black-box optimization~\citep{squirrel,ngoptmeuniertevc22}.\\
{\bf Bet-And-Run} algorithms, for budget $b$, consists in running $k$ algorithms with budget $\alpha b/k$, picking up the one reaching the best objective value and running it for the remaining $(1-\alpha)b$ budget; {\bf \Triple} is an instance of Bet-And-Run with $\alpha=0.5$ and three \OneFifth\ in parallel; with $\alpha=1$, {\bf \MultiDisc} runs three \Discrete\ in parallel and {\bf \BARD} runs DE and \DCMA{}. 

These algorithms were chosen because of their finite branching factor: \cref{bran} lists the branching factors of all cited algorithms. The branching factor is 2 for all (1+1) algorithms. The derivation of the other values is detailed in \cref{genbou}.

\noindent
{\bf Discussion.} The choice of a BBO algorithm for a given problem is a long-lasting problem, and can of course be left to the Nevergrad wizard \NgIohTuned. However, some characteristics of the algorithm introduced above are worth highlighting here, and might in turn give some information about the problems addressed with \retrofit\ based on its experimental results.\\
The adaptive approaches are now well mastered: 
\CMA\ offers the most sophisticated update mechanism. It is competitive in terms of performance with the best algorithms of Mathematical Programming, and has recently been proven to converge linearly \cite{gisslerPhD2024} 
However, it has a high computational complexity with respect to dimension $d$ and hence poorly scales up. Among several variants, \DCMA\ offers a very good compromise. But the simplest adaptive mechanism is that of \OneFifth, which is very robust and often very efficient. Note, however, that its adaptation mechanism requires a number of iterations to take place, and was, for example, discarded in the OpenAI ES, one of the rare examples of large-scale BBO in complete training of Deep Neural Networks \cite{openAI-ES4DeepRL2017}.\\
Because \DCMA{} only adapts a diagonal covariance matrix, one cannot expect that it effectively handles non-separable objectives. Less intuitively, this is also true for PSO, and, to a lesser extent, for DE unless its crossover rate is small \cite{CMA-PSO-DE2009}.  On the other hand, their adaptive mechanism (implicit for DE) makes \CMA, \DCMA\ and (some variants of) DE very efficient in dealing with randomly rotated ill-conditioned problems, which is not the case for PSO \cite{CMA-PSO-DE2009}.  \\
Regarding the discrete-inspired (1+1)-EAs, \Discrete, \Portfolio, \Lengler, and \FastGA, they are based on proved convergence proofs for known benchmark functions in the bitstring case, but their efficiency in the continuous setting is more surprising, and would deserve deep mathematical analyses. \\
Last but not least, Bet-And-Run algorithms, including \Triple, \MultiDisc, and \BARD, are robust, as proved in \cref{sec:bar} and demonstrated in the experimental results (\cref{sec:xp} and \cref{sec:addexps}). However, they still require the choice of the embedded algorithms.

\subsection{Possible Extensions} 
\label{sec:otherBBO}

\textbf{Code and Package.}
As said, all experiments in this work have been run on the Nevergrad open source platform~\cite{nevergrad}, which encompasses many BBO algorithms, along with specialized wizards that select among them based on problem-specific properties.
Nevergrad also includes algorithms that are black-box optimizers without being comparison-based, such as~\citep{cobyla,powell}. Since these operate with a finite number of bits per loss function, they might also fall under the theoretical framework of \cref{sec:generalization}, though this is left for future work. 

Extension to multi-objective algorithms is another promising direction, for instance by optimizing different criteria (e.g., pass@k and accuracy) or by considering performance across several distinct benchmarks as the training objective, rather than averaging them.

\paragraph{Beyond Comparison-based Algorithms}
An important assumption (\cref{alg:bound}) is that the only thing that depends on the dataset $D$ is $choice_1,\dots,choice_b$, and each $choice_i$ has at most $k_i$ possible values ($k_i:=alg.numCases()$ is called the branching factor~\cite{teytaudfournier}).
This is unusual in deep learning (DL), as in DL $compare$ is replaced by some function involving gradients (even Reinforce~\cite{williams1992simple} or GRPO~\cite{grpo} return a gradient update), and $choice_i$ is replaced by a gradient, or, for reinforcement learning such as GRPO~\cite{grpo}, the outputs and their rewards. 

Algorithms that do not return any gradient are termed black-box optimization algorithms. Among black-box optimization algorithms, some methods are comparison-based: they have finite $k_i$.
These comparison-based algorithms are typically classified in three families of algorithms: \begin{inparaenum}[(i)] \item random search and derandomized variants~\cite{LHS,Niederreiter1992}, \item direct search methods~\cite{DSproba}, and \item evolutionary computation~\cite{Beyer:bookES}\end{inparaenum}. Note that GRPO~\cite{grpo} is sometimes considered as gradient-free in the sense that it does not require the gradient of the loss function; however, it does require other detailed information about the outputs of models, so that it is not black-box in the same sense (and a fortiori not comparison-based), and the present results do not apply to GRPO or REINFORCE~\cite{williams1992simple} if we assume that they work on infinite-precision floats.

\paragraph{Taking into account the Finite Float Precision.}
Because our results are mainly useful when $k_i$ is small, this work focused on comparison-based algorithms. However, all results of \cref{sec:generalization} are true for large values of $k_i$. Let us check $k_i$ in the case of some classical methods for post-training. 
In \retrofit, the information flow from the data to the model is limited to roughly one bit per iteration, leading to at most $2^b$ possible outputs for a budget of $b$ steps. This limited capacity directly reflects the compression properties showcased in~\cref{eq:allinone} and \cref{fig:fairuse}.
In contrast, supervised fine-tuning updates the model using one gradient per mini-batch. For an 8B model with 16-bit precision, each mini-batch carries about 256B bits: a few mini-batches are sufficient to encode the entire content of Wikipedia.
Reinforcement learning methods, such as GRPO, typically rely on model outputs and reward signals that would allow them to encode the full dataset within a single training epoch.

\section{Overfitting Bounds and their Generalizations}\label{genbou}%
In this Section we provide the theoretical generalization bounds for the deterministic (\cref{app:deterministic}) and randomized cases (\cref{app:stochastic}), followed by (\cref{es}) strengthened bounds for the special cases of $(1,\lambda )$ and $(1+\lambda)$ BBO algorithms as described in \cref{sec:bbo}.
We proceed by making our bounds explicit for  various families of algorithms, based on \citet{teytaudfournier}; in \cref{sec:bar} we obtain strengthened bounds for Bet-and-Run combinations like \Triple, MultiDisc, or \BARD\  (deployed in some of our experiments in \cref{sec:xp}), \cref{sec:pop} lays out bounds for population-based evolution strategies (i.e., $(\mu+\lambda)$ and $(\mu,\lambda)$ with $\mu>1$) like CMA-ES and D-CMA, and \cref{sec:depso} deals with differential evolution and particle swarm optimization.
The key to our results is the branching factor $k_i$, i.e., the number of distinct behaviors of the algorithm that can arise in an iteration. Finally, \cref{ssize} derives explicit bounds as a function of the dataset size.

\subsection{Proof of Theorem 1}\label{app:deterministic}

\begin{proof}

For~\cref{g}: apply the Bonferroni bound (\cref{simplebonfbound}) to the set of events 
$$\{ |\hat L(FinalModel_{\w,D,a,b})-L(FinalModel_{\w,D,a,b})|>\epsilon ; D \in \D\}.$$ This set has cardinality at most $N(\w,a,b)$.

For \cref{e}, apply the Bonferroni bound to the set of events  $$\{ |\hat L(model_{\w,D,a,b,i})-L(model_{\w,D,a,b,i})|>\epsilon ; i\in \{1,\dots,b\},D \in \D\},$$
where $model_{\w,D,a,b,i}$ is $m_i$ when \retrofit is applied with seed $\w$, algorithm $a$, budget $b$, dataset $D$.
This set has cardinality $b \cdot N(\w,a,b)$.
\end{proof}

\subsection{Generalization Bounds for Randomized \retrofit}\label{app:stochastic}

\cref{deterthm} applies to deterministic algorithms, where the only source of randomness is the sampling of the dataset
\(D\), determined by the seed \(\w_D\). In this case, the deviation risk \(\delta_{1,\epsilon}\) is computed over
that randomness alone. When the optimization algorithm itself is randomized (i.e., \(\w\) is a random seed), the
deviation risk \(\delta_{\epsilon, \w}\) also depends on \(\w\). Averaging over \(\w\) yields the total
deviation risk: 
\[
\delta_\epsilon = \mathbb{E}_\w \delta_{\epsilon, \w}.
\]
Since the bound of \cref{deterthm} holds uniformly for any fixed \(\w\), it extends naturally to the randomized setting, in which $\widehat x$ depends on both $\w$ and $\w_D$:

\setcounter{theoremUnboxed}{2}
\begin{theoremUnboxed}[Stochastic case]\label{stothm}
    Consider a randomized seed $\w$ (i.e. $\w$ is a random variable), an algorithm $a$  and a budget $b\in \N$. Then
    \begin{equation*} 
        P_{\w,\w_D}(|\hat L(\widehat x) - L(\widehat x)| \geq \epsilon))\leq  \left(\sup_{\w\mbox{ constant}} N(\w,a,b)\right)\cdot \delta_{1,\e}.\nonumber
    \end{equation*}
\end{theoremUnboxed}

Here, the probability is taken over both \(\w\) (which determines the optimization trajectory and hence the final model \(\widehat{x}\)) and \(\w_D\) (which governs dataset sampling). The supremum is over a constant $\w$, and $\widehat x$ can be built with a random variable $\w$.

\paragraph{Application to Random Search.}
In this case, the \(b\) candidate models are sampled independently of the data. Since each choice is fixed a priori, the number of internal states satisfies \(N(\w, \texttt{random-search}, b) \leq b\). $N(\w,a,b)$ is derived from $k_b=b$ (we choose between the $b$ models $x_1,\dots,x_n$ at the very end) and $k_i=1$ for other values of $i$. Apply the Bonferroni correction to the $N(\w,a,b)$  possible values of $\widehat x$ leads to
\begin{equation}\label{c1}
    P_{\w_D}(|\widehat{L}(\widehat{x}) - L(\widehat{x})| \geq \epsilon) \leq b \cdot \delta_{1, \epsilon}\ .
\end{equation}

This also applies to quasi-random strategies and design of experiments, where the sampling of hyperparameters does not depend on their empirical performance. In these cases, we observe that $b$ can be exponential in $s$ without overfitting.%

\subsection{Comparison-based Strategies: the \texorpdfstring{\((1,\lambda)\)}{(1, lambda)} and \texorpdfstring{\((1+\lambda)\)}{(1 + lambda)} Cases.}\label{es}
In contrast to random search, where all candidate parametrizations are chosen independently of their evaluation, comparison-based algorithms build candidates iteratively based on performance. If \(\lambda\) candidates are evaluated at each step and the best is retained, then the number of internal states after one iteration is \(\lambda\); after two iterations, \(\lambda^2\); and after \(n = b/\lambda\) iterations, at most \(\lambda^n\). The total risk is, therefore, bounded by \(\lambda^n \cdot \delta_{1,\epsilon}\), where \(n\) is the number of iterations with \(k_i > 1\) (with \(k_i = \lambda\) in such iterations); see~\cite{teytaudfournier}.

In the elitist variant \((1+\lambda)\), the best model is selected among the current best and the \(\lambda\) new candidates, so \(k_i \in \{1, \lambda + 1\}\). This leads to a total number of internal states bounded by \((\lambda+1)^{(b-1)/\lambda}\). Simply apply  \cref{stothm} in the case 
\begin{itemize}
    \item $k_i=\lambda$ if $i+1 \mbox{ mod }\lambda=0$ (every $\lambda^{th}$ iterations we pick up the best of the previous $\lambda$ models).
    \item $k_i=1$ otherwise.
\end{itemize}
and the corresponding generalization bound becomes:
\begin{equation}\label{theco}
    P(|\widehat L(\widehat x) - L(\widehat x)| \geq \epsilon) \leq (\lambda + 1)^{(b-1)/\lambda} \cdot \delta_{1,\epsilon}\ .
\end{equation}

For example, this applies to the \((1+1)\) evolution strategy with the one-fifth success rule or the Lengler variant~\cite{relengler}, both corresponding to \(k_i = 2\).

\subsection{Generalization Bounds for Bet-And-Run Combinations 
}\label{sec:bar}
Consider the example in  \cref{theco}.
Additionally, \textcolor{black}{if we perform $k>1$ independent optimization runs, and if the budget $b$ is divided by $k$}, then the Bonferroni correction leads to a risk 
\begin{equation}
\underbrace{k}_{\text{Bonferroni correction}}\cdot \underbrace{\lambda^{(b/(\lambda\cdot k))}}_{\text{Branching factor per run}} \cdot \underbrace{\delta_{1,\e}}_{\text{Risk for a single model}}\label{eqov}
\end{equation} 
for the resulting parametrization. According to these bounds, increasing the number $k$ of independent runs for a given total budget decreases the risk of overfitting.
More generally, consider $bet-and-run(a_1,a_2,\dots,a_k)$, the bet-and-run of $k$ algorithms $a_1$,\dots,$a_k$. This algorithm, for a budget $b$, runs each $a_i$ with budget $b/k$: 
\begin{itemize}
\item These $k$ runs are completely independent.
\item Its random seed is split into $k$ random seeds $\w^{1},\dots,\w^{k}$.
\item At the end, we choose the model with best empirical performance, i.e. $\widehat x$ is $\widehat x_i=FinalModel(\w^{i},a_i,b/k)$ minimizing $\hat L(\widehat x_i)$.
\end{itemize}

\begin{theoremUnboxed}[Bet and Run has low overfitting risk]\label{th:bar}
Consider $bet-and-run(a_1,a_2,\dots,a_k)$.%
Then 
\begin{equation}
N(\w,bet-and-run(a_1,a_2,\dots,a_k),b)\leq \sum_{i=1}^k N(\w,a_i,b/k).\label{eq:bar}
\end{equation}
\end{theoremUnboxed}
\begin{proof}
 The final outcome $\widehat x$ is one of the $\widehat x_i$, so the number of possible outcomes is the sum of the numbers of possible outcomes of each algorithm individually.
\end{proof}

Let us consider that the $a_i$ are all the same algorithm, and let us check the impact of bet-and-run on overfitting.
\cref{eq:bar} implies a lower overfitting than with a single algorithm (corresponding to $N(\w,a_i,b)$) if $N(\w,a_i,b/k)$ increases more than linearly in $b/k$, which is the case for most optimization algorithms. For example, with $(1+1)$ evolution strategies, \cref{es} shows that $N(\w,evolution\_strategy,b)$ is exponential in $b$.

Let us illustrate this in an example. Consider \cref{eq:co1} in the case $k_i\leq 2$, namely:
\begin{equation}
P_{\w_D}\left(|\widehat L(\widehat x) - L(\widehat x)| \geq \epsilon\right)  \leq 2\cdot 
 \exp(\ln2\cdot b -2s\e^2).\nonumber
\end{equation}
With a bet-and-run of $3$ algorithms, it becomes:
\begin{equation}
 P_{\w_D}\left(|\widehat L(\widehat x) - L(\widehat x)| \geq \epsilon\right)  \leq 2\cdot 
 \exp(\ln2 \cdot (\log_2(3) +  b/3) -2s\e^2)\label{eq:barexplained}.
\end{equation}
which is better (lower) as $b\to \infty$.

\subsection{Population-Based Evolution Strategies (e.g. \CMA\ and \D-CMA)}\label{sec:pop}

For a $(\mu+1)$ evolution strategy which stores a full ranking of the population of size $\mu$, each new point has $(1+\mu)$ possible ranks: from best to worst in a list of cardinal $1+\mu$, so the bound becomes $(1+\mu)^{N-\mu}$. Extensions are possible for $(\mu,\lambda)$ or $(\mu+\lambda)$ evolution strategies:
\begin{eqnarray*}
 N(\omega,(\mu,\lambda),a=(\mu,\lambda)\mbox{-ES},b)  & \leq & {\binom{\lambda}{\mu}}^{\frac{b-\mu}{\lambda}} \nonumber \\
 N(\omega,(\mu+\lambda),a=(\mu+\lambda)\mbox{-ES},b)  & \leq & {\binom{\lambda+\mu}{\mu}}^{\frac{b-\mu}{\lambda}} \nonumber %
\end{eqnarray*}
In the case of CMA or \D-CMA, the $(\mu,\lambda)$ strategy is more usual and used in our experiments; and typically, the population-size $\lambda$ is set as $\lambda=4+\lfloor 3 \log(d)\rfloor$ in dimension $d$, and $\mu$ frequently scales as $\lambda/4$ or $\lambda/2$. 
Sometimes there is a complete ranking of the $\mu$ selected points so that different weights are used for the $\mu$ selected points depending on their rank: this adds a factor $\mu!$ to the number of possible branches: otherwise, we get bounds as in \cref{bran}. 

\subsection{Differential Evolution (DE) and Particle Swarm Optimization (PSO)}\label{sec:depso}
For differential evolution without storing the best so far, each point is just compared to its parent: we get $2^b\delta_{1,\e}$: this suggests that recombination is not a problem as it does not increase the number of bits of information. When the best so far is stored, we have 3 possibilities: worse than parent, better than parent but not better than best so far, and new best so far: so the bound is $3^b\delta_{1,\e}$. Similar bounds are straightforward for PSO. 

\subsection{Dependence on Dataset Size:
Scaling and Numerical Application}\label{ssize}

\cref{mathover:reminder}
 presents classical large deviation inequalities, and the present paper shows how, in the case of comparison-based algorithms, we can derive non-vacuous bounds. Let us now see numerically how our overfitting bound behaves in practical cases.
 
We have focused on bounds (on the generalization loss of the model obtained by optimization) derived from a given $\delta_{1,\e}$ (obtained for a single model): 

We have seen in \cref{eq:hoeffding} that, using Hoeffding's bound, $\delta_{1,\e}=2\exp(-2s\e^2)$ with $s=|D|$ the sample size, and 
$\delta \leq 2\times (2^b) \exp(-s\e^2)$,
leading to $\delta$ constant when $b$ scales linearly with $s$: the acceptable number of iterations (i.e., a fixed bound
on $\delta$) is linear in the sample size $D$. Applying Hoeffding, and aiming at a risk $\delta$, we get that overfitting at level $\e$ is impossible if
\begin{equation*}
b \leq \frac{\ln(\delta)+2s\e^2}{\ln(2)}-1.\nonumber
\end{equation*}
If we prefer the \rebuttal{Bennett}\&Bernstein inequality (based on an assumption on the standard deviation, see \cref{eq:benet}~\citep[Section 8.2]{dgl}), we get

\begin{equation*}
b \leq {\frac{s\epsilon h_1(\epsilon/\sigma^2)}{\ln(2)}}+\frac{\ln(\delta)}{\ln(2)}-1.\nonumber
\end{equation*}
This equation, for $s=8000$ examples (close to GSM8K), $\delta=\frac12$ (we are considering the median case), $\epsilon=0.01$ (we want that precision) and $\sigma=0.06$ (the standard variation of the loss is 0.06, we assume that the initial model is already good), leads to the claim that a budget $b\leq 91$ cannot lead to overfitting.
If we assume $\epsilon=0.04$ and $\sigma=0.3$, we get with Hoeffding a bound $34$ and with \rebuttal{Bennett}\&Bernstein a bound $88$. With $\epsilon=0.06$, $\sigma=0.3$, \rebuttal{Bennett}\&Bernstein leads to a budget $189$. The budget becomes $482$ if we use $\epsilon=0.1$, $\sigma=0.3$.
Also, these bounds scale linearly with $s$  (the number of user preferences in A/B testings or dataset size).
These numbers are low, but we see that our bounds are not vacuous even in realistic scenarios.

\section{Consequences for Privacy and Poisoning}
\label{app:privacy}
\subsection{Privacy by Design}\label{sec:privacy}
In the context of large language models, privacy typically refers to users' desire to prevent their inputs, such as queries or conversations, from being inferred or reconstructed. Several studies have shown that gradient-based training methods can leak such information~\cite{privacyattack1,privacyattack2,privacyattack3,privacyattack4}.

In contrast, the \retrofit\ approach is an instance of {\bf privacy by design}~\cite{privacydesign}: It does not rely on gradients, nor does it require access to the raw content of user sessions. Instead, it only leverages a binary satisfaction signal\textendash{}whether a user prefers one version of the model over another. This can be collected via simple A/B testing, or from the individual samples of the dataset by comparing the models over the singleton \{sample\}. Then the algorithm proceeds using aggregated comparison outcomes. In particular, for A/B testing, our method operates without knowledge of the specific queries or model responses, that do not need to leave the user's device, and in the general case, we only feed the algorithm the aggregate of model comparison without direct access to the underlying~data.

We now examine how this applies in two distinct cases: the privacy of user queries and responses and the general case of privacy.

\paragraph{Privacy of Queries and Responses.} (like in A/B testing setting)
Our algorithm never observes the actual content of user interactions. If \(D_1\) and \(D_2\) differ by one prompt or
answer or any data unrelated to preference signals, then the outcome of~\cref{alg:retro} remains unchanged.
Since no gradients are
shared or aggregated, the system is inherently robust to inference attacks such as those
in~\cite{privacyattack1,privacyattack2,privacyattack4}.

\paragraph{\rebuttal{Algorithm stability.}} (general case of applying BBoxER on a dataset $D$)
We may also be interested in \rebuttal{the influence} of a single satisfaction response \(r_i\) associated with a specific user, or in the case were we compute this preference signal from the data. In this setting, we consider the case where \(D_1\) and \(D_2\) differ by exactly one such response, altered during a run.
Applying \cref{probadifferenceb} with $m=1$, we get that $A(D_1)$ and $A(D_2)$ are the same with probability at least $1-3b/\sqrt{2s\pi}$ when we have $s$ data points and $b$ iterations.  \rebuttal{Which represents $\varepsilon=0$ and $\delta=3b/\sqrt{2s\pi}$, but is insufficient from a DP standpoint given that $\delta$ needs to be polynomial in $\frac{1}{s}$. 
}

\paragraph{Comparison with Other Approaches.}
In supervised fine-tuning, both inputs and outputs are required for training. Reinforcement learning additionally requires access to the model’s responses and reward signals. While federated learning aims to limit privacy leakage by aggregating gradients locally, it still exposes sensitive intermediate data—such as gradients—that have been shown to be vulnerable to attacks~\cite{privacyattack2}. In contrast, our approach bypasses these vulnerabilities entirely by never interacting with such information in the first place, in the case of A/B testing. Or, by limiting the impact of individual samples through an aggregate compressed signal from the data.

\subsection{Poisoning}\label{sec:poisoning}
Poisoning attacks seek to influence the outcome of a learning algorithm by injecting malicious data. This threat has been widely studied in machine learning~\cite{biggio2012poisoning,poisoningthreat,wan2023poisoning}, and it is particularly relevant in interactive settings such as retrofitting. In our context, a poisoning attempt manipulates up to \(m\) user preferences per iteration to bias the optimization process.

Our robustness \cref{tpp} quantifies the impact of such attacks: the probability that the final model is
altered remains bounded by \(\frac{(2m + 1)b}{\sqrt{2s\pi}}\). This implies that an adversary must control \(m \sim
\sqrt{s}/b\) user preferences per round to affect the result significantly. Such an attack becomes increasingly difficult as
the number of users (or dataset size) \(s\) grows.

Importantly, this guarantee holds even under a strong adversarial model~\cite{malicious}, where the attacker is computationally unbounded and has full knowledge of the system. In practice, this highlights the resilience of our retrofitting method to small-scale poisoning.

\section{Additional Theoretical Results and Proofs}\label{sec:proofs}
\subsection{Proof of \texorpdfstring{\Cref{tpp}}{Theorem}}

We first prove the following Lemma.
\begin{lemma}[Privacy and robustness to poisoning -- single comparison]\label{ltpp}
    Consider a frequency $f=\frac1s \sum_{i=1}^s r_i$ computed by considering $s$ samples $r_1,\dots,r_s$ independent and
    identically distributed with $P(r_i=1)=f_0$ and $P(r_i=0)=1-f_0$ for some unknown $f_0$. Then
    \begin{equation}
    P_{r_1,\dots,r_s}\left( \exists \ f'\mbox{ such that $|f'-f|\leq m/s$ \text{ and } }(f-\frac12)\cdot(f'-\frac12)\leq 0   \right) \leq \frac{2m+ 1}{\sqrt{2s\pi}}.\label{probadifference1}
    \end{equation}
\end{lemma}
\begin{proof}
First, using \cite[Lemma A.3]{dgl}, we observe
that
\begin{equation}   
\forall x\in \{0,\ldots,s\}, P(B(s,\frac12)=x)\leq \frac1{\sqrt{2s\pi}}.\label{dgllemma}
\end{equation}

Consider a probability $p$ that individual comparisons are in favor of a model $model_1$ against a model $model_2$. 
Assume that a modification of the training set has an impact $\epsilon$ on the frequency $f$ with which $model_1$ was preferred against model $model_2$: the new frequency $f'$ is in $[f-\epsilon,f+\epsilon]$. $(f-\sfrac12)\times (f'-\sfrac12)\leq 0$ is possible only if $f\in [\sfrac12-\epsilon,\sfrac12+\epsilon]$. And since $f \sim \frac{1}{s}B(s,p)$, 
$P(f\in [\sfrac12-\epsilon,\sfrac12+\epsilon])$ is maximum if $p=\sfrac12$.
Then,
\begin{align*}    
    P\left(f\in [\sfrac12-\epsilon,\sfrac12+\epsilon]\right)   & \leq P\left(\frac{1}{s} B(s,\sfrac12) \in [\sfrac12-\epsilon,\sfrac12+\epsilon]\right) && \\
      & \leq P\left(B(s,\sfrac12) \in [s(\sfrac12 -\epsilon),s(\sfrac12 +\epsilon)]\right) &&\\
      & \leq \delta:= \frac{2s\epsilon + 1}{\sqrt{2s\pi}}\mbox{ thanks to \cref{dgllemma}.} &&
  \end{align*}
If someone modifies the data by modifying the answers corresponding to $m$ users, the frequency moves to $f'$ instead of $f$ with $|f-f'|\leq \epsilon=m/s$ where $s$ is the number of independent users/samples. 
This shows \cref{probadifference1}.
\end{proof}

To show~\cref{tpp}, we apply the Bonferroni correction when $m$ data values are modified at each iteration and we have $b$ iterations. Then, the probability of a difference between $output_1$ and $output_2$ is upper bounded by $b\times\delta$.
This shows~\cref{probadifferenceb}.

\subsection{Theorem on Data Extraction and Proof of \texorpdfstring{\Cref{co:unfair}}{Corollary}}\label{app:fair}

\cref{co:unfair} is a direct consequence of the following Theorem.

\vspace{5pt}

\begin{theoremUnboxed}\label{th:unfair}
    The cardinality of the set of possible outputs of \retrofit{} for a given $m_0,\w,a,b$ is bounded as follows:
    \begin{equation}
    \#\{  \mbox{\retrofit{}}(m_0,\w,D_i,a,b); i\in \{1,\dots,n \}\} \leq \sup_D \prod_{i=1}^b k_i\label{bou}.
    \end{equation}
\end{theoremUnboxed}
\begin{proof}
\cref{bou} is an immediate consequence of  \cref{eq:impact}.
If $A(m_0,D):=$\retrofit{} $(m_0,\w,D,a,b)$, then vulnerability to data extraction as in \cref{unfaireq} immediately implies that $\{E( A(m_0,D_i)); i \in \{1,\dots,n\} \}$ has cardinal at least $n$, and therefore $\{A(m_0,D_i);i\in \{1,\dots,n\}\}$ has cardinal at least $n$. If $2^b>n$ and $k_i\leq 2$, this contradicts~\cref{bou}.
\end{proof}
This means that no matter how large $s$ is, retrofitting can only produce a number of different outputs upper-bounded by the product over the \(k_i\), which represent the number of possible data-dependent choices per iteration - e.g., $2^b$ for algorithms with $k_i\leq 2$ and $b$ iterations.

\section{Experimental Details}\label{appendix:expdetails}

Here we provide details for the settings of the experiments in \cref{sec:xp}. 
\rebuttal{Experiments were implemented using Lingua \cite{meta_lingua} and vLLM \cite{kwon2023efficient}.}

All BBO algorithms are used with their default parameters in Nevergrad \citep{nevergrad}. 
At each iteration of \cref{alg:retro}, the modified model is evaluated on GSM8K train (8-shot, unless otherwise specified) and the objective function to optimize is the exact match. At the end of each run, the {\em FinalModel} is evaluated on the same objective function. Specifically, we compute the performance of $modified(m_0, \hat x)$ where $\hat{x}$ is the update that reaches the best objective function value during training.

\paragraph{Experiment 1.}
We apply \retrofit{} (\cref{alg:retro}) with rank-1 (LoRA) multiplicative parameter updates on the {\em unembedding layer}, as detailed in \cref{concrete}. In both base and instruct models, the total number of parameters to optimize is $hidden\_dim + vocab\_size = 4096 + 128,256 = 132,352$.
Each iteration requires around 40s on 16 A100 GPUs.

\paragraph{Experiment 2.}
We explore 3 setups on Llama3.1-8B:
\begin{itemize}
    \item[1.] \textbf{First attention layer}, where we perform a broadcast multiplicative update to the Q 
    matrix in the first attention layer by broadcasting a vector of size $4096$ to the $4096 \times 4096$ Q matrix and then using the update formula as in \cref{sec:onelayer}. 
    \item[2.] \textbf{All attention layers}, where we perform a broadcast  multiplicative update 
    on all the $Q$ matrices of all attention layers. Resulting in $n\_layers \times hidden\_dim = 32\times 4096 = 131,072$ parameters to optimize, using the update formula as in \cref{sec:multi_layers}
    \item[3.] \textbf{All attention layers and no few-shot}, where we apply a broadcast multiplicative update on all the $Q$ matrices of all attention layers, as in 2 above. However, we do not include any few-shot examples in the prompt during either training or testing. 
\end{itemize}

All experiments are carried out with budget $b=150$ unless otherwise specified. In addition to GSM8K test, we also evaluate the results on GSM+ (8-shot) and SVAMP (8-shot).
Each iteration requires around 40s on 16 A100 GPUs.

\paragraph{Experiment 3}
In this experiment, \DCMA\ is used to directly optimize all normalization layers of Llama3.1-8B(-Instruct) and Qwen2.5-3B-Instruct (without using any low rank modification). The number of updated parameters is $n\_norm\_layers \times hidden\_dim = 65 \times 4096 = 266,240$ for Llama3.1-8B(-Instruct) and $n\_norm\_layers \times hidden\_dim = 73 \times 2048 = 149,504$ for Qwen2.5-3B-Instruct.
All experiments are conducted with budget $b \in \{150, 300\}$. Furthermore, we evaluate on additional mathematical reasoning benchmarks: MATH500 (4-shot), Hellaswag (0-shot), GSM+ test\_mini (8-shot), ARC Easy (0-shot) and ARC Challenge (0-shot), AMC23 (0-shot).
Each iteration requires 30s on 16 A100 GPUs for Llama3.1-8B(-Instruct) and 25s on 16 A100 GPUs for Qwen2.5-3B-Instruct.

\paragraph{Experiment 4}
Most MIA rely on the loss (negative likelihood)  of the model over the tested sample to determine membership in the training dataset~\cite{fu2023practical,carlini2022membership,privacyattack4}. We can formulate the membership test as follows:
\begin{align*}
    f\left( \text{NLL}\{sample\} \right) > \text{threshold}\{sample\}
\end{align*}
where $f$ denotes a transformation (usually continuous) applied to the sample loss. The threshold used in membership test is sample dependent, in order to account for the fact that some samples may exhibit low losses even without being part of the training set, simply because they share patterns that were seen during training.

For example, in~\citet{carlini2022membership}, the authors propose to model the distribution of model's ``confidence'' with a gaussian distribution; they train shadow models to estimate the parameters $(\mu_{in}, \sigma_{in},\mu_{out},\sigma_{out})$ and perform a Likelihood-ratio test, \rebuttal{with $p$ the density function of Gaussian distribution and $\phi$ the logit scaling as defined in the LiRA paper $\phi(t) = \log(\frac{t}{1-t})$}:
\begin{align*}
    \frac{p(\phi(\text{NLL}\{sample\});\mu_{in}, \sigma_{in})}{p(\phi(\text{NLL}\{sample\});\mu_{out}, \sigma_{out})} &> \text{threshold}\\
    \Leftrightarrow p(\phi(\text{NLL}\{sample\});\mu_{in}, \sigma_{in}) &> \text{threshold} \times p(\phi(\text{NLL}\{sample\});\mu_{out}, \sigma_{out})
\end{align*}

Therefore, if a learning method does not impact the loss of the model over each sample, it would be difficult \rebuttal{or} impossible to perform any MIA that relies on statistical tests based on the loss. We thus propose the two following metrics as proxies for the ability of membership inference attacks to succeed:
\begin{align}\label{diff-full}
    \frac{1}{\text{\scriptsize num of samples}} \sum_{sample} |\text{\scriptsize per\_token\_}\text{NLL}_{\text{new model}}\{sample\} - \text{\scriptsize per\_token\_}\text{NLL}_{\text{\rebuttal{initial} model}}\{sample\}|
\end{align}

\begin{align}\label{diff-cot-a}
    \frac{1}{\text{\scriptsize num of samples}} \sum_{sample} |\text{\scriptsize per\_token\_}\text{NLL}^{CoT,a}_{\text{new model}}\{sample\} - \text{\scriptsize per\_token\_}\text{NLL}^{CoT,a}_{\text{\rebuttal{initial} model}}\{sample\}|
\end{align}

The only distinction between \cref{diff-full} and \cref{diff-cot-a} is which tokens are used to compute the negative log-likelihood (NLL): \cref{diff-cot-a} computes NLL only over the (chain-of-thought, answer) pair, whereas \cref{diff-full} computes NLL over the entire input (few-shot examples, question, chain-of-thought, and answer). We aggregate the absolute differences (i.e., sum of absolute changes in sample NLL) across all samples so that positive and negative changes do not cancel each other out; although in practice the loss of the new finetuned model will be lower than the initial model. To interpret the metric: a low value indicates that the sample losses are only slightly affected by the learning method, suggesting it would be difficult to mount an attack that relies on loss changes, whereas a large value indicates that it might be possible (but not certain) to perform such an attack successfully.
\changes{We propose using these metrics instead of a standard membership inference attack (MIA) metric. Because the initial model may have been trained in a way that is already detectable, it is more meaningful to measure the absolute difference in loss between the updated model and the initial model. Moreover, MIA performance is sensitive to hyperparameters (such as the number of shadow models, learning rate, and the amount of member and non-member data) so we prefer an intrinsic metric like the one we propose.}
By showing in~\cref{tab:exp4} that \retrofit{} barely affects these metrics, and is orders of magnitude less than finetuning methods, including privacy-aware ones like DP-AdamW, we demonstrate robustness of \retrofit{} against a vast majority of membership inference attacks.

\changes{\paragraph{WBC parameters:} we use the default window sizes as used by the authors in~\cite{chen2026window}, namely $w \in \{2,3,4,6,9,13,18,25,32,40\}$. The members data are 
$500$ datapoints sampled randomly from the training set of GSM8K and the non-members data are 500 datapoints sampled randomly from the test set of GSM8K. We only consider the NLL over the tokens of the CoT and answer as most methods we consider fine-tune only over those tokens.}

Below we detail all algorithms used in \cref{tab:exp4}:
\begin{itemize}
    \item \textbf{ft}: default SFT with next token prediction loss, global batch size of 128, same template used during evaluation that includes few shots examples but only the CoT and the answer contribute to the loss. AdamW is used with $(beta_1=0.9, beta_2=0.95)$ and a learning rate of $1e^{-5}$ with cosine annealing and warm-up over the first 5\% steps. The training is done for 2 and 5 epochs.
    \item \textbf{ft-norm}: similar to \textbf{ft}, but only the parameters of RMSNorm layers are optimized, and the learning rate is increased to $2e^{-4}$ instead.
    \item \textbf{DP-AdamW}: differentially private training with gradient norm clipping set to $1.0$, $\delta=10^{-5}$, $L=128$, $N=7473$ samples, $292$ gradient steps (5 epochs) and a target $\epsilon=8$, the privacy accounting of~\cite{abadi2016deep} is used to infer the added noise scale $\sigma$. The training is done with the same template used during evaluation that includes few shots, all tokens contribute to the loss.
    \item \textbf{\retrofit{}-norm}: exact setup as in experiment 3 for our algorithm, step $\delta = 0.01$, only updating RMSNorm layers with an exponential multiplicative full update. Different budget values are used 150, 300, 1200. Nevergrad seed = 42.
\end{itemize}

\section{Additional Experiments}\label{sec:addexps}
\FloatBarrier
\FloatBarrier

In this section, we describe extensions of the experiments performed in \cref{sec:xp}, in particular for Bet-and-Run variants and for training on larger datasets.

\paragraph{Extensions of Experiment 2.} We have performed additional experiments on Llama3.1-8B in the spirit of Experiment 2 (\cref{sec:xp}), by modifying the {\em output layer} instead of the attention layers, using a multiplicative update by broadcasting a vector of size $vocab\_size = 128,256$ as described in \cref{sec:onelayer}. We also experiment with two different update constants in $\{0.01, 0.0001\}$. \cref{fulltable} shows the results. We see improvements particularly for the larger update constant and Bet-and-Run algorithms.

\begin{table}[ht]
\scriptsize
\centering
\caption{\label{fulltable}
{ {\bf Additional Experiments}. Counterpart of \cref{tab} for the output layer: Performance difference ($\%$)  after vs before \retrofit{} with Llama3.1-8B (Base) on GSM8K-train, evaluated on GSM8K-test (ID), and (OOD) SVAMP and GSM+, with budget $b=150$. Number of seeds in brackets. \textcolor{MyRed}{Red: detrimental runs,  $<-1\%$.} \textcolor{MyGreen}{Green: beneficial runs, $>1\%$.} \textcolor{cyan}{Blue rows are bet-and-run cases}.
}
}
\setlength{\tabcolsep}{0.3em}
\begin{tabular}{|p{2.2cm}|c|c|c|}
\hline
\multicolumn{1}{|c|}{BBO}  &  \multicolumn{1}{c|}{GSM} &  \multicolumn{1}{c|}{SVAMP} &  \multicolumn{1}{c|}{GSM+}\\
\multicolumn{1}{|c|}{algorithm  (\# runs)} & progress & progress & progress \\
\hline 
\multicolumn{4}{|c|}{Output layer, no few-shot, broadcast d=128256}\\
\multicolumn{4}{|c|}{$w0\times e(0.0001\times x)$}\\
\hline
\cellcolor{lightgray} Base score  & \cellcolor{lightgray} 20.31 & \cellcolor{lightgray} 58.30 & \cellcolor{lightgray} 19.06 \\
\hline 
\OneFifth\ (2)  & 0.49 $\pm$ 0.34 & 0.05 $\pm$ 0.15 & 0.01 $\pm$ 0.02\\
\DCMA  & 0.07 & 0.00 & 0.01\\
\Lengler  & 0.00 & 0.00 & 0.00\\
\COLengler\ (2)  & 0.00 $\pm$ 0.00 & 0.05 $\pm$ 0.05 & 0.00 $\pm$ 0.00 \\
\hline
\multicolumn{4}{|c|}{Output layer, broadcast d=128256}\\
\multicolumn{4}{|c|}{$w0\times e(0.01\times x)$}\\
\hline
\cellcolor{lightgray} Base score  & \cellcolor{lightgray} 56.93 & \cellcolor{lightgray} 76.6 & \cellcolor{lightgray}51.34 \\
\hline 
\Lengler\ (2)   & \cellcolor{MyGreen}  1.89 $\pm$ 1.44 & -0.2 $\pm$ 0.6 &  0.9 $\pm$ 0.41 \\
\COLengler\ (5)   & \cellcolor{MyGreen}  2.01 $\pm$ 0.49 & -0.76 $\pm$ 1.03 & 0.19 $\pm$ 0.43 \\
\Discrete\ (4)   & 0.018 $\pm$ 0.03 & 0.075 $\pm$ 0.08 & 0.005 $\pm$ 0.015 \\
\Portfolio\ (2) & 0.79 $\pm$ 0.11 & 0.45 $\pm$ 0.15 & -0.05 $\pm$ 0.05\\
\rowcolor{cyan} \Triple   & \cellcolor{MyGreen}  2.27 & 0.00 & 0.57\\
\rowcolor{cyan} \MultiDisc   & \cellcolor{MyGreen}  1.89 & 0.20 & -0.27\\
\rowcolor{cyan} \MultiDisc   & \cellcolor{MyGreen}  2.19 & \cellcolor{MyRed}  -2.20 & -0.02\\
\hline
\end{tabular}
\end{table}

\paragraph{Experiment 3 for Bet-and-Run Algorithms.} The theoretical results in \cref{sec:bar} demonstrate that Bet-and-Run algorithms enable the use of larger budgets without risks of overfitting.  This has been partially  validated in Experiment 2 by the results of \Triple, performing well even with greater budget (Table \ref{tab}). We here extend this analysis and  report Bet-and-Run results for the setup of Experiment 3 in \cref{tab:various_gsm8k_BetAndRun}. Specifically, we run \retrofit{} under the same setup as in Experiment 3 (detailed in \cref{appendix:expdetails}), and use the BetAndRun (N, b) variant of \DCMA, where we run \DCMA\ for N different runs with budget $b'=b/N$, and pick the update $\hat x$ that performed the best across these runs.
In practice, we use the runs with 5 different seeds from \cref{tab:various_gsm8k} and report the mean and standard deviation of all possible combinations $\binom{5}{N}$ with $N \in \{3, 4\}$.

\begin{table*}[ht]{
\caption{
{\bf Extension of Experiment 3 to Bet-and-Run algorithms}, same setting as \cref{tab:various_gsm8k}
}\label{tab:various_gsm8k_BetAndRun}
\centering
\tabcolsep=0.1cm
\resizebox{1.\linewidth}{!}{
    \begin{tabular}{cc@{\hskip 0.2in}  c@{\hskip 0.2in} cccccc}
      \toprule
        \multicolumn{2}{c}{\multirow{2}{*}{Models}} & \multicolumn{1}{c}{ID} & \multicolumn{6}{c}{OOD}\\
        \cmidrule(rr){3-3} \cmidrule(rr){4-9}
        & & GSM8K & GSM+ & MATH500 & ARC Easy & ARC Challenge & AMC23 & Hellaswag \\
        \midrule
        \rowcolor{lightgray!30} \multicolumn{2}{c}{Llama3.1-8B} & 54.97 & 36.50 & 20.20 & 83.04 & 55.19 & 0.00 & 80.71 \\
        \multirow{4}{*}{Llama3.1-8B-BBoxER-BetAndRun} & (N=3, b=450) & 55.66$\pm$0.42 & 37.62$\pm$0.34 & 20.04$\pm$0.45 & 82.96$\pm$0.08 & 55.03$\pm$0.03 & 3.00$\pm$3.67 & 80.62$\pm$0.05 \\
        & (N=3, b=900) & 56.83$\pm$0.21 & 37.74$\pm$0.28 & 20.36$\pm$0.77 & 83.26$\pm$0.11 & 54.96$\pm$0.18 & 1.00$\pm$1.22 & 80.68$\pm$0.11 \\
        & (N=4, b=600) & 55.62$\pm$0.21 & 37.73$\pm$0.30 & 20.20$\pm$0.40 & 82.94$\pm$0.03 & 55.02$\pm$0.00 & 1.50$\pm$3.00 & 80.60$\pm$0.04 \\
        & (N=4, b=1200) & 56.85$\pm$0.12 & 37.77$\pm$0.22 & 20.48$\pm$0.64 & 83.29$\pm$0.10 & 54.88$\pm$0.07 & 0.50$\pm$1.00 & 80.71$\pm$0.10 \\
        \cmidrule{1-8}
        \rowcolor{lightgray!30} \multicolumn{2}{c}{Llama3.1-8B-Instruct} & 77.26 & 54.37 & 37.60 & 79.62 & 55.45 & 22.50 & 79.90 \\
        \multirow{4}{*}{Llama3.1-8B-Instruct-BBoxER-BetAndRun} & (N=3, b=450) & 78.20$\pm$0.21 & 55.02$\pm$0.12 & 36.68$\pm$0.36 & 78.90$\pm$0.38 & 55.40$\pm$0.08 & 23.25$\pm$2.25 & 79.96$\pm$0.04 \\
        & (N=3, b=900) & 78.37$\pm$0.21 & 55.18$\pm$0.15 & 37.20$\pm$0.54 & 79.57$\pm$0.12 & 55.59$\pm$0.19 & 18.00$\pm$1.50 & 79.92$\pm$0.02 \\
        & (N=4, b=600) & 78.26$\pm$0.03 & 55.03$\pm$0.10 & 36.52$\pm$0.24 & 78.77$\pm$0.34 & 55.42$\pm$0.07 & 24.00$\pm$2.00 & 79.96$\pm$0.03 \\
        & (N=4, b=1200) & 78.45$\pm$0.18 & 55.17$\pm$0.08 & 37.04$\pm$0.48 & 79.54$\pm$0.10 & 55.54$\pm$0.17 & 17.50$\pm$0.00 & 79.91$\pm$0.02 \\
        \cmidrule{1-8}
        \rowcolor{lightgray!30} \multicolumn{2}{c}{Qwen-2.5-3B-instruct} & 79.98 & 62.29 & 41.00 & 72.39 & 47.38 & 32.50 & 74.94 \\
        \multirow{4}{*}{Qwen-2.5-3B-instruct} & (N=3, b=450) & 83.06$\pm$0.24 & 62.12$\pm$0.32 & 42.70$\pm$0.20 & 72.69$\pm$0.12 & 47.85$\pm$0.17 & 36.00$\pm$1.22 & 74.99$\pm$0.07 \\
        & (N=3, b=900) & 83.90$\pm$0.42 & 61.66$\pm$0.32 & 42.26$\pm$0.65 & 72.61$\pm$0.21 & 47.63$\pm$0.31 & 37.00$\pm$2.45 & 75.04$\pm$0.05 \\
        & (N=4, b=600) & 83.08$\pm$0.18 & 62.24$\pm$0.27 & 42.72$\pm$0.16 & 72.71$\pm$0.05 & 47.86$\pm$0.10 & 35.50$\pm$1.00 & 75.00$\pm$0.06 \\
        & (N=4, b=1200) & 84.08$\pm$0.30 & 61.55$\pm$0.02 & 42.52$\pm$0.16 & 72.68$\pm$0.19 & 47.52$\pm$0.27 & 36.00$\pm$2.00 & 75.06$\pm$0.04 \\
        \bottomrule
    \end{tabular}}
\vspace{0.5cm}}
\end{table*}

Comparing the results in \cref{tab:various_gsm8k_BetAndRun} with those in \cref{tab:various_gsm8k} (\retrofit{} without Bet-and-Run), we observe that in 100\% runs (with or without Bet-and-Run) \retrofit{} improves results; in 83\% of cases Bet-and-Run outperforms its base counterpart.

\paragraph{Experiment 3 for Larger Training Sets.} In \cref{tab:various_gsm8k_math} we report the results of running \retrofit{} with a larger training dataset. This experiment is motivated by the theoretical prediction that the allowable budget scales linearly with the dataset size. Formally, we adopt the same experimental setup as in Experiment 3 (see \cref{appendix:expdetails}), but in addition to GSM8K train (7,473 samples), we also include the MATH train set (7500 samples) into the training set, thus doubling its size compared to previous experiments. We define the objective function for \retrofit{} to be the average of exact matches (EM) across both datasets : $metric = \frac{1}{2}EM(GSM8K) + \frac{1}{2}EM(MATH)$. We run \retrofit{} with a budget of up to $b=1200$ and evaluate the model every 150 iterations. As the budget increases, we observe eventual overfitting.  Interestingly, we find that the Llama3.1-8B base model does not overfit on the GSM8K test set and continues to improve, contrary to its instruct variant, which begins to overfit earlier. See also \cref{table7_figures} for illustration of performance gain on some of the benchmarks.

\FloatBarrier

\begin{table*}[ht]{
\caption{{\bf Experiment 3 (full update, normalization layers) on larger training set.} Comparison of model performance of BBoxER run with D-CMA across various benchmarks, 1 seed. 
}\label{tab:various_gsm8k_math}
\small
\centering
\tabcolsep=0.1cm
\resizebox{1.\linewidth}{!}{
    \begin{tabular}{cc@{\hskip 0.2in}  ccc@{\hskip 0.2in} ccccc}
      \toprule
        \multicolumn{2}{c}{\multirow{2}{*}{Models}} & \multicolumn{3}{c}{ID} & \multicolumn{5}{c}{OOD}\\
        \cmidrule(rr){3-5} \cmidrule(rr){6-10}
        &  & GSM8K & MATH & MATH500 & GSM+ & ARC Easy & ARC Challenge & AMC23 & Hellaswag \\
\midrule
\rowcolor{lightgray!30} \multicolumn{2}{c}{Llama3.1-8B} & 54.97 & 21.22 & 20.20 & 36.50 & 83.00 & 55.19 & 0.00 & 80.71 \\
\multirow{9}{*}{Llama3.1-8B-BBoxER} & (b=150) & 55.42 & 21.00 & 19.20 & 36.96 & 83.09 & 54.85 & 0.00 & 80.76 \\
& (b=300) & 55.27 & 21.14 & 18.80 & 36.79 & 83.09 & 54.85 & 2.50 & 80.78 \\
& (b=450) & 55.27 & 21.14 & 18.80 & 36.79 & 83.09 & 54.85 & 2.50 & 80.78 \\
& (b=600) & 55.42 & 20.74 & 20.60 & 37.67 & 83.04 & 54.94 & 5.00 & 80.77 \\
& (b=750) & 57.85 & 20.92 & 20.20 & 37.67 & 83.00 & 54.85 & 0.00 & 80.83 \\
& (b=900) & 57.47 & 21.58 & 21.20 & 38.75 & 82.66 & 55.11 & 0.00 & 80.66 \\
& (b=1050) & 57.47 & 21.58 & 21.20 & 38.75 & 82.66 & 55.11 & 0.00 & 80.66 \\
& (b=1200) & 58.61 & 21.14 & 21.00 & 39.33 & 83.04 & 55.36 & 5.00 & 80.81 \\
\midrule
\rowcolor{lightgray!30} \multicolumn{2}{c}{Llama3.1-8B-Instruct} & 77.26 & 38.84 & 37.60 & 54.37 & 79.62 & 55.36 & 22.50 & 79.90 \\
\multirow{9}{*}{Llama3.1-8B-Instruct-BBoxER} & (b=150) & 77.41 & 38.78 & 37.40 & 54.12 & 79.45 & 55.36 & 25.00 & 80.03 \\
 & (b=300) & 78.24 & 38.86 & 38.60 & 55.04 & 79.58 & 55.36 & 30.00 & 79.91 \\
 & (b=450) & 77.56 & 38.40 & 37.00 & 54.62 & 80.00 & 55.54 & 22.50 & 79.94 \\
 & (b=600) & 77.79 & 39.04 & 38.20 & 54.58 & 79.83 & 55.54 & 22.50 & 79.94 \\
 & (b=750) & 78.70 & 38.86 & 37.00 & 55.33 & 79.45 & 55.36 & 15.00 & 80.00 \\
 & (b=900) & 78.54 & 38.30 & 37.40 & 55.25 & 79.70 & 55.45 & 20.00 & 80.10 \\
 & (b=1050) & 77.86 & 38.62 & 37.40 & 55.29 & 79.53 & 55.54 & 20.00 & 79.98 \\
 & (b=1200) & 77.79 & 38.86 & 37.40 & 55.42 & 79.87 & 55.79 & 15.00 & 80.00 \\
\midrule
\rowcolor{lightgray!30} \multicolumn{2}{c}{Qwen2.5-3B} & 79.98 & 43.68 & 41.00 & 62.38 & 72.35 & 47.38 & 32.50 & 74.94 \\
\multirow{9}{*}{Qwen2.5-3B-BBoxER} & (b=150) & 83.55 & 43.32 & 39.80 & 62.00 & 72.77 & 48.24 & 37.50 & 75.07 \\
 & (b=300) & 82.79 & 44.10 & 42.60 & 62.17 & 72.60 & 47.64 & 40.00 & 75.09 \\
 & (b=450) & 83.09 & 43.66 & 39.40 & 62.58 & 72.81 & 48.24 & 42.50 & 74.94 \\
 & (b=600) & 82.64 & 44.94 & 42.60 & 62.17 & 72.98 & 47.90 & 42.50 & 74.93 \\
 & (b=750) & 82.87 & 45.26 & 44.00 & 61.54 & 71.54 & 47.47 & 32.50 & 74.85 \\
 & (b=900) & 84.08 & 45.68 & 47.80 & 63.04 & 72.60 & 48.15 & 40.00 & 74.92 \\
 & (b=1050) & 83.17 & 46.50 & 45.60 & 62.58 & 72.85 & 47.47 & 32.50 & 74.98 \\
 & (b=1200) & 82.49 & 46.72 & 45.80 & 61.46 & 73.32 & 46.27 & 37.50 & 74.91 \\
        \bottomrule
    \end{tabular}}
\vspace{0.5cm}}
\end{table*}

\begin{figure}[ht!]
\centering
\begin{minipage}[c]{\textwidth}
    \centering
    \includegraphics[width=0.4\textwidth]{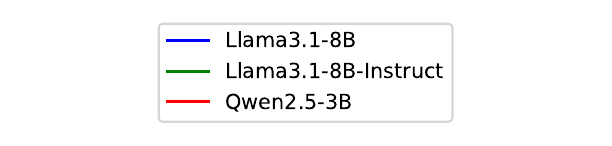}
\end{minipage}
\begin{minipage}[c]{\textwidth}
\hspace{0.2cm}
\includegraphics[width=0.3\textwidth]{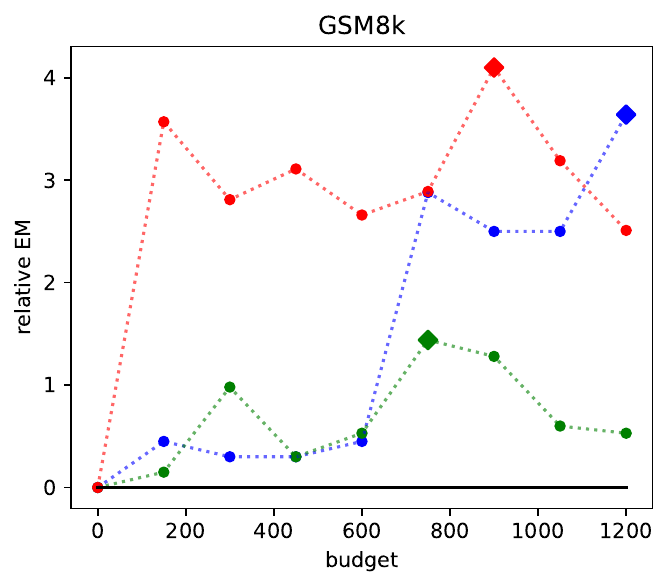}
\includegraphics[width=0.32\textwidth]{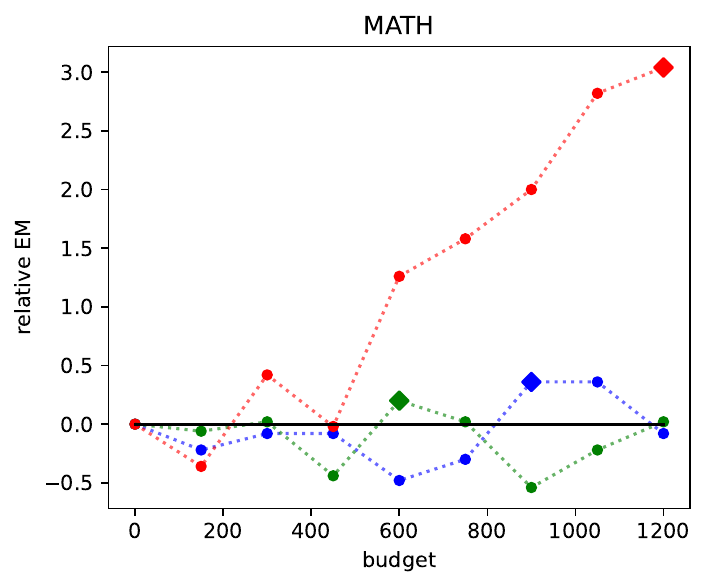}
\includegraphics[width=0.32\textwidth]{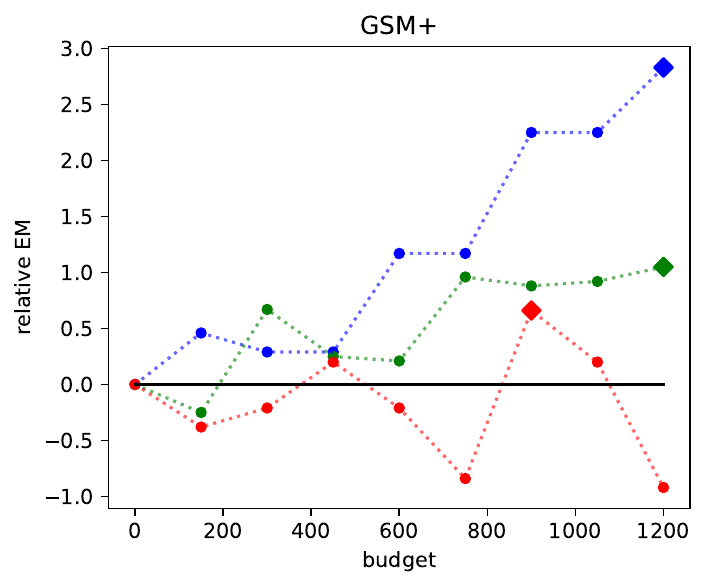}
\end{minipage}
\begin{minipage}[c]{\textwidth}
\includegraphics[width=0.32\textwidth]{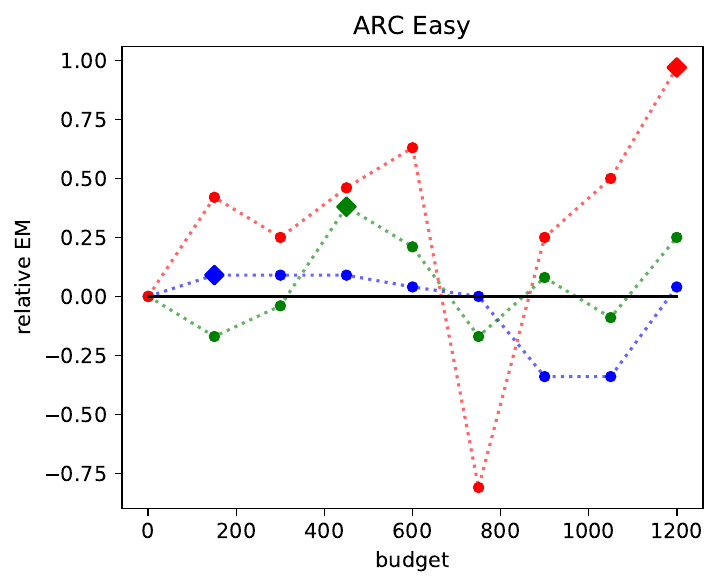}
\includegraphics[width=0.32\textwidth]{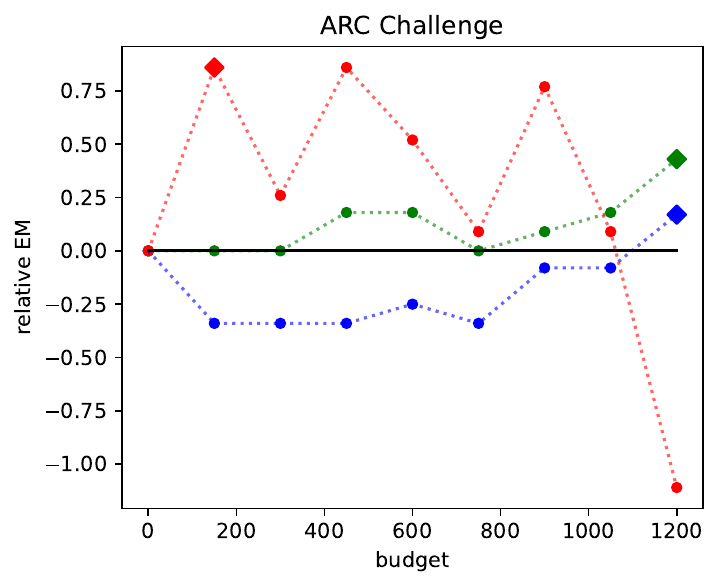}
\includegraphics[width=0.32\textwidth]{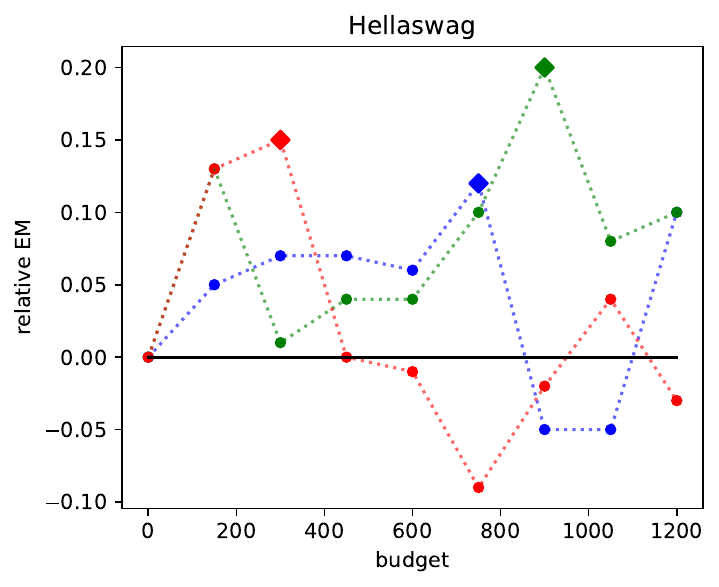}
\end{minipage}
\caption{\label{table7_figures} {\bf (full update, normalization layers)} Comparison of model performance of BBoxER run with D-CMA across various benchmarks, 1 seed. See \cref{tab:various_gsm8k_math}.}
\end{figure}

\rebuttal{
\paragraph{Experiment 3 for different tasks.} In \cref{tab:diff_tasks} we report the results of running \retrofit{} on tasks beyond mathematical reasoning. We update Llama-3.1-8B on the ARC-Easy training split and evaluate on the ARC-Easy test split, and repeat the same procedure for ARC-Challenge and Hellaswag. We report the mean and standard deviation across five random seeds of the relative improvement, or difference of accuracy, compared to the base model. We notice that on ARC Easy and Hellaswag we get an improvement, while it is not the case on ARC Challenge. This likely stems from the base model’s performance across tasks. On ARC Easy and Hellaswag, the base model scores about 80\%, so “retrofitting” can provide meaningful improvement. In contrast, on ARC-Challenge the base is about 50\% (random is at ~25\%), and BBoxER is designed to provide additional improvement as a lightweight add-on, rather than heavy duty training: if the base is weak, perturbing parameters won’t meaningfully improve the model even if we notice slight improvement on the training split.
}

\begin{table}
\centering
\caption{{\bf Experiment 3 (full update, normalization layers) different tasks.} Relative accuracy improvement using \retrofit{} run with D-CMA, 5 seeds.}\label{tab:diff_tasks}
\begin{tabular}{cccc}
\toprule
\retrofit{} budget b & ARC Easy & ARC Challenge & Hellaswag \\
\midrule
150 & 0.01 $\pm$ 0.10 & -0.14 $\pm$ 0.12 & 0.16 $\pm$ 0.07 \\
300 & 0.13 $\pm$ 0.06 & -0.21 $\pm$ 0.17 & 0.25 $\pm$ 0.11 \\
600 & 0.18 $\pm$ 0.16 & -0.05 $\pm$ 0.25 & 0.47 $\pm$ 0.06 \\
1200 & 0.85 $\pm$ 0.16 & 0.24 $\pm$ 0.21 & 0.89 $\pm$ 0.13 \\
\bottomrule
\end{tabular}
\end{table}

\paragraph{Further Discussion of all Experiments.}
Consistent with our theoretical results, our results are robust and show performance improvement with little to no overfitting even with larger budgets.
We note that even when running various experiments with randomly chosen parameters, significant performance drops are rare:  In \cref{xplama} and in \cref{tab}, all data points are independent from one another (i.e. computed independently); This allows us to compute P-values (\cref{pval}) to support our claims of statistical significance of the observed performance improvements.

 In order to contextualize our findings, we can also compare to post-training with GRPO which proved to be highly effective on reasoning tasks. 
 Most works deploy GRPO on larger training sets and we were unable to find results for the models deployed in our study. However, as shown in ~\cite{lin2025cppo},  GRPO achieved a notable 22\% improvement when training a Qwen-2.5-1.5B-Instruct model on the GSM8K dataset, leading to results between 77\% and 81\% with 1B models, without transfer to OOD. However, such improvement jumps tend to shrink with larger models. Nevertheless, though  the Qwen-2.5-3B-Instruct model we use already starts at nearly 80\%, it is improved to 84\% by \retrofit{} (\cref{tab:various_gsm8k_BetAndRun}).
 
There are only few works that train on GSM8K and study OOD transfer. 
The authors of ~\cite{sun2025self1improvement} employ an 'on-policy' training approach, generating training samples through model prompting and subsequently fine-tuning on the resulting question-answer pairs. Their approach yields significant gains on both GSM8k and GSM+ benchmarks when evaluated in a {\em zero-shot} setting, by up to 28.8\% on GSM8K with Llama3.1-8B-Instruct. However, the gains are less pronounced in the 5-shot setting, with an increase of 1.8\% using self-consistency and 5 generations, figures that are close to those obtained by \retrofit\ on GSM8K and GSM+, as can be seen in \cref{tab:various_gsm8k_BetAndRun}.\\
In any case, remember that our goal in this work is not to achieve state-of-the-art performances on these and other benchmarks, but rather to show that our approach sacrifices little for provable privacy and overfitting guarantees.

\section{P-values}\label{pval}
Contemporary papers on LLMs frequently involve a small number of experiments, due to the significant computational burden.
But because we run independent tests, we can compute rigorous p-values by Fisher's exact test, getting a measure of statistical significance despite the moderate number of experiments:
\begin{itemize}
    \item We computed the $p$-value of the frequency of green vs red in the columns of  \cref{tab} and \cref{fulltable}, the null hypothesis being "the probability of improvement is $\leq 0.5$". 
    \item For the values in \cref{xplama}, we computed the $p$-value for the obtained frequency of successful runs (\ie{} runs above baseline, vs below baseline) for budget $>20$, under the null hypothesis "the probability of improvement is $\leq 0.5$". 
    \item In Experiment 3, we run \retrofit{} with budget $b=150$ or $b=300$. For each column of results we performed an exact Fisher test on all \retrofit{} runs (both 150 and 300, without considering the Bet-And-Runs) and computed p-values: there is therefore one $p$-value per downstream task, against the null hypothesis  "there is a probability of improvement $\leq \frac12$". These different tests are not independent.
\end{itemize}

When computing p-values from tables of results, we used the initial data, not the averaged ones, obtaining one data point per independent run.

{\bf P-values Values.}

{\bf Experiment 1}: P-values in \cref{xplama} are $3e^{-8}$ (left), $3e^{-5}$ (right).

{\bf Experiment 2}:
\begin{itemize}
\item \cref{tab}: P-values  are significant for GSM ($5e^{-3}$) and  GSM+ ($0.03$), not for SVAMP. 
\item \cref{fulltable}: P-values are 0.003 for GSM, 0.5 for the transfer to GSM+, no improvement for SVAMP.
\end{itemize}

{\bf Experiment 3}: The P-values are $5e^{-7}$ for GSM8K, $1e^{-4}$ for GSM+, $1e^{-4}$ for AMC23.\\
Other results are usually positive but without significant p-value: 0.13 for ARC Easy, 0.21 for ARC Challenge, 0.36 for MATH500, 0.12 for Hellaswag.

\section{BBoxER Execution Example}\label{section:bboxer_example}
\rebuttal{
In this section we provide a step by step execution example of \retrofit{} (\cref{alg:retro}), using \OneFifth{} (\cref{three}, (1+1)-ES) as the BBO algorithm and \textit{Full Layer Update} on the normalization layers as our choice of $modified$.
}

\rebuttal{
\textbf{Input}: seed $\omega$, dataset $D$, BBO algorithm $a$=\OneFifth{}, budget $b$
\begin{description}
    \item[\textbf{step 1}:] initialize the internal state by seeding the rng with $\omega$, and fixing the step size $\sigma$
    \item[\textbf{step 2}:] sample a $z \sim \sigma\mathcal{N}(0,I)$ via $a.ask()$
    \item[\textbf{step 3}:] update model weights $m_i = modified(m_0, z)$, where every RMSNorm layer weights $w_i^{(j)}$ are updated with $w_i^{(j)} = w_0^{(j)} \times \exp{(0.01*z_j)}$ and $z = (z_1,\dots,z_\text{number RMSNorm layers})$
    \item[\textbf{step 4}:] $k_i$ = $a.numCases()$ = 2 (we only compare with the best so far, therefore 2 possible choices)
    \item[\textbf{step 5}:] $choice_i = a.compare((m_{best}, m_i), D)$ this comparison is either done by majority vote in A/B testing or by comparing the performance of $m_{best}$ and $m_i$ over $D$ (note that only the comparison result is needed and not the actual performance of each model)
    \item[\textbf{step 6}:] $a.tell(choice_i)$ and $I_i = (I_{i-1}, choice_i)$ update the internal state of algorithm a : if $m_i$ is better ($choice_i=1$) then $\sigma = 2*\sigma$, otherwise ($choice_i=0$) $\sigma = 2^{-\frac{1}{4}}*\sigma$
    \item[$\rightarrow$] repeat steps 2 to 6 for b iterations
\end{description}
}

\end{document}